\documentclass{article} % For LaTeX2e
\usepackage{amsmath,amsfonts,bm}

\def\Figref#1{Figure~\ref{#1}}

\def\secref#1{section~\ref{#1}}
\def\eqref#1{\ref{#1}}
\def\1{\bm{1}}

\DeclareMathAlphabet{\mathsfit}{\encodingdefault}{\sfdefault}{m}{sl}
\SetMathAlphabet{\mathsfit}{bold}{\encodingdefault}{\sfdefault}{bx}{n}

\DeclareMathOperator*{\argmax}{arg\,max}
\DeclareMathOperator*{\argmin}{arg\,min}

\DeclareMathOperator{\sign}{sign}

\usepackage{hyperref, times}
\usepackage{url}

\usepackage[utf8]{inputenc} % allow utf-8 input
\usepackage[T1]{fontenc}    % use 8-bit T1 fonts
\usepackage{hyperref}       % hyperlinks
\usepackage{url}            % simple URL typesetting
\usepackage{booktabs}       % professional-quality tables
\usepackage{amsfonts}       % blackboard math symbols
\usepackage{nicefrac}       % compact symbols for 1/2, etc.
\usepackage{microtype}      % microtypography
\usepackage{xcolor}         % colors
\usepackage{natbib}
\usepackage{amsthm}

\usepackage{bm, fullpage}
\usepackage{amsfonts,color,float,verbatim}
\usepackage{algorithm,algorithmic}
\usepackage{bbm}
\usepackage{subcaption}

\usepackage{array}

\usepackage{hyperref}
\hypersetup{
	colorlinks   = true, %Colours links instead of ugly boxes
	urlcolor     = blue, %Colour for external hyperlinks
	linkcolor    = blue, %Colour of internal links
	citecolor   = black %Colour of citations
}

\newtheorem{theorem}{Theorem}[section]
\newtheorem{lemma}{Lemma}[section]

\newtheorem{definition}{Definition}[section]
\newtheorem{remark}{Remark}[section]
\renewenvironment{proof}[1][Proof: ]{\noindent\textbf{#1}}{\qed\medskip}
\newtheorem*{example}{Example}

\newcommand{\lemref}[1]{Lemma~\ref{#1}}

\newcommand{\thmref}[1]{Theorem~\ref{#1}}

\newcommand{\appref}[1]{Appendix~\ref{#1}}
\newcommand{\defref}[1]{Definition~\ref{#1}}
\newcommand{\assref}[1]{Assumption~\ref{#1}}

\newtheorem{open problem}[theorem]{Open Problem}

\usepackage{amssymb}

\usepackage{dsfont}
\newcommand{\onefunc}{\mathds{1}}

\newcommand{\stam}[1]{}

\usepackage{mathtools}

\newtheorem{assumption}[theorem]{Assumption}

\newcommand{\bx}{\mathbf{x}}
\newcommand{\bw}{\mathbf{w}}

\newcommand{\bu}{\mathbf{u}}
\newcommand{\bv}{\mathbf{v}}

\newcommand{\bmu}{\boldsymbol{\mu}}
\newcommand{\bzeta}{\boldsymbol{\zeta}}

\newcommand{\btheta}{{\boldsymbol{\theta}}}

\DeclareMathOperator*{\cossim}{cossim}

\newcommand{\reals}{{\mathbb R}}
\newcommand{\nat}{{\mathbb N}}

\newcommand{\zero}{{\mathbf{0}}}

\newcommand{\inner}[1]{\langle #1 \rangle}
\newcommand{\norm}[1]{\left\|#1\right\|}
\title{Provable Unlearning with Gradient Ascent on \\
Two-Layer ReLU Neural Networks}

\author{
    Odelia Melamed \thanks{Weizmann Institute of Science, \texttt{odelia.melamed@weizmann.ac.il}}
    \and 
    Gilad Yehudai\thanks{Center for Data Science, New York University, \texttt{gy2219@nyu.edu}}
    \and 
	Gal Vardi \thanks{Weizmann Institute of Science,  \texttt{gal.vardi@weizmann.ac.il}}
}

\date{}

\begin{document}

\maketitle

\begin{abstract}
Machine Unlearning aims to remove specific data from trained models, addressing growing privacy and ethical concerns. We provide a theoretical analysis of a simple and widely used method—gradient ascent— used to reverse the influence of a specific data point without retraining from scratch. Leveraging the implicit bias of gradient descent towards solutions that satisfy the Karush-Kuhn-Tucker (KKT) conditions of a margin maximization problem, we quantify the quality of the unlearned model by evaluating how well it satisfies these conditions w.r.t. the retained data. To formalize this idea, we propose a new success criterion, termed \textbf{$(\epsilon, \delta, \tau)$-successful} unlearning, and show that, for both linear models and two-layer neural networks with high dimensional data, a properly scaled gradient-ascent step satisfies this criterion and yields a model that closely approximates the retrained solution on the retained data.  We also show that gradient ascent performs successful unlearning while still preserving generalization in a synthetic Gaussian-mixture setting. 
\end{abstract}

\section{Introduction}

% \gilad{Changes: (1) Forgetting a subset. (add a separate section)(2) single step vs. multiple steps. (don't mention) (3) Experiments on a more realistic dataset. (return to section 4) (4) Necessary vs. sufficient condition. (5) Sensitivity of the stepsize. (don't mention) (6) lines 173-174, initialization mistake
% (7) emphasize that proximity in cossim implies a similar behavior (due to the homogeneity and the Lipschitzness of the output w.r.t. the parameters)}

Machine Unlearning is an emerging field motivated by growing societal and legal demands—specifically, the need for machine learning models to "forget" specific data upon request. This concern has intensified following discoveries that private training data can be extracted from model outputs or weights \citep{carlini2019secret, haim2022reconstructing, fredrikson2015model}. The demand is further reinforced by regulations such as the EU GDPR’s Right to be Forgotten, as well as concerns about security and ethical AI. Machine unlearning addresses this challenge by aiming to undo the effect of particular samples without incurring the cost of full retraining.

The concept of unlearning was first formalized by \citet{Cao2015Towards} in the context of statistical query learning and has since been extended to deep neural networks. Broadly, two main approaches have emerged: retraining-based unlearning, which ensures complete data removal but is computationally expensive, and approximate unlearning, which aims for efficiency at the cost of weaker guarantees. Due to the stochastic and incremental nature of modern training procedures, which entangle data contributions, it is nontrivial to reverse the effect of the data to be forgotten while minimizing disruption to the retained data. 
% Thus, one line of research try to alter the training process to adapt to a faster future forgetting
% % i.g., training different sub-models on different partitions of the data, or saving intermediate training gradient information for points we may wish to forget
% \citep{allouah2024fast,bourtoule2021machine,graves2021amnesiac}.\gal{Fix the first reference} 
% % \gilad{The paragraph from here on is confusing. DP is not invented for unlearning, and the text mixes methods of unlearning and measuring their success. I think it is better to remove the rest of the text and go straight to the next paragraph. }

There is a large body of research on adapting given networks, namely, manipulating the weights post-training. For a training set $S$, a set of points $S_\text{forget} \subseteq S$ to unlearn, and its complement $S_\text{retain} = S \setminus S_\text{forget}$, a direct approach is to increase the training loss for samples in $S_\text{forget}$ using gradient steps.
% this line of works include fine-tune for $S_\text{retain}$ only (hoping for catastrophic forgetting of the rest of data) or with wrong labels for data in $S_\text{forget}$ (\cite{golatkar2020eternal, triantafillou2024we, graves2021amnesiac}). A more direct approach for this post-training unlearning, is the effort to directly increase the training loss for samples in $S_\text{forget}$.
This direct method was first implemented in \emph{NegGrad} \citep{golatkar2020eternal}, simply taking multiple negative gradient steps for $S_\text{forget}$ with respect to the training loss. Other gradient-related post-training methods use other losses and second order information for improved results \citep{guo2019certified, golatkar2020eternal, warnecke2021machine,triantafillou2024we,graves2021amnesiac}. There are also additional variants of \emph{NegGrad}, such as \emph{NegGrad+} \citep{kurmanji2023towards}, and \emph{Advanced NegGrad} \citep{choi2023towards} which add a recovery phase, performing additional training steps on the retained set. %\gal{say what it means}. 
In this work, we study the important building block of this foundational and widely-used method, a single gradient ascent step on the training loss w.r.t. $S_\text{forget}$.  
% Our focus is on analyzing the core mechanism—gradient ascent—both empirically and theoretically. 

One central question in the regime of approximate unlearning is how to measure unlearning performance. A common criterion, inspired by differential privacy \citep{dwork2014algorithmic}, evaluates success by comparing the output distributions of a model retrained from scratch with those of the unlearned model. This approach allows for approximate guarantees, where the distance between the two distributions is bounded by small parameters \citep{triantafillou2024we, ginart2019making}, providing a formal framework for quantifying the effectiveness of unlearning algorithms, albeit it is often too stringent.
%\gal{DP always allow for some $epsilon, delta$ change in the distribution. What's the difference?}

To provide a rigorous framework for analyzing unlearning, we turn to recent results on the implicit bias of neural networks under gradient descent \citep{lyu2019gradient, ji2020directional}. These works show that training tends toward solutions that satisfy the Karush-Kuhn-Tucker (KKT) conditions of the maximum-margin problem. We use these conditions to formulate an unlearning criterion: A successful unlearning procedure should modify the model from satisfying the KKT conditions w.r.t. $S$ to approximately satisfying them w.r.t. $S_\text{retain}$. 
This property is necessary for successful unlearning. That is, since a network retrained only on $S_\text{retain}$ converges to a KKT point w.r.t. $S_\text{retain}$, then a successful unlearning algorithm also needs to obtain such a KKT point, at least approximately.
Note that the approximation relaxation here is analogous to the relaxation for the distribution distance, allowing bounds on the deviation from exact solution attained by retraining. 
% While this setting doesn't include analysis of the distribution of such solutions, prior work suggests that any KKT point of the margin maximization problem over $S_\text{retain}$ is reachable by training solely on it starting from some initialization point \gilad{This sentence is unclear}. 
% Moreover, since real-world training yields only approximate satisfaction of these conditions, we define unlearning as the process of transitioning between two such approximate KKT points \gilad{This point is too technical for the intro }.

In our work, we analyze the unlearning performance of one gradient ascent step of a carefully chosen size. We define a new unlearning criterion for an unlearning algorithm $\mathcal{A}$, called  \textbf{$(\epsilon, \delta, \tau)$-successful} unlearning, using the KKT conditions as discussed above. Next, in both linear models and two-layer neural networks trained with high dimensional (or nearly-orthogonal) data, we prove that a gradient ascent step of an appropriate size is a successful unlearning algorithm. In addition, 
%we investigate how unlearning affects generalization, show a connection between meeting the KKT conditions and the ability to generalize.
we show a setting where unlearning using gradient ascent is both successful and does not hurt the model's generalization performance. 

% demonstrate a successful unlearning using the KKT conditions. We show, that gradient ascent step of such size can yield a model that is similar to one satisfying the KKT conditions for  $S_\text{retain}$. We next investigate how unlearning affects generalization, show a connection between meeting the KKT conditions and the ability to generalize. Finally, we experimentally show the sensitivity of unlearning to the choice of step size. In more details, our main contributions are:
In a bit more detail, our main contributions are:
\begin{itemize}
    \item For linear predictors, where the margin-maximizing solution is unique, we prove that 
    %selecting the step size using this theoretical framework 
    gradient ascent with an appropriate step size is a \textbf{$(\epsilon, \delta, \tau)$-successful} unlearning algorithm. Specifically, it yields an approximately max-margin predictor for $S_\text{retain}$. Moreover, due to the uniqueness of the solution, the unlearned predictor aligns closely—measured via cosine similarity—with the exact model retrained on $S_\text{retain}$.%\gal{rephrase...}
    \item We extend these findings to a two-layer neural network setting. Despite the added complexity and nonlinearity, we prove that a single gradient ascent step is a \textbf{$(\epsilon, \delta, \tau)$-successful} unlearning algorithm for some small $\epsilon,\delta,\tau$.
    \item We show that unlearning does not compromise out-of-sample prediction, using a synthetic mixture-of-Gaussians dataset.
    %, where the cluster diameters exceed their means. 
    We show that models unlearned via gradient ascent maintain generalization performance comparable to the original.
    % \item We preform experiments of unlearning with single step of different sizes, demonstrating that deviations from the theoretically optimal value degrade unlearning performance, resulting in models that diverge from the approximate KKT solution.
\end{itemize}

\subsection*{Related Work}
Machine unlearning was initially proposed in the statistical query setting by \citet{Cao2015Towards} and later extended to deep neural networks. The strongest unlearning guarantees are often formalized via \emph{differential privacy} \citep{dwork2014algorithmic}, requiring indistinguishability between unlearned and retrained model outputs. This was relaxed using KL-divergence \citep{golatkar2020eternal}, while other lines of work evaluate unlearning effectiveness through privacy attacks, such as membership inference or data reconstruction \citep{NIU2024404, haim2022reconstructing}.

To achieve these goals, many methods aim to avoid full retraining. For example, SISA \citep{bourtoule2021machine} partitions the training data into multiple shards to enable a faster future forgetting. \citet{graves2021amnesiac} proposed saving intermediate gradients during training with respect to different training data points, enabling faster simulation of retraining using these intermediate gradients without the forget set. Post-training approaches include fine-tuning for $S_\text{retain}$ only (hoping for catastrophic forgetting of the rest of data) or with wrong labels for data in $S_\text{forget}$ (\cite{golatkar2020eternal, triantafillou2024we, graves2021amnesiac, kurmanji2023towards}), or using different losses \citep{golatkar2020eternal}. These techniques often rely on gradient-based updates, with loss functions adjusted for unlearning objectives. Several methods also incorporate second-order information for better precision \citep{guo2019certified, golatkar2020eternal, warnecke2021machine}.

The gradient-ascent method was first introduced by \citet{golatkar2020eternal} as \emph{NegGrad}, applying negative gradient steps to increase loss on the forget set. Its extensions, \emph{NegGrad+} \citep{kurmanji2023towards} and \emph{advanced NegGrad} \citep{choi2023towards}, add a recovery phase by performing fine-tuning on the retained set. In this work, we isolate the basic component—gradient ascent—and study its behavior analytically.

On the theoretical side, \citet{guo2019certified} analyzed linear models and proposed a certified unlearning framework. Leveraging the existence of a unique optimal solution, they argue that inspecting the training gradients on the retained dataset can reveal residual influence from the deleted point—particularly when the model incurs non-zero loss, which may indicate incomplete unlearning. \citet{sekhari2021remember} analyze unlearning capacity based on test loss degradation. Our approach defines unlearning through the lens of KKT conditions, building on a line of work showing that training converges to a KKT point of the margin maximization problem for the dataset.

\paragraph{implicit bias and margin maximization}

A great body of research has studied the implicit bias of training neural networks with gradient methods toward solutions that generalize well \citep{neyshabur2017exploring,zhang2021understanding}. Our analysis is based on the characterization of the implicit bias of gradient flow on homogeneous models towards KKT solutions of the max margin problem, a result due to \cite{lyu2019gradient} and \cite{ji2020directional}. Implicit bias towards margin maximization was previously studied also for linear predictors \citep{soudry2018implicit}, deep linear networks and linear convolutional networks \citep{gunasekar2018bimplicit}. For a survey on implicit bias of neural networks see \cite{vardi2023implicit}.
% \gal{TODO}

\section{Settings}
\paragraph{Notations.} For $m \in \nat$, we denote $[m]=\{1,2,\dots,m\}$, and for $l \in [m]$, we denote $[m]_{-l} = [m] \setminus\{\ell\}$. We use bold-face letters to denote vectors, e.g., $\bx = (x_1,\dots,x_d) \in \reals^d$. We use $\norm{\bx}$ to denote the Euclidean norm of a vector $\bx$. We denote by $\onefunc_{x\geq 0}$ the indicator function such that $\onefunc_{x\geq 0} = 1$ if $x\geq 0$ and $0$ otherwise. We denote by $\sign(x)$ the sign function, $\sign(x)=1$ if $x\geq 0$ and $-1$ otherwise. We denote by $\mathcal{U}(A)$ the uniform distribution over a set $A$. For a distribution $\mathcal{D}$, we denote by $\bx \sim \mathcal{D}^m$ a vector $\bx$ that consists of $m$ i.i.d. samples from $\mathcal{D}$. We denote by $\cossim(\bx_1, \bx_2)$ the cosine similarity of vectors $\bx_1, \bx_2$, defined by $\cossim(\bx_1, \bx_2) = \frac{\inner{\bx_1, \bx_2}}{\norm{\bx_1}\norm{\bx_2}}$.

\subsection{Architectures and training}
\label{sec:settings}
In this paper, we discuss unlearning in two fundamental models: 
a linear predictor and a two-layer fully connected network.
For an input $\bx \in \reals^d$ and a vector $\bw \in \reals^d$, we will denote a linear predictor by $N(\bw,\bx) = \bw^\top \bx$.
Our two-layer network is defined by 
\begin{equation}\label{eq:NN}
    N(\btheta,\bx)  = \sum\limits_{j=1}^{n} u_j \sigma(\bw_j^\top \bx)~,
\end{equation}
where $\sigma(z) = \max(z,0)$ is the ReLU activation function. For all $j \in [n]$, we initialize $u_j \sim \mathcal{U}\left(\{-\frac{1}{\sqrt{n}},\frac{1}{\sqrt{n}}\}\right)$ and fix them throughout training. 
%Fixing the second layer weights is done for tractability of the proofs, and is common in many theoretical works that analyze the implicit bias of models, see e.g. \gilad{Gal, can you add some examples?}\gal{This assumption is common in theoretical works in general. Mostly when analyzing optimization. For implicit bias, there aren't many such works (I have two works with spencer where we fixed the second layer, but currently don't recall other works).}.
% , initialize the hidden weights $\bw_j$ randomly from some distribution,
%\om{add weight init for both cases} 
The parameters $\bw_1,\ldots,\bw_n$ are trained. 
We denote by $\btheta$ a vectorization of all the trained parameters.
% the vector obtained by concatenating all hidden weights.

% \gilad{No need to introduce both GD and GF. We can only talk about the one we use.}
Given a training set $S=\{(\bx_i,y_i)\}_{i=1}^m$, we train our models using gradient descent over the empirical loss 
\[
    L(\btheta) = \frac{1}{m}\sum\limits_{i=1}^{m} \ell (y_i N(\btheta,\bx_i))~,
\] 
where $\ell$ is either the logistic loss $\ell (q) = \log(1+ e^{-q})$ or the exponential loss $\ell (q) = e^{-q}$.
%refer to the training as done using the logistic loss $\ell (q) = \log(1+ e^{-q})$ or the exponential loss $\ell (q) = e^{-q}$, as $L(\btheta) = \frac{1}{m}\sum\limits_{i=1}^{m} \ell (y_i N(\btheta,\bx_i))$, with gradient descent. 
%We describe the gradient descent training as 
That is, we have 
% $\btheta_{t+1} = \btheta_t -\beta_t \nabla L(\btheta_t)$, 
$\btheta_{t+1} = \btheta_t -\beta \nabla L(\btheta_t)$, 
where $\btheta_t$
are the weights after the $t$-th training epoch, 
%$\nabla L(\btheta_t)$ is the gradient of the loss $L$ at $\btheta_t$, 
and 
%$\beta_t$ 
$\beta$ 
is the step size.
% learning rate (or step size).
We consider the limit where $\beta$ is infinitesimally small, called \emph{gradient flow}. More formally, in gradient flow the trajectory $\btheta_t$ is defined for all $t \geq 0$ and satisfies the differential equation 
$\frac{d \btheta_t}{dt} = -\nabla L(\btheta_t)$. 
%$\frac{d \btheta_t}{dt} \in -\partial L(\btheta_t)$. 
%Here, $\partial$ denotes the \emph{Clarke subdifferential} \citep{clarke2008nonsmooth}, which is a generalization of the derivative for non-differentiable functions.

For a model $N(\btheta,\bx)$, where $\btheta$ are the parameters and $\bx$ is the input, we say that $N$ is homogeneous if there exists $C>0$ such that for every $\alpha >0$, and $\btheta,\bx$, we have $N(\alpha \btheta, \bx) = \alpha^C N(\btheta,\bx)$. We note that both a linear predictor and a two-layer network, as defined above, are homogeneous with $C=1$.

% We train the linear predictor with the following margin-maximization problem:
% \begin{align}\label{eq:lin_margmax}
%     \min_{\bw \in \reals^d} \norm{\bw}^2 \text{ such that for all } i \in [m] \text{,  } y_i N(\bw,\bx_i) \geq 1~.
% \end{align}

% By \cite{soudry2018implicit}, gradient descent in linear classification with exponentially-tailed losses (e.g., logistic or exponential loss) has implicit bias towards 
% %such solutions.
% max-margin solutions. That is, as $t \to \infty$ the linear predictor $\bw_t$ converges in direction to a solution of
% the following margin-maximization problem:
% \begin{align}\label{eq:lin_margmax}
%     \min_{\bw \in \reals^d} \norm{\bw}^2 \; \text{ s.t. } \; \forall i \in [m] \text{,  } y_i N(\bw,\bx_i) \geq 1~.
% \end{align}

For both linear and two-layer ReLU networks, there is an implicit bias towards margin maximization, as implied by the following theorem:
% we define an equivalent margin-maximization problem:
% \begin{align}\label{eq:margmax}
%     \min_{\btheta \in \reals^{d \times n}} \frac{1}{2}\norm{\btheta}^2  \; \text{ s.t. } \;  \forall i \in [m] \text{,  } y_i N(\btheta,\bx_i) \geq 1~.
% \end{align}
% We paraphrased a well-known result regarding this problem, for completeness
\begin{theorem}
    [\cite{lyu2019gradient}, \cite{ji2020directional}] 
    \label{thm:lyuli1}
    Let $N(\bx,\btheta)$ be a homogeneous linear or ReLU neural network. Consider minimizing the logistic or exponential loss using gradient flow 
    %(a continouos analog to gradient descent) 
    over a binary classification set $S=\{(\bx_i,y_i)\}_{i=1}^m \subseteq \reals^d \times \{-1, 1\}$. Assume that there is a time $t_0$ where $L(\btheta_{t_0}) < \frac{1}{m}$. Then, gradient flow converges in direction\footnote{we say that gradient flow \textit{converges in direction} to some $\tilde \btheta$ if $\lim_{t \to \infty}\frac{\btheta_t}{\norm{\btheta_t}} = 
\frac{\tilde{\btheta}}{\norm{\tilde{\btheta}}}$.\label{fn:dirconverge}} to a first-order stationary point (i.e., Karush--Kuhn--Tucker point, or KKT point for short) of the margin-maximization problem:
    \begin{align}\label{eq:margmax}
         \min_{\btheta} \frac{1}{2}\norm{\btheta}^2  \; \text{ s.t. } \;  \forall i \in [m] \text{,  } y_i N(\btheta,\bx_i) \geq 1~.
     \end{align}
\end{theorem}

Note that in the case of linear predictors a KKT point is always a global optimum,\footnote{For linear predictors, the theorem was obtained by \cite{soudry2018implicit}.} but in the case of non-linear networks this is not necessarily the case. Thus, in non-linear homogeneous models gradient flow might converge to a KKT point which is not necessarily a global optimum of Problem~\eqref{eq:margmax}. 

While the above theorem captures the asymptotic behavior of gradient flow, namely as the time $t \to \infty$ it converges to a KKT point, the behavior of gradient flow after a finite time can be characterized by \emph{approximate KKT points}.

% Given this powerful theorem, we will assume both of our networks have converged to a $(\epsilon, \delta)$-approximate KKT point of the respective margin-maximization problem. Next, we give an $(\epsilon, \delta)$-approximate KKT definition taken from \cite{lyu2019gradient}.
\begin{definition}
\label{def:epsdeltakkt}
 We say that $\btheta$ is a $(\epsilon, \delta)$-approximate KKT point for 
 % $\min_\btheta \frac{1}{2} \norm{\btheta}^2 \text{s.t.} \forall i, y_i N(\btheta, \bx_i) \geq 1$ 
Problem~\ref{eq:margmax},
if there exist $\lambda_1,...,\lambda_m$ such that
\begin{enumerate}
    \item Dual Feasibility: $ \lambda_1, ..., \lambda_m \geq 0$.
    \item Stationarity: $\norm{\btheta - \sum_{i=1}^m\lambda_iy_i \nabla_{\btheta} N(\bx_i, \btheta)} \leq \epsilon$.
    \item Complementary Slackness: $\forall i \in [m]$, $\lambda_i \left( y_i N(\bx_i, \btheta)-1 \right) \leq \delta$.
    \item Primal Feasibility: $\forall i \in [m]$, $y_i N(\bx_i, \btheta) \geq 1$.
\end{enumerate}
\end{definition}
We note that a $(0,0)$-approximate KKT point is a KKT point. 
When training with gradient flow, the parameters after finite time satisfy the following:
\begin{theorem}
    [\cite{lyu2019gradient}, \cite{ji2020directional}] 
    \label{thm:lyuli2}
    Under the conditions of Theorem~\ref{thm:lyuli1}, the parameters $\btheta_t$ at time $t$ point at the direction of an $(\epsilon_t, \delta_t)$-approximate KKT point for Problem~\ref{eq:margmax}, and $(\epsilon_t, \delta_t) \to (0,0)$ as $t \to \infty$.
\end{theorem}

% In \cite{lyu2019gradient}, the authors show that training with a finite number of steps, will converge to an $(\epsilon, \delta)$-approximate KKT point w.r.t. the dataset (that will vanish to $0$ when number of steps approaches infinity). 
Hence, when training a model
% linear predictor or a two-layer ReLU network 
%for finitely many iterations, 
it is reasonable to expect that the trained model is an $(\epsilon,\delta)$-approximate KKT point of Problem~\ref{eq:margmax}, for some small $\epsilon,\delta$.

\subsection{An objective for unlearning}

Let $S=\{(\bx_i,y_i)\}_{i=1}^m \subseteq \mathbb{R}^d \times \{-1,1\}$ be a dataset, and let $(\bx_r,y_r)$ be the example that we wish to unlearn. We call the dataset $S$ the \emph{original dataset}, and $S_{\text{retain}} = S \setminus \{(\bx_r,y_r)\}$ the \emph{retain dataset}.
Note that we focus on unlearning a single data point. In Section~\ref{sec:kset} we will consider unlearning a subset.

Following Theorem~\ref{thm:lyuli2}, we assume that we start from a trained model that is an $(\epsilon, \delta)$-approximate KKT point w.r.t. the original dataset. We also note that for the same reason, retraining for $S_{\text{retain}}$ will results in an $(\epsilon^*, \delta^*)$-approximate KKT point w.r.t. $S_{\text{retain}}$. Our objective can be stated as follows:
% \gilad{I'm not sure this is our goal, since we don't achieve it in the 2-layer case}
% \gal{Right. So how should we state it? Should we write "wish to obtain a model 'similar' to an $(\epsilon', \delta')$-approximate KKT"?}\gilad{Yes, that is better even if informal}
\begin{quote}
\emph{
    In the unlearning process, we wish to obtain a model that is close to an $(\epsilon^*, \delta^*)$-approximate KKT point w.r.t. the retain dataset, for some small $\epsilon^*, \delta^*$. 
}
\end{quote}

Indeed, in unlearning, we wish to find a model that is ``similar'' to a model that we could have learned if we had trained on the retain dataset in the first place,
%\om{It had been shown that such re-training will converge to an approximate KKT point w.r.t. the retained dataset after a finite number or training epochs \cite{lyu2019gradient}.}
and by Theorem~\ref{thm:lyuli2} such a model must be an $(\epsilon^*, \delta^*)$-approximate KKT point w.r.t. the retain dataset. 
Hence, 
%\om{being such an approximate-KKT point is a necessary condition for a re-learned network} and 
our objective can be viewed as a necessary condition for successful unlearning.
That is, a successful unlearning algorithm needs to obtain a network which is close to an approximate KKT point, since otherwise the network cannot be similar to a model which is retrained with the retained dataset.
%The property that we formally defined holds for every network retrained on the retained data, and thus it is a desired property of any unlearning algorithm.

%
% Our objective can also be viewed a sufficient for successful unlearning in the following sense. 
% Moreover, we note that
% if $\btheta$ is a KKT point of the retain dataset, then for some initializations, training the model on the retain dataset will converge to $\btheta$. Hence, our objective implies that the unlearned model is similar to a model that we can converge to when training on the retain dataset directly.
% \om{add : It had been shown that such re-training will converge to an approximate KKT point w.r.t. the retained dataset after a finite number or training epochs (see [1]). Thus, being such an approximate-KKT point is a necessary condition for a re-learned network, and therefore a successful unlearning algorithm should aim to obtain an approximate-KKT point for the retained data.}

More formally, we have the following definition: 
% \om{maybe add here a formal definition}
\begin{definition}[successful unlearning]\label{def:suc_unlearning}
    For a dataset $S$, and a homogeneous model with parameters $\btheta$, we say that $\mathcal{A}$ is an \textbf{$(\epsilon, \delta, \tau)$-successful} unlearning algorithm w.r.t. $\btheta$ and $S$, if for every point $(\bx_l, y_l) \in S$ there exists an $(\epsilon, \delta)$-approximate KKT point $\widetilde{\btheta}$ w.r.t. $S \setminus (\bx_l, y_l)$, such that 
    \[
    \cossim (\mathcal{A}(\btheta, S, l), \widetilde{\btheta}) \geq 1 - \tau~.
    \]
\end{definition}

We note that from \thmref{thm:lyuli2}, retraining for time $t$ is a \textbf{$(\epsilon_t, \delta_t, \tau)$-successful} unlearning algorithm with $\tau =0$ and $(\epsilon_t,\delta_t) \to (0,0)$. 
Our objective is to perform \textbf{$(\epsilon, \delta, \tau)$-successful} unlearning for small $(\epsilon, \delta, \tau)$ but in an efficient manner that avoids retraining from scratch. 

Definition~\ref{def:suc_unlearning} requires that the unlearned network $\mathcal{A}(\btheta, S, l)$ and the approximate KKT point $\widetilde{\btheta}$ have high cosine similarity. Indeed, note that since we consider homogeneous networks, the scale of the parameters only affects the scale of the output, and thus to show that $\mathcal{A}(\btheta, S, l)$ behaves similarly to $\widetilde{\btheta}$ it suffices to consider their corresponding normalized parameters. Moreover, for the normalized parameters, high cosine similarity implies small $\ell_2$ distance, and since the the model is Lipschitz w.r.t. the parameters, it implies a similar behavior.

\subsection{Unlearning with gradient ascent}

% \gal{Give references and explain that it is a very common approach to unlearning (right?)}

Consider a network $N(\bx,\btheta)$ trained on a dataset $S=\{(\bx_i,y_i)\}_{i=1}^m \subseteq \reals^d \times \{-1, 1\}$. In this paper, we consider the widely used \textit{Gradient Ascent} method for unlearning. In this method, to unlearn a training point $(\bx_r, y_r)$, we take a gradient step towards increasing the training loss for this particular point. Namely, for a step size $\beta$, the algorithm $\mathcal{A}_{\text{GA}}$ given $\btheta$, $S$ and $r$, performs the following
\begin{align}\label{eq:gradascent}
    \mathcal{A}_{\text{GA}}(\btheta,S,r) = \btheta + \beta \nabla_{\btheta} \ell(y_r N(\bx_r, \btheta))~.
\end{align}

Intuitively, training examples are often memorized in the sense that their training loss is too small, and gradient ascent allows us to undo it, that is, reduce the level of overfitting for these examples.

The gradient ascent method is a significant building block in the widely used unlearning method \emph{NegGrad}, that consists of multiple such steps, and is the unlearning approach also for other variants of it (such as \emph{NegGrad+} \citep{kurmanji2023towards} and \emph{advanced NegGrad} \citep{choi2023towards}) that %add a recovery goal, 
additionally perform
fine-tuning for the retained data.

In \secref{sec:linear} and \secref{sec:nonlin}, we demonstrate that in both linear predictors and two-layer ReLU networks, respectively, unlearning with a single step of gradient ascent ($\mathcal{A}_{\text{GA}}$) is \textbf{$(\epsilon, \delta, \tau)$-successful}, under certain assumptions.
\subsection{Data}
We consider a size-$m$ training set $S=\{(\bx_i,y_i)\}_{i=1}^m \subseteq \reals^d \times \{-1, 1\}$. We make the following assumption on $S$, for some parameters $\psi,\phi>0$.
\begin{assumption}
\label{ass:data}
The training set $S=\{(\bx_i,y_i)\}_{i=1}^m$ satisfies
\begin{enumerate}
    \item For all $(\bx,y) \in S$ we have $\norm{\bx}^2 \in [1-\psi, 1+\psi]$. 
    \item For all $(\bx_i,y_i), (\bx_j,y_j)\in S$ with $i \neq j$ we have $|\inner{\bx_i,\bx_j}| \leq \phi$. %=\frac{c\cdot log(d)}{\sqrt{d}} \leq \frac{1}{m}$
    % \item nearly orthogonal, namely $\norm{\bx_i}^2 > m \cdot |\inner{\bx_i,\bx_j}|$.
\end{enumerate}
\end{assumption}

% \gal{Discuss the assumption, and give examples for simple distributions and their corresponding $\psi,\phi$.}
%The assumption above is quite reasonable. 
The data normalization assumption (Item 1 above) is very common, as data points with significantly different norms might cause biases during training, toward higher norm data points. The latter assumption can be phrased as near orthogonality of the data points, which is also quite common in the literature for high dimensional data \citep{frei2022implicit, vardi2022gradient}, and holds with high probability for popular distributions. A profound example of a distribution that satisfies both conditions with high probability is the Gaussian distribution $\mathcal{N}(0, \frac{1}{d}I_d)$, where $d$ is the vector dimension. Another example is the uniform distribution over the unit sphere $\mathbb{S}^{d-1}$.
% \gilad{Another example is the uniform distribution over the unit sphere $\mathbb{S}^{d-1}$.} 

% \gilad{example clause}
\begin{example}\label{exp:gausiandata}
For a training set $S=\{(\bx_i,y_i)\}_{i=1}^m$ where the $\bx_i$'s are drawn i.i.d. from $\mathcal{N}(0, \frac{1}{d}I_d)$, %we get that with probability 
%$\geq 1 - (2e^{-d/500})$ we have $\norm{\bx_i}^2 \in [1-\psi, 1+\psi]$ for $\psi = 0.1$ (for a full proof see \lemref{lem:randomnormbounded}) and with probability $\geq 1- (e^{-d/500}+2d^{-\log(d)/2})$ we have that $|\inner{\bx_i,\bx_j}| \leq \phi$ for $\phi = 1.1 \frac{\log(d)}{\sqrt{d}}$ (see \lemref{lem:randominnermul}). Then, for such dataset of size $m$, $S=\{(\bx_i,y_i)\}_{i=1}^m$ satisfy \assref{ass:data} with probability 
\assref{ass:data} holds with probability at least
$1- (2me^{-d/500} + m^2 e^{-d/500} + 2m^2d^{-\frac{\log(d)}{2}})$, for $\psi = 0.1$ and $\phi = 1.1 \frac{\log(d)}{\sqrt{d}}$ (see \thmref{thm:data}). 
%Particularly, in \thmref{thm:data} we show that $S$ satisfies \assref{ass:data} for $\psi = 0.1$ and $\phi = \frac{\epsilon_d}{4mn}$ for any $n \in \nat$, for $d$ large enough.
Moreover, in Section~\ref{sec:generalization} we will show that \assref{ass:data} holds with high probability for a mixture of Gaussians. 
\end{example}

% \gal{In the proof, do we need the original network to be a KKT, or that it can be an approximate KKT?}

% s.t. $\forall i \in [m]$, $x_i \sim \mathcal{N}(0, \frac{1}{d})$. $m< \frac{\sqrt{d}}{}$ Then:

% \begin{enumerate}
%     \item For any $i \in [m]$, with probability $\geq 1-(2e^{-\frac{d}{500}})$, $\norm{x_i}^2 \in [0.9,1.1]$. Proven in \lemref{lem:randomnormbounded}
%     \item For all $i\neq j$, w.p. $\geq 1-(2e^{-\frac{d}{500}} + 2d^{-\frac{\log(d)}{2}})$, $x_i,x_j$ are nearly orthogonal, namely $\norm{x_i}^2 > m \cdot |\inner{x_i,x_j}|$.\\
%     In \lemref{lem:randomnormbounded} we show that $\norm{x_i}^2 \in [0.9,1.1]$ w.p. $\geq 1-(2e^{-\frac{d}{500}})$. In \lemref{lem:randominnermul} we show that $|\inner{\bx_i,\bx_j}| \leq 1.1 \frac{\log(d)}{\sqrt{d}}$ w.p. $\geq 1-(e^{-\frac{d}{500}} + 2d^{-\frac{\log(d)}{2}})$ (the $e^{-\frac{d}{500}}$ factor is due to using the norm bound). As conclusion, by our definition of $m$,  $x_i,x_j$ are nearly orthogonal with the mentioned probability. 
%     % we get that $|\inner{\bx_i,\bx_j}| \leq 1.1 \frac{log(d)}{\sqrt{d}}$. Proven in \lemref{lem:randominnermul}.
% \end{enumerate}

\section{Linear Predictors}
\label{sec:linear}

% delta epsilon
In this section, we consider a linear predictor $N(\bw, \bx) = \inner{\bw,\bx}$ trained on a dataset $S=\{(\bx_i,y_i)\}_{i=1}^m$.
%such that $\bw$ is an $(\epsilon, \delta)$-approximate KKT point for Problem~\eqref{eq:margmax}
%, and $S$ satisfies \assref{ass:data}. 
Recall that when training a linear predictor, gradient flow converges in direction to the max-margin solution (i.e., global optimum of Problem~\eqref{eq:margmax}), and after time $t$ it reaches an $(\epsilon_t,\delta_t)$-approximate KKT point of Problem~\eqref{eq:margmax} where $(\epsilon_t,\delta_t) \to (0,0)$ as $t \to \infty$. Moreover, recall that for linear predictors, Problem~\eqref{eq:margmax} has a unique global optimum.

The following theorem shows that unlearning using gradient ascent (denoted by $\mathcal{A}_{\text{GA}}$) is \textbf{$(\epsilon, \delta, \tau)$-successful} w.r.t. $S$ that satisfies \assref{ass:data} and $\bw$ which is an approximate KKT point according to \defref{def:epsdeltakkt}, in two distinct aspects. In the first part (item 1 below), we show it for $\tau =0$, that is, $\mathcal{A}_{\text{GA}}(\bw, S, l)$ is a linear predictor which is an approximate KKT point of the max-margin problem w.r.t. $S \setminus (\bx_l,y_l)$. Then, we show it for $\epsilon= \delta= 0$, namely, the cosine similarity of $\mathcal{A}_{\text{GA}}(\bw, S, l)$ and the max-margin predictor w.r.t. $S \setminus (\bx_l,y_l)$ is large.

%(such $\bw$ is retainable within a finite number of training steps). 
% Given the corresponding KKT multipliers for the constraints in the margin maximization problem \eqref{eq:margmax}, denoted by $\{\lambda_i\}_{i=1}^m$, we show that unlearning a data sample $(\bx_l,y_l) \in S$ using the gradient ascent unlearning method with step of size $\lambda_l$ preforms well. 
%
% Namely, given the KKT approximation parameters $\epsilon, \delta > 0$, and the data assumption parameters $\psi \leq 0.1$ and $\epsilon_d>0$ for which $\phi \leq \frac{\epsilon_d}{mn}$, we first show that preforming a gradient ascent step 
% %of size $\lambda_l$ 
% with an appropriate step size
% reaches an unlearned predictor that has the same direction as an $(O(\epsilon) , O(\frac{\epsilon_d}{m}+ \delta))$-approximate KKT point for the margin maximization problem (Problem~\eqref{eq:margmax}) w.r.t. the ramaining dataset $S \setminus (\bx_l,y_l)$. Second, we compare the unlearned predictor to a new predictor trained solely on the remaining dataset $S \setminus (\bx_l,y_l)$, namely, to a max-margin predictor for $S \setminus (\bx_l,y_l)$, and show that the two predictors are close in term of direction, having a cosine similarity of  $1 - O(\sqrt{\epsilon_d} + \sqrt{\epsilon} + \sqrt{\delta})$.

\begin{theorem}
\label{thm:linunlearning}
% \om{2 different epsilons, one for the data and one for the network approximation}
Let $0< \epsilon_1, \delta_1 \leq 0.5$, $\epsilon_d < 0.1$.  Let $\bx \mapsto \inner{\bw, \bx}$ be a linear predictor trained on dataset $S = \{(\bx_i,y_i)\}_{i=1}^m$, where $S$ satisfies \assref{ass:data} for $\psi \leq 0.1$ and $\phi \leq \frac{\epsilon_d}{4m}$. Assume that $\bw$ is an $(\epsilon_1, \delta_1)$-approximate KKT point for Problem~\eqref{eq:margmax} w.r.t. $S$ according to \defref{def:epsdeltakkt}.
%for $\lambda_1,...,\lambda_m$. 
%Fix some $l \in [m]$.
Then, 
% unlearning $(\bx_l,y_l)$ using 
the gradient ascent algorithm $\mathcal{A}_{\text{GA}}$, with an appropriate step size, is a \textbf{$(\epsilon, \delta, \tau)$-successful} unlearning algorithm  w.r.t. $\bw$ and $S$ 
% for an example $(\bx_l,y_l)$ 
for:
% \gilad{I think it is better to phrase it as "There exists an appropriate step size such that performing gradient ascent for unlearning using this step size we obtain a predictor..."}\om{maybe a bit risky since we want to avoid questions about finding it...}
% we denote by $\hat{\bw}$ an unlearned linear predictor created by performing gradient ascent unlearning of the data sample $(\bx_l,y_l) \in S$ 
%with step size determined by $\lambda_l$ (\eqref{eq:gradascent}). 
% Then, 
\begin{enumerate}
    \item \textbf{The case of $\epsilon=\epsilon_1 +\frac{\epsilon_1\epsilon_d}{m-\epsilon_d}$, $\delta=\delta_1 + \frac{\delta_1\epsilon_d}{m-\epsilon_d}+ \frac{7.2 \epsilon_d}{m}$, $\tau =0$:}\\
    The predictor $\mathcal{A}_{\text{GA}}(\bw, S, l)$ has the direction of an $(\epsilon, \delta)$-approximate KKT point for the margin maximization problem (Problem~\eqref{eq:margmax}) w.r.t. $S \setminus (\bx_l,y_l)$.
    % \item The predictor $\hat{\bw}$ has the direction of an $(\epsilon +\frac{\epsilon\epsilon_d}{m-\epsilon_d}, \delta + \frac{\delta\epsilon_d}{m-\epsilon_d}+ \frac{7.2 \epsilon_d}{m})$-approximate KKT point for the margin maximization problem (Problem~\eqref{eq:margmax}) w.r.t. $S \setminus (\bx_l,y_l)$.
    \item \textbf{The case of $\epsilon=\delta=0$, $\tau =C(\sqrt{\epsilon_d} + \sqrt{\epsilon_1} + \sqrt{\delta_1})$ for some universal constant $C>0$:} \\
    Let $\bw^*$ be a max-margin linear predictor w.r.t. the remaining training set $S \setminus (\bx_l,y_l)$, i.e. the global optimum of the Problem~\eqref{eq:margmax} w.r.t. $S \setminus (\bx_l,y_l)$. Then, $ \cossim(\mathcal{A}_{\text{GA}}(\bw, S, l), \bw^*) \geq 1 -\tau$.
    
    % \item Let $\bw^*$ be a max-margin linear predictor w.r.t. the remaining training set $S \setminus (\bx_l,y_l)$, i.e. the global optimum of the Problem~\eqref{eq:margmax} w.r.t. $S \setminus (\bx_l,y_l)$.
    %, and recall that by \citet{soudry2018implicit} gradient descent on $S \setminus (\bx_l,y_l)$ converges in direction to $\bw^*$. 
    % Then, $ \cossim(\hat{\bw}, \bw^*) \geq 1 -C(\sqrt{\epsilon_d} + \sqrt{\epsilon} + \sqrt{\delta})$ for some constant $C>0$.
    % \gal{check whether some dependence on $m$ is missing here}
    % \gal{needs to be rephrased. Something like: Let $\bw^*$ be the max-margin predictor w.r.t. the remaining dataset $S \setminus (\bx_l,y_l)$, i.e., the global optimum of Eq.(?), and recall that by Thm. ? GD on the remaining dataset converges in direction to $\bw^*$. Then, we have $\text{cossim}(\hat{\bw},\bw^*) \geq 1-\epsilon$.}
\end{enumerate}
\end{theorem}

We now briefly discuss the proof intuition. 
%First, when performing gradient ascent (as in Eq.~\ref{eq:gradascent}), we take a gradient of the loss (i.e., not the gradient of the predictor). Using the KKT multipliers $\{\lambda_i\}_{i \in [m]_{-l}}$ of $\bw$, our goal step size is proportionate to $\lambda_l$ multiplied by the gradient of the predictor $\bw$. Thus, we should adjust the step size according to the local derivative of the loss function at $(\bx_l, y_l)$. Such step size will induce a subtraction of $\lambda_l y_l \nabla_{\bw} N(\bw, \bx_l)$ from the parameter vector (Eq.~\ref{eq:gradascent}).
Due to the stationarity condition for $\bw$ (\defref{def:epsdeltakkt}), we can express $\bw$ as weighted sum of the network's gradient up to some error vector $\bv_{\epsilon_1}$ of norm $\epsilon_1$
\[
\bw = \sum\limits_{i=1}^{m} \lambda_i y_i \nabla_{\bw} N(\bw, \bx_i) + \bv_{\epsilon} = \sum\limits_{i=1}^{m} \lambda_i y_i \bx_i + \bv_{\epsilon_1}~.
\]

Then, by performing gradient ascent $\mathcal{A}_{\text{GA}}$ with the appropriate step size we get
\[
\mathcal{A}_{\text{GA}}(\bw, S, l) = \sum\limits_{i=1}^{m} \lambda_i y_i \nabla_{\bw} N(\bw, \bx_i) +\bv_{\epsilon} - \lambda_l y_l \nabla_{\bw} N(\bw, \bx_r) = \sum\limits_{i \in [m]_{-l}} \lambda_i y_i \bx_i +\bv_{\epsilon}~.
\]
% For the proof of (1), regarding the unlearned predictor $\mathcal{A}_{\text{GA}}(\btheta, S, l)$ as an approximate KKT point (\defref{def:epsdeltakkt}), 
First, one can see that the subtraction will result in a stationary condition w.r.t. $S \setminus (\bx_l,y_l)$ and the original $\lambda_i$'s. 
Observing the margin for a point $(\bx_t, y_t)$ (for $t \neq l$), 
\[
\inner{\bw,\bx_t} = \sum\limits_{i=1}^{m} \lambda_i y_i \inner{\bx_i, \bx_t} + \inner{\bv_{\epsilon_1}, \bx_t}~,
\]
% namely $\inner{\bw,\bx_t}$, 
we get that the change in the parameter vector (due to the gradient step) results in an additional term of at most $\lambda_l|\inner{\bx_l,\bx_t}|$ compared to the original predictor's margin. Due to the near-orthogonality of the data points in $S$ (\assref{ass:data}), and a constant upper bound for $\lambda_l$ which we prove, we get that this difference is of order $O(\frac{\epsilon_d}{m})$. 
% Note, since we suffer from this difference in both upper and lower bound of the margin, we slightly scale $\mathcal{A}_{\text{GA}}(\btheta, S, l)$ to meet the margin conditions for an approximate KKT (\defref{def:epsdeltakkt}).
%
Regarding the proof for (2), we consider the representation of $\bw^*$
\[
\bw^* = \sum\limits_{i=1}^{m} \lambda^*_i y_i \nabla_{\bw} N(\bw, \bx_i)  = \sum\limits_{i=1}^{m} \lambda^*_i y_i \bx_i~.
\]
% Regarding the proof for (2), we study $\bw^*$ and its corresponding KKT multipliers $\{\lambda^*_i\}_{i \in [m]_{-l}}$.
For $i \in [m]_{-l}$ we prove a small $O(\epsilon_1 + \epsilon_d)$ upper bound for the difference $\lambda^*_i - \lambda_i$, which implies that the two predictors  $\mathcal{A}_{\text{GA}}(\btheta, S, l)$ and $\bw^*$ independently reach very similar KKT multipliers for the margin maximization problem (\defref{def:epsdeltakkt}). This yield an $1- O(\sqrt{\epsilon_d} + \sqrt{\epsilon_1} + \sqrt{\delta_1})$ lower bound in the cosine similarity. For the full proof we refer the reader to \appref{app:linunlearning}.

\section{Two-Layer ReLU Networks}
\label{sec:nonlin}

In this section, we extend our analysis to two-layer ReLU neural networks. We consider a neural network of the form $N(\bx,\btheta) = \sum\limits_{j=1}^{n} u_j \sigma(\bw_j^\top \bx)$, trained on dataset $S=\{(\bx_i,y_i)\}_{i=1}^m$. 
Note that unlike the linear setting, the non-smoothness of $N(\bx,\btheta)$ implies that even small perturbations in $\btheta$ can cause significant shifts in the model's gradients. This introduces new challenges and, as a result, leads to a slightly weaker guarantee. 

The following theorem establishes that unlearning using gradient ascent with an appropriate step size, constitutes an \textbf{$(\epsilon, \delta, \tau)$-successful} unlearning w.r.t. $S$ that satisfies \assref{ass:data} and $\btheta$ which is an approximate KKT according to \defref{def:epsdeltakkt}, where $\epsilon$, $\delta$, and $\tau$ are small quantities determined by the KKT approximation parameters of $\btheta$ and the underlying data characteristics. This implies that the unlearned parameter vector $\mathcal{A}_{\text{GA}}(\btheta, S, l)$ is close—in terms of cosine similarity—to an approximate KKT point $\widetilde{\btheta}$ corresponding to the retained dataset $S \setminus {(\bx_l, y_l)}$.

% To address this sensitivity, we construct a modified parameter vector $\widetilde{\btheta}$, which remains close to $\hat{\btheta}$ while preserving the same activation map as the original $\btheta$. We demonstrate that $\widetilde{\btheta}$ is an approximate KKT point of Problem~\eqref{eq:margmax}, and retains a high cosine similarity to the unlearned solution $\hat{\btheta}$.

% \gilad{Explain the problem with extending the linear case to a 2-layer network: The gradient are not Lipschitz, they are step functions. But, it is still possible to show it by adding a small fix.}

\begin{theorem}
\label{thm:nonlinunlearning}
    Let $0<\epsilon_1, \delta_1 \leq 1$, $0<\epsilon_d\leq 0.01$. 
    Let 
    $N(\bx,\btheta)  = \sum\limits_{j=1}^{n} u_j \sigma(\bw_j^\top \bx)$ 
    %$N(\bx,\btheta)$
    be a two-layer ReLU network as defined in Eq.~\eqref{eq:NN}, 
    %trained on $S = \{(\bx_i,y_i)\}_{i=1}^m$, such that 
    %$\forall j \in [n], u_j \sim \mathcal{U}\{-\frac{1}{\sqrt{n}}, \frac{1}{\sqrt{n}}\}$, and 
    such that 
    $\btheta$ is an $(\epsilon_1, \delta_1)$-approximate KKT point for Problem~\eqref{eq:margmax} w.r.t. $S = \{(\bx_i,y_i)\}_{i=1}^m$ according to \defref{def:epsdeltakkt},
    %for $\lambda_1,...,\lambda_m$, 
    and suppose that $S$ satisfies \assref{ass:data} for $\psi \leq 0.1$ and $\phi \leq \frac{\epsilon_d}{4mn}$. 
    %Fix $l \in [m]$. 
    % we denote by $\hat{\btheta}$ the parameters created by performing gradient ascent on the first layer weights, for the data sample $(\bx_l,y_l) \in S$ with step size determined by $\lambda_l$ (\eqref{eq:gradascent}). 
    Then, the gradient ascent algorithm $\mathcal{A}_{\text{GA}}$ 
    %on the first layer weights 
    with an appropriate step size 
    %(Eq.~\eqref{eq:gradascent}), 
    is a \textbf{$(\epsilon, \delta, \tau)$-successful} unlearning algorithm w.r.t. $\btheta$ and $S$, for $\epsilon=\epsilon_1 + \frac{9\epsilon_d\epsilon_1}{m-9\epsilon_d} + \frac{23\epsilon_d}{\sqrt{m}}$, $\delta=\delta_1 + \frac{9\epsilon_d\delta_1}{m-9\epsilon_d}+ \frac{22.6\epsilon_d}{m}$ and $\tau = \frac{82\epsilon_d}{m}$.
    
    % Then, there exists $\widetilde{\btheta}$, such that $\widetilde{\btheta}$ has the direction of a $(\epsilon + \frac{9\epsilon_d\epsilon}{m-9\epsilon_d} + \frac{23\epsilon_d}{\sqrt{m}}, \delta + \frac{9\epsilon_d\delta}{m-9\epsilon_d}+ \frac{22.6\epsilon_d}{m})$-approximate KKT point for the margin maximization problem (\eqref{eq:margmax}) w.r.t. $S \setminus (\bx_l,y_l)$ and $\cossim(\hat{\btheta} , \widetilde{\btheta}) \geq 1- \frac{82\epsilon_d}{m}$. 
    % \gilad{Since the cossim is with a specific constant, is it possible to also write the approximate KKT point with a specific constant?} 
\end{theorem}

\begin{figure}[t]
  \centering
  % \fbox{\rule[-.5cm]{0cm}{4cm} \rule[-.5cm]{4cm}{0cm}}
  \includegraphics[width=8cm]{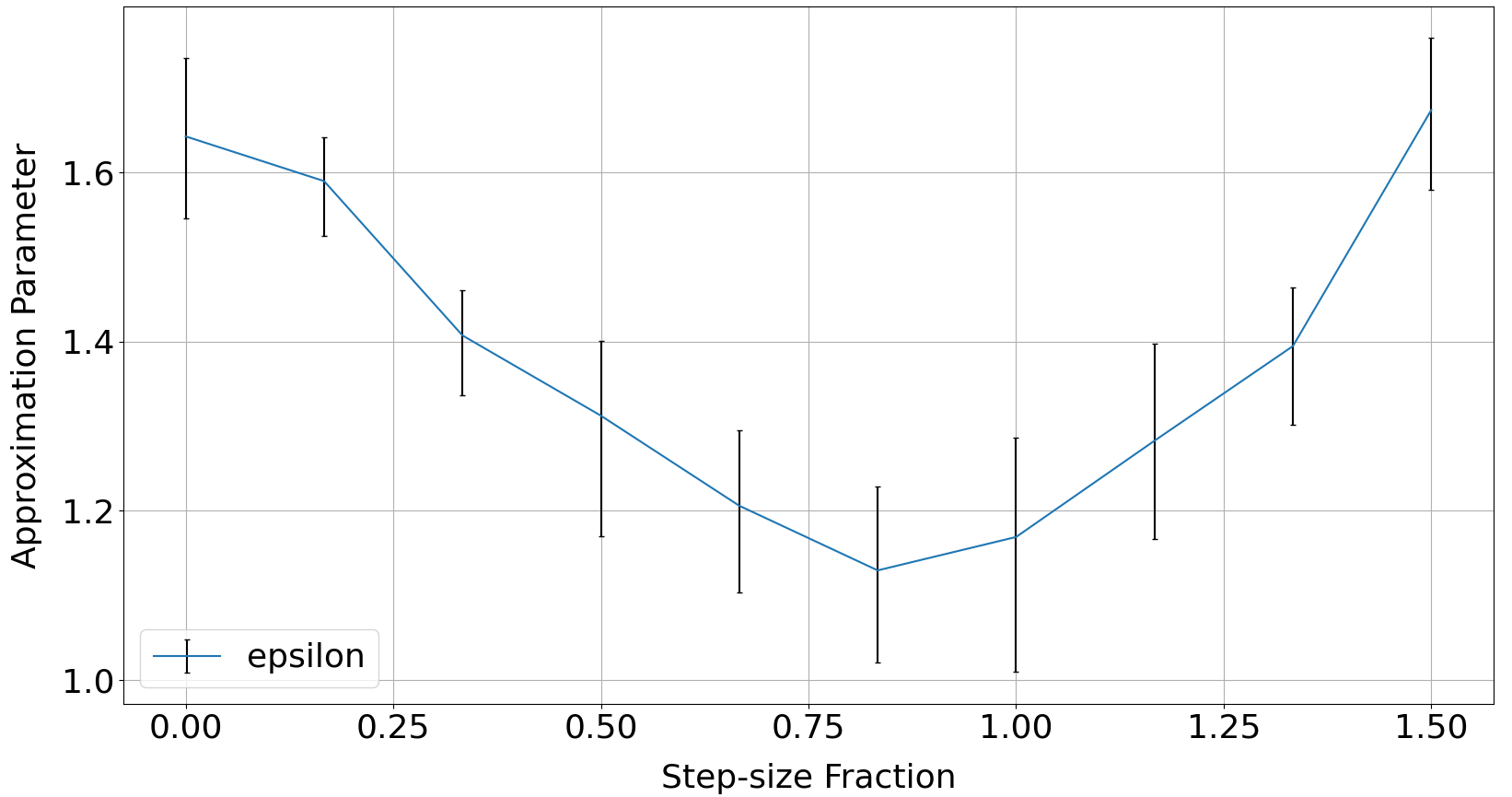}
  \caption{Effect of deviation from the correct step size on the KKT approximation parameter $\epsilon$ for a two-layer network. The $x$-axis shows the step size as a fraction of the step size from \thmref{thm:nonlinunlearning}, and the $y$-axis shows 
  %the resulting $\epsilon$ value for the \textbf{$(\epsilon, \delta, \tau)$-successful} unlearning algorithm $\mathcal{A}_{\text{GA}}$ for some $(\bx_l,y_l)\in S$ (i.e., the 
  the KKT approximation parameter $\epsilon$ 
  of the unlearned model w.r.t. the retain dataset.}
  \label{fig:lambda}
\end{figure}

In \Figref{fig:lambda}, we show the effect of varying the step size around the appropriate value $\beta_l$ from \thmref{thm:nonlinunlearning} when unlearning a point $(\bx_l, y_l) \in S$. The x-axis represents the step size as a fraction of $\beta_l$, and the y-axis shows the resulting KKT approximation parameter $\epsilon$ w.r.t. the retain dataset. We use a two-layer network (Eq.~\eqref{eq:NN}) trained on a $10$-point dataset in $\reals^{1000}$, and apply $\mathcal{A}_{\text{GA}}(\btheta, S, l)$ to a random data point. We can see that significantly deviating for $\beta_l$ results in a worse approximation variable. See \appref{app:exp} for more details.

% The intuition is similar as well. While the stationarity condition of the $\btheta$ parameter vector denote it as a sum of $m$ summons, any approximate KKT point w.r.t. the retain dataset will consists of only $m-1$, creating the unavoidable gap in the cosine similarity of the two parameter vectors. Full proof is in \appref{app:nonlinidentity}.

\subsection{Proof sketch}

We now outline the main ideas behind the proof.
In this setting, unlike the linear setting, comparing the original parameter vector $\btheta$ with the unlearned parameter vector $\mathcal{A}_{\text{GA}}(\btheta, S, l)$ is nontrivial. Although the unlearning procedure introduces only a small perturbation, it may lead to significant changes in the activation map—the pattern of neuron activations across the data. Specifically, we define the activation map as the set of neurons $\bw_j$ that are active on a data point $\bx_i$, i.e., $\inner{\bw_j, \bx_i} \geq 0$. A key challenge arises when even small weight changes cause certain neurons to flip activation status.

To address this, we introduce an additive correction term (or "fix") for each weight vector $\bw_j$, for $j \in [n]$, that restores the activation pattern. Using the stationarity conditions satisfied by $\btheta$ (\defref{def:epsdeltakkt}), we express each $\bw_j$ as a weighted sum of the network's gradients, up to a small error term $\bv_{\epsilon_1,j}$:
    \[
    \bw_j = \sum\limits_{i=1}^{m} \lambda_i y_i \nabla_{\bw_j} N(\bx_i, \btheta) + \bv_{\epsilon,j} = u_j \sum\limits_{i=1}^{m}  \lambda_i y_i \sigma'_{i,j} \bx_i + \bv_{\epsilon_1,j}
    \]

where $\sigma'_{i,j}$ denotes the local derivative of the activation function.

After applying the gradient ascent step, the contribution of the forgotten point $(\bx_l, y_l)$ is removed, which may alter the activation state of some neurons. To mitigate this, we construct a correction vector using a small scaling factor $c = O\left(\frac{\epsilon_d}{mn}\right)$, forming a new weight vector:
\[
\widetilde{\bw}_j = \bw_j - u_j \lambda_l y_l \sigma'_{l,j} \bx_l + |u_j| \lambda_l \sigma'_{l,j} c \sum\limits_{k \in [m]_{-l}}  \bx_k \sign(\inner{\bx_k, \bw_j})~.
\] 
This correction reintroduces a small averaged influence from the retained points, specifically those where $\bw_j$ was previously active.
For a data point $\bx_t$ where $\bw_j$ was originally active, the new inner product becomes:

\[
\inner{\widetilde{\bw}_j, \bx_t} = \inner{\bw_j,\bx_t} - u_j \lambda_l y_l \sigma'_{l,j} \inner{\bx_l, \bx_t} + |u_j| \lambda_l \sigma'_{l,j} c \sum\limits_{k \in [m]_{-l}}  \inner{\bx_k,\bx_t} \sign(\inner{\bx_k, \bw_j})~.
\]
Since the data points $\bx_l$ and $\bx_t$ are nearly orthogonal (i.e., $\inner{\bx_l, \bx_t} = O(\frac{\epsilon_d}{mn})$, see \assref{ass:data}), the middle term is of the same order as the correction, thus the correction term restores the activation. As a result, the corrected weight vector $\widetilde{\bw}_j$ remains active on $\bx_t$, preserving the original activation map.
This activation preservation is essential: it enables us to meaningfully compare $\btheta$ and $\widetilde{\btheta}$ in terms of margin, gradient differences, and parameter norms, facilitating the rest of the proof.

In establishing stationarity, the fixed vector introduces an additional error term beyond the original stationarity bound. In addition, because the activation map is preserved, we can upper bound the change in the margins of the remaining data points by a small factor of order \( O\left(\frac{\epsilon_d}{mn}\right) \). Similar to the linear case, this margin deviation appears in both the upper and lower bounds, so we slightly rescale \( \widetilde{\btheta} \) to restore feasibility and obtain an approximate KKT point for Problem~\eqref{eq:margmax} with respect to the reduced dataset \( S \setminus \{(\bx_l, y_l)\} \). To complete the proof, we show that \( \mathcal{A}_{\text{GA}}(\btheta, S, l) \) remains close—in cosine similarity—to the rescaled \( \widetilde{\btheta} \), differing only by the small fix and the minor scaling. The complete proof is provided in Appendix~\ref{app:nonlinunlearning}.

\section{Unlearning batches of data points}\label{sec:kset}
In the previous sections, we analyzed the unlearning of a single data point. We now extend 
%The results in \secref{sec:linear} and \secref{sec:nonlin} 
these results
to the case of unlearning a set of data points. Let $S_{\text{forget}} \subseteq S$ denote a subset of size $k$. We unlearn $S_{\text{forget}}$ using a natural extension of the $\mathcal{A}_{\text{GA}}$ algorithm, namely by performing a step that consists of the $k$ gradients of the points in $S_{\text{forget}}$, with appropriate coefficients. We denote this algorithm by $\mathcal{A}_{\text{k-GA}}$.
Formally, for some real coefficients $\{\beta_r\}$, the algorithm $\mathcal{A}_{\text{k-GA}}$ performs the following
\begin{align*}\label{eq:gradascent}
    \mathcal{A}_{\text{k-GA}}(\btheta,S,S_{\text{forget}}) = \btheta + \sum_{(\bx_r,y_r) \in S_{\text{forget}}} \beta_r \nabla_{\btheta} \ell(y_r N(\bx_r, \btheta))~.
\end{align*}

In the case of linear predictors, the algorithm $\mathcal{A}_{\text{k-GA}}$ still satisfies the result from Theorem~\ref{thm:linunlearning}, but with slightly modified additive terms in the bounds on the KKT-approximation parameters $\epsilon,\delta$,
% , becoming $\frac{\epsilon_1\epsilon_d}{\frac{m}{k}-\epsilon_d}$ and $\frac{\delta_1\epsilon_d}{\frac{m}{k}-\epsilon_d}$ respectively. Nonetheless, both remain upper bounded by $C_1(\epsilon_1\epsilon_d + \delta_1\epsilon_d)$ for some constant $C_1>0$ as long as $\frac{m}{k}> \epsilon_d$. 
while the bound on the cosine similarity (i.e., the parameter $\tau$) remains unchanged. See a formal statement and proof in \appref{app:linunlearning_subset}.
%, since the factor $C(\sqrt{\epsilon_d} + \sqrt{\epsilon_1} + \sqrt{\delta_1})$ arises solely from bounding the deviation $\lambda^*_i - \lambda_i$, which is independent of the size of the forget set. The formal theorem and proof can be found in \appref{app:linunlearning_subset}. 

% Moreover, f
For two-layer networks, we show that the result from Theorem~\ref{thm:nonlinunlearning} holds when unlearning a subset $S_{\text{forget}}$ using the algorithm $\mathcal{A}_{\text{k-GA}}$, but with slightly modified parameters $\epsilon,\delta,\tau$. See \appref{app:nonlinunlearning_subset} for the formal statement and proof.

\section{Generalization of the Unlearned Classifier}
\label{sec:generalization}
In this section, we show that if $\btheta$ satisfies \defref{def:epsdeltakkt} and the dataset $S$ satisfies \assref{ass:data}, then unlearning via a single gradient ascent step (i.e., $\mathcal{A}_{\text{GA}}$) may not harm generalization. 
%This result suggests a potential connection between the approximate satisfaction of the KKT conditions and the generalization ability of the resulting model. 
As a concrete example, we consider a data distribution $\mathcal{D}_{MG}$ such that a dataset from this distribution satisfies w.h.p. Assumption~\ref{ass:data} with parameters $\psi \leq 0.1$ and $\phi \leq \frac{\epsilon_d}{4mn}$. The distribution consists of two opposite Gaussian clusters, such that the cluster means have magnitude $d^{-\alpha}$ for some $\alpha \in (0, \frac{1}{4})$, and each deviation from the mean is drawn as $\bzeta \sim \mathcal{N}(0, \frac{1}{d}I_d)$.
% , resulting in a width on the order of $O\left(\frac{\log d}{\sqrt{d}}\right)$\gal{what is width?}. 
%Despite the fact that the cluster variance exceeds the mean separation, 
We show that both the original model and the unlearned model can generalize well, that is, classify the clusters with high probability. 

% and its $\epsilon$-close neighbor $\widetilde{\btheta}$ that has the direction of a $(\epsilon, \epsilon)-$approximate KKT point for the margin maximization problem w.r.t. $S \setminus (\bx_l,y_l)$. We show that $\widetilde{\btheta}$ generalize well \gal{why $\widetilde{\btheta}$ and not $\hat{\btheta}$?}, demonstrating the connection between an approximate KKT point of the margin maximization problem and its ability to generalize.

% \subsection{Data}\label{def:exampledata}

Formally, our data satisfies the following.\label{def:exampledata}
we denote the dataset by $S= \{(\bx_i,y_i)\}_{i=1}^m \sim \mathcal{D}_{MG}^m$, where $\forall i \in [m], (\bx_i,y_i) \in \reals^d \times \{-1,1\}$, and where $\mathcal{D}_{MG}$ is detailed as follows. It consists of a mixture of two Gaussians with means $\bmu_+, \bmu_- \in \reals^d$, such that $\norm{\bmu_+} = d^{-\alpha}$ for $\alpha \in (0,\frac{1}{4})$, and $\bmu_- = - \bmu_+$. For each $i$, we choose $\bmu_i \sim \mathcal{U}\{\bmu_+, \bmu_-\}$, then $\bx_i \sim \mathcal{N}(\bmu_i, \frac{1}{d}I_d)$ and finally $y_i =1$ if $\bmu_i = \bmu_+$ and $-1$ otherwise. Note that we can denote $\bx_i = \bmu_i + \bzeta_i$ where $\bzeta_i \sim \mathcal{N}(\mathbf{0}, \frac{1}{d}I_d)$. We refer the reader to \lemref{lem:gen_dataass}, where we prove that for a given $\epsilon_d > 0$, $m$ and $\alpha$, $S$ satisfies \assref{ass:data} for $\psi \leq 0.1$ and $\phi \leq \norm{\bmu_i}^2 + 2\norm{\bmu_i} \frac{\log(d)}{\sqrt{d}} +   1.1 \frac{\log(d)}{\sqrt{d}} \leq \frac{\epsilon_d}{4mn}$, w.p. $\geq 1- (2me^{-\frac{d}{1700}}+ m^2 e^{-d/500} + 2m^2d^{-\frac{\log(d)}{2}})$ for large enough $d$. 

The following theorem shows that the unlearned network achieves generalization bounds comparable to those of the original classifier. Combined with the fact that it is close to an approximate KKT point of Problem~\eqref{eq:margmax} with respect to the retained dataset (as established in Theorem~\ref{thm:nonlinunlearning}), this 
%supports a broader connection between approximate satisfaction of the KKT condition and generalization performance for both the original and the unlearned models.
demonstrates a clean setting where unlearning is successful, and it does not hurt generalization.

% \[ \bmu^2 m < 1\]
% \om{w.h.p. the dataset satisfied assumption, w.h.p. the new sample is also good}
\begin{theorem}
\label{thm:regulargenerlization}
    Let $0<\epsilon_d\leq 0.01$. 
    Let 
    $N(\bx,\btheta)  = \sum\limits_{j=1}^{n} u_j \sigma(\bw_j^\top \bx)$ 
    %$N(\bx,\btheta)$
    be a two-layer ReLU network as defined in Eq.~\eqref{eq:NN}, 
    %trained on $S = \{(\bx_i,y_i)\}_{i=1}^m$, such that 
    %$\forall j \in [n], u_j \sim \mathcal{U}\{-\frac{1}{\sqrt{n}}, \frac{1}{\sqrt{n}}\}$, and 
    such that 
    $\btheta$ is a KKT point for Problem~\eqref{eq:margmax} w.r.t. $S= \{(\bx_i,y_i)\}_{i=1}^m \sim \mathcal{D}_{MG}^m$ according to \defref{def:epsdeltakkt}.
    %for $\lambda_1,...,\lambda_m$, 
    % and suppose that $S$ satisfies \assref{ass:data} for $\psi \leq 0.1$ and $\phi \leq \frac{\epsilon_d}{4mn}$\gal{why do we need to assume it? doesn't it hold w.h.p.?}.  
    %\om{Let $S_{\text{forget}} \subset S$ denote a subset of size $k$. Denote by $\mathcal{A}_{\text{k-GA}}(\btheta, S, l)$ the parameters vector obtained by the extended algorithm $\mathcal{A}_{\text{k-GA}}$ for the set $S_{\text{forget}} \subset S$ with the appropriate step size.}    
    Fix $l \in [m]$ and denote by $\mathcal{A}_{\text{GA}}(\btheta, S, l)$ the parameters vector obtained by the gradient ascent algorithm $\mathcal{A}_{\text{GA}}$ for the data point $(\bx_l,y_l) \in S$ with the appropriate step size from Theorem~\ref{thm:nonlinunlearning}. Then, w.p. $\geq 1- (2me^{-\frac{d}{1700}}+ m^2 e^{-d/500} + 2m^2d^{-\frac{\log(d)}{2}})$ over the choice of the dataset $S$,
    % \gal{over what? the training dataset $S$?},
    both
    $N(\bx,\btheta)$ and $N(\bx, \mathcal{A}_{\text{GA}}(\btheta, S, l))$ generalize. Namely, 
    \[\Pr_{(\bx_t,y_t) \sim \mathcal{D}_{MG}} \left[ y_t N(\bx_t, \btheta) > 0 \right] \geq 1- (2e^{-\frac{d}{1700}}+ me^{-d/500} + 2md^{-\frac{\log(d)}{2}}) ~,   \]
    \[\Pr_{(\bx_t,y_t) \sim \mathcal{D}_{MG}} \left[ y_t N(\bx_t, \mathcal{A}_{\text{GA}}(\btheta, S, l)) > 0 \right] \geq 1- (2e^{-\frac{d}{1700}}+ me^{-d/500} + 2md^{-\frac{\log(d)}{2}})~.    \]
    
\end{theorem}

We briefly outline the intuition behind the generalization proof. Due to the small cluster means and relatively large variance, the data points in $S$ are nearly orthogonal. 
%(see \ref{def:exampledata}). 
%We refer to the cluster mean of a point as its signal, and show that this signal dominates over the Gaussian noise. 
Although the deviation from orthogonality is small, it is crucially structured: the inner product sign is determined by whether two points belong to the same or different clusters, namely
\begin{align*}
\bx_i, \bx_j \text{~~are in the same cluster~~} \Rightarrow \inner{\bx_i, \bx_j} > 0 ~,\\
\bx_i, \bx_j \text{~~are in different clusters~~} \Rightarrow \inner{\bx_i, \bx_j} < 0~.
\end{align*}

Now, using the fact that the classifier $\btheta$ satisfies the stationarity conditions with respect to $S$ (\defref{def:epsdeltakkt}), we denote it by the weighted sum of its gradients direction, and consider its inner product with some $\bx_t \sim \mathcal{D}_{MG}$
    \[
    \inner{\bw_j, \bx_t} = \inner{\sum\limits_{i=1}^{m} \lambda_i y_i \nabla_{\bw_j} N(\bx_i, \btheta),\bx_t}  = u_j \sum\limits_{i=1}^{m}  \lambda_i y_i \sigma'_{i,j} \inner{\bx_i, \bx_t} ~.
    \]
Since the inner product and the label align, we get that the activation map is of the same sign as $u_j$, hence each training point contributes positively to the classification of other points in the same cluster, and negatively to the others.
% Leading to the margin $y_t N(\btheta,\bx_t) = y_t \sum\limits_{j=1}^{n} u_j \sigma(\bw_j^\top \bx_t)$
%
% Since these directions are determind by the training data points, we conclude that each training point contributes positively to the classification of other points in the same cluster. 
This similarity of contribution implies that removing a point from $S$ during unlearning does not significantly degrade the model’s classification accuracy. The full proof is provided in Appendix~\ref{app:generalization}.\
% gal{not very clear}
Finally, we note that Theorem~\ref{thm:regulargenerlization} can be readily extended to the case of unlearning a subset of data points using the algorithm $\mathcal{A}_{\text{k-GA}}$ discussed in Section~\ref{sec:kset}.

\section{Discussion and future work}
In this work, we analyze the theoretical effectiveness of a single gradient-ascent step as a machine unlearning algorithm. Focusing on post-training unlearning methods, we propose a new criterion for unlearning success—called \textbf{$(\epsilon, \delta, \tau)$-successful} unlearning—based on approximate satisfaction of KKT conditions. We prove that, in both linear models and two-layer neural networks, applying a gradient-ascent step $\mathcal{A}_{\text{GA}}$ with an appropriate step size w.r.t. the point we wish to forget is a \textbf{$(\epsilon, \delta, \tau)$-successful} unlearning algorithm with some small $\epsilon, \delta, \tau$, for a dataset $S$ that satisfies \assref{ass:data} and a parameter vector $\btheta$ that is an approximate KKT point according to \defref{def:epsdeltakkt}. 
%Specifically, it produces a parameter vector that approximates the retrained solution on the retained set. 
In the linear case, we additionally achieve near-exact recovery of the margin-maximizing predictor, implying stronger unlearning guarantees. 
%We further demonstrate that the unlearned models preserve generalization ability, establishing a connection between KKT satisfaction and out-of-sample performance. 
We also demonstrate a clean distribution where unlearning is both successful and does not hurt generalization.
Together, our results offer a rigorous foundation for analyzing gradient-based unlearning and confirm the practical utility of this simple yet widely used technique.

This work opens several avenues for further exploration. First, while we focus on a gradient-ascent step, it would be valuable to analyze the effect of an additional recovery phase for the retain data, including those used in NegGrad+ and related variants, under the same KKT-based framework. Second, it would be interesting to develop tighter bounds connecting approximate KKT satisfaction with practical privacy metrics, such as membership inference risk. On the applied side, evaluating unlearning methods under the new success criterion can lead to interesting comparisons between different methods.
%, that might differ in various parameters, suggesting different solution for different desired $\epsilon$, $\delta$ and $\tau$. 
Moreover, a broader integration of our theoretical criterion with empirical privacy guarantees (e.g., differential privacy) could help bridging the gap between formal definitions and real-world deployment in safety-critical applications. Finally, extending our results to deeper architectures and 
%more complex loss landscapes r
additional distributions
remains an important challenge.
%, particularly where non-convexity and over parameterization affect the stability of unlearning.

\subsubsection*{Acknowledgments}

GV is supported by The Israel Science Foundation (grant No. 2574/25), by a research grant from Mortimer Zuckerman (the Zuckerman STEM Leadership Program), and by research grants from the Center for New Scientists at the Weizmann Institute of Science, and the Shimon and Golde Picker -- Weizmann Annual Grant.

\bibliography{bib}
% \bibliography{iclr2026_conference}
\bibliographystyle{iclr2026_conference}

\newpage
\appendix
\section{Proofs of data preliminaries for \secref{sec:settings}}

\begin{theorem} \label{thm:data}
    Let a set $S=\{(\bx_i,y_i)\}_{i=1}^m$ such that $\forall i$, $\bx_i \in \reals^d$ and $\bx_i \sim \mathcal{N}(0, \frac{1}{d}I_d)$, $y_i \in \{-1,1\}$ and $n \in \nat$. Then, w.p. $\geq 1- (2me^{-d/500} + m^2 e^{-d/500} + 2m^2d^{-\frac{\log(d)}{2}})$, the dataset $S$ satisfies \assref{ass:data} for $\psi = 0.1$ and $\phi =  1.1 \frac{\log(d)}{\sqrt{d}}$. 
\end{theorem}

\begin{proof}
\assref{ass:data} have 2 conditions:
\begin{enumerate}
    \item For all $(\bx,y) \in S$, $\norm{\bx}^2 \in [1-\psi, 1+\psi]$.\\
    Follows from \lemref{lem:setofrandomnorm} w.p. $\geq 1- 2me^{-\frac{d}{500}} $.
    \item For all $(\bx_i,y_i), (\bx_j,y_j)\in S$ s.t. $i \neq j$, $|\inner{\bx_i,\bx_j}| \leq \phi$.\\
    From \lemref{lem:setofrandominner} we have that w.p. $\geq 1- (m^2 e^{-d/500} + 2m^2d^{-\frac{\log(d)}{2}})$,For all $(\bx_i,y_i), (\bx_j,y_j)\in S$
    \[|\inner{\bx_i,\bx_j}| \leq 1.1 \frac{\log(d)}{\sqrt{d}}~.\]
\end{enumerate}
% For a given dataset size $m$, and natural number $n$ and $\epsilon_d$ there exists $d$ large enough such that 
% \[ 1.1\frac{\log(d)}{\sqrt{d}} \leq \frac{\epsilon_d}{4mn} \]
% since $\lim_{d \rightarrow \infty} \frac{\log(d)}{\sqrt{d}} = 0$.
\end{proof}

% \gilad{Add a theorem here stating that Gaussian distribution satisfies Assumption 3.2 w.p $> 1- ??$ and $\psi=?, \phi=?$. The proof can be "follows directly from lemma ??".}

\begin{lemma}
\label{lem:randomnormlarge}

    Let $w\in\reals^n$ such that $w \sim \mathcal{N}(\zero,\sigma^2 I_n)$. Then:

        \[ \mathbb{P} \left[\norm{w}^2 \leq 0.9 \sigma^2 n \right] \leq e^{-\frac{n}{400}}~.\] 
\end{lemma}

\begin{proof}
	Note that $\norm{\frac{w}{\sigma}}^2$ has the Chi-squared distribution. A concentration bound by Laurent and Massart \citep[Lemma 1]{laurent2000adaptive} implies that for all $t > 0$ we have
	\[
	 	\Pr\left[ n - \norm{\frac{w}{\sigma}}^2 \geq 2\sqrt{nt} \right] \leq e^{-t}~.
	\]
	Plugging-in $t=c \cdot n$, we get
	\begin{align*}
        \Pr\left[ n - \norm{\frac{w}{\sigma}}^2 \geq 2 \sqrt{c} n \right] 
		= \Pr\left[\norm{\frac{w}{\sigma}}^2 \leq (1 - 2 \sqrt{c})n  \right] 
		\leq e^{-c \cdot n }~.
	\end{align*}
	Thus, we have for $c = \frac{1}{400}$
	\begin{align*}
        \Pr\left[\norm{\frac{w}{\sigma}}^2 \leq (1 - 2 \frac{1}{\sqrt{400}})n  \right]
        = \Pr\left[\norm{\frac{w}{\sigma}}^2 \leq \frac{9}{10}n  \right] 
		\leq e^{-\frac{n}{400}}~.
	\end{align*}
    And finally, 
	\begin{equation*} 
		\Pr\left[ \norm{w}^2  \leq  \frac{9}{10} \sigma^2 n \right] 
		\leq e^{-\frac{n}{400}}~.
	\end{equation*}
\end{proof}

\begin{lemma}
\label{lem:randomnormsmall}
    Let $w \in \reals^n$ with $w \sim \mathcal{N}(\zero,\sigma^2 I_n)$. Then:
        \[
        \Pr \left[\norm{w}^2 \geq 1.1 \sigma^2 n \right] \leq e^{-\frac{n}{500}}~.
        \] 
\end{lemma}
\begin{proof}
	Note that $\norm{\frac{w}{\sigma}}^2$ has the Chi-squared distribution. A concentration bound by Laurent and Massart \citep[Lemma 1]{laurent2000adaptive} implies that for all $t > 0$ we have
	\[
	 	\Pr\left[ \norm{\frac{w}{\sigma}}^2 - n \geq 2\sqrt{nt} + 2t \right] \leq e^{-t}~.
	\]
	Plugging-in $t=c \cdot n $, we get
	\begin{align*}
        % \Pr\left[ \norm{\frac{w}{\sigma}}^2  \geq  2n \right] 
		% \leq 
        \Pr\left[ \norm{\frac{w}{\sigma}}^2 - n \geq  2 \sqrt{c} n + 2cn \right] 
        = \Pr\left[ \norm{\frac{w}{\sigma}}^2 
        \geq  (2 \sqrt{c}+ 2c +1 ) n \right]
		\leq e^{c \cdot n}~.
	\end{align*}
	Thus, we have for $c = \frac{1}{500}$
    \begin{align*}
        % \Pr\left[ \norm{\frac{w}{\sigma}}^2  \geq  2n \right] 
		% \leq 
        \Pr\left[ \norm{\frac{w}{\sigma}}^2  \geq  1.1 n \right] 
        = \Pr\left[ \norm{\frac{w}{\sigma}}^2 
        \geq  (2 \frac{1}{\sqrt{500}}+ \frac{2}{500} +1 ) n \right]
		\leq e^{-\frac{n}{500}}~.
	\end{align*}
    And finally,
	\begin{equation*} 
		\Pr\left[ \norm{w}^2  \geq 1.1 \sigma^2 n \right] 
		\leq e^{-\frac{n}{500}}~.
	\end{equation*}
\end{proof}

\begin{lemma}
    \label{lem:randomnormbounded}
    For any $i \in [m]$, with probability $\geq 1-(2e^{-\frac{d}{500}})$, $\norm{x_i}^2 \in [0.9,1.1]$.
\end{lemma}
\begin{proof}
    Using \lemref{lem:randomnormlarge} to lower bound $\norm{x_i}^2$ for $x_i \sim \mathcal{N}(0,\frac{1}{d})$ w.p. $\geq 1-e^{-\frac{n}{400}}$, and use \lemref{lem:randomnormsmall} to upper bound $\norm{x_i}^2$  w.p. $\geq 1-e^{-\frac{n}{500}}$.
\end{proof}

\begin{lemma}
\label{lem:randominnermul_1}
    Let $u \in \reals^n$, and $v \sim \mathcal{N}(\zero,\sigma^2 I_n)$. Then, for every $t>0$ we have
        \item \[\Pr \left[| \inner{u,v}| \geq \norm{u} t  \right] \leq 2 \exp \left(-\frac{t^2}{2\sigma^2}\right)~.\]
\end{lemma}
\begin{proof}
We first consider $\inner{\frac{u}{\norm{u}},v}$. As the distribution $\mathcal{N}(\zero,\sigma^2 I_n)$ is rotation invariant, one can rotate $u$ and $v$ to get $\Tilde{u}$ and $\Tilde{v}$ such that $\Tilde{\frac{u}{\norm{u}}}= e_1 $, the first standard basis vector and $\inner{\frac{u}{\norm{u}},v} = \inner{\Tilde{\frac{u}{\norm{u}}},\Tilde{v}}$. Note, $v$ and $\Tilde{v}$ have the same distribution. We can see that $\inner{\Tilde{\frac{u}{\norm{u}}},\Tilde{v}} \sim \mathcal{N}(0,\sigma^2)$ since it is the first coordinate of $\Tilde{v}$. By a standard tail bound, we get that for $t > 0$:

\begin{align*}
    \Pr \left[|\inner{\frac{u}{\norm{u}},v}| \geq t  \right]   
    = \Pr \left[| \inner{\Tilde{\frac{u}{\norm{u}}},\Tilde{v}}| \geq t  \right] 
    = \Pr \left[|\Tilde{v}_1| \geq t  \right]  \leq 2 \exp\left(-\frac{t^2}{2\sigma^2}\right)~.
\end{align*}
Therefore
\[\Pr \left[| \inner{u,v}| \geq \norm{u} t  \right]  \leq 2 \exp\left(-\frac{t^2}{2\sigma^2}\right)~.\]

\end{proof}

\begin{lemma}
\label{lem:2randominnermul_t}
    Let $u \sim \mathcal{N}(\zero,\sigma_1^2 I_n)$, and $v \sim \mathcal{N}(\zero,\sigma_2^2 I_n)$. 
    %  $\mathcal{N}(0,\frac{1}{\sqrt{c}}I_n)$. 
    Then, for every $t>0$ we have
        \[ \Pr \left[|\inner{u,v}| \geq 1.1 \sigma_1 \sqrt{n} t  \right] \leq e^{-\frac{n}{500}} + 2e^{-t^2/2\sigma_2^2}~.\]
\end{lemma}
\begin{proof}
    
Using Lemma~\ref{lem:randomnormsmall} we get that w.p. $\leq e^{-\frac{n}{500}}$ we have $\norm{u} \geq 1.1 \sigma_1 \sqrt{n}$. 
Moreover, by Lemma~\ref{lem:randominnermul_1}, w.p. $\leq 2 \exp \left(-\frac{t^2}{2\sigma_2^2}\right)$ we have $| \inner{u,v}| \geq \norm{u} t$. 
By the union bound, we get
\begin{align*}
    \Pr \left[|\inner{u,v}| \geq 1.1\sigma_1 \sqrt{n} t  \right] 
    \leq \Pr\left[ \norm{u} \geq 1.1 \sigma_1 \sqrt{n} \right] + \Pr \left[ | \inner{u,v}| \geq \norm{u} t \right]
    \leq e^{-\frac{n}{500}} + 2 \exp \left(-\frac{t^2}{2\sigma_2^2} \right)~.
\end{align*}

\end{proof}

\begin{lemma}
\label{lem:randominnermul}
    Let $u,v \sim \mathcal{N}(\zero,\frac{1}{d} I_d)$.
    %  $\mathcal{N}(0,\frac{1}{\sqrt{c}}I_n)$. 
    Then, 
        \[ \Pr \left[|\inner{u,v}| \geq 1.1 \frac{\log(d)}{\sqrt{d}}  \right] \leq e^{-\frac{d}{500}} + 2d^{-\frac{\log(d)}{2}}~.\]
\end{lemma}
\begin{proof}
    Using Lemma~\ref{lem:2randominnermul_t} for $n=d$, $\sigma_1 = \sigma_2 = \frac{1}{\sqrt{d}}$ and $t=\frac{\log(d)}{\sqrt{d}}$.
\end{proof}

\begin{lemma}
\label{lem:setofrandomnorm}
Let a dataset $S=\{(\bx_i,y_i)\}_{i=1}^m$ be such that $\forall i$, $\bx_i \in \reals^d$ and $\bx_i \sim \mathcal{N}(0, \frac{1}{d}I_d)$, for $m\leq d$. Then, w.p. $\geq 1- 2me^{-\frac{d}{500}} $,
For all $(\bx,y) \in S$, $\norm{\bx}^2 \in [0.9, 1.1]$
\end{lemma}

\begin{proof}
We prove both upper and lower bounds.
    \begin{align*}
        &\Pr \left[ \min_{i \in [m]} \left\{ \norm{x_i}^2\right\} < 0.9 \right] = \\
        &= \Pr \left[ \exists i \in [m], \norm{x_i}^2 < 0.9 \right]\\
        &\leq \sum\limits_{i=1}^{m} \Pr \left[\norm{x_i}^2 < 0.9 \right] \leq m e^{-\frac{d}{400}}
    \end{align*}
where the last inequality holds due to \ref{lem:randomnormlarge}. 
% We note that $\lim_{d \rightarrow \infty} d e^{-\frac{d}{400}} =0$.
    \begin{align*}
        &\Pr \left[ \max_{i \in [m]} \left\{ \norm{x_i}^2\right\} > 1.1 \right] = \\
        &= \Pr \left[ \exists i \in [m], \norm{x_i}^2 > 1.1 \right]\\
        &\leq \sum\limits_{i=1}^{m} \Pr \left[\norm{x_i}^2 > 1.1 \right] \leq  me^{-\frac{d}{500}} 
    \end{align*}
where the last inequality holds due to \ref{lem:randomnormsmall}, and the claim follows.
\end{proof}

\begin{lemma}
\label{lem:setofrandominner}
Let a dataset $S=\{(\bx_i,y_i)\}_{i=1}^m$ be such that $\forall i$, $\bx_i \in \reals^d$ and $\bx_i \sim \mathcal{N}(0, \frac{1}{d}I_d)$, for $m\leq d$. Then, w.p. $\geq 1- (m^2 e^{-d/500} + 2m^2d^{-\frac{\log(d)}{2}})$,
For all $(\bx_i,y_i), (\bx_j,y_j)\in S$, $|\inner{\bx_i,\bx_j}| \leq 1.1 \frac{\log(d)}{\sqrt{d}}$

\end{lemma}

\begin{proof}
We prove an upper bound.
    \begin{align*}
        &\Pr \left[\max_{i\neq j} \left\{ |\inner{x_i,x_j}| \right\} > 1.1 \frac{\log(d)}{\sqrt{d}} \right] = \\
        &= \Pr \left[ \exists i.j \in [m], |\inner{x_i,x_j}| > 1.1 \frac{\log(d)}{\sqrt{d}} \right]\\
        &\leq \sum\limits_{i=1}^{m}\sum\limits_{j=1}^{m} \Pr \left[ |\inner{x_i,x_j}| > 1.1 \frac{\log(d)}{\sqrt{d}} \right] \leq m^2 e^{-d/500} + 2m^2d^{-\frac{\log(d)}{2}}
    \end{align*}
where the last inequality holds due to \lemref{lem:randominnermul}.
\end{proof}
\section{Proofs for \secref{sec:linear}}
\label{app:linear}

\begin{lemma}
\label{lem:lin_boundedlambda}
Let $\epsilon_d, \epsilon, \delta \leq 0.5$ and let $N(\bw, \bx)$ be a linear classifier  trained on a dataset $S = \{(\bx_i,y_i)\}_{i=1}^m$, and assume that $\bw$ is an $(\epsilon, \delta)$-approximate KKT point satisfying \defref{def:epsdeltakkt}, and $S$ satisfies \assref{ass:data} for $\psi \leq 0.1, \phi \leq \frac{\epsilon_d}{4m}$. Note, for readability of the proof we denote $\epsilon_1$ by $\epsilon$ and $\delta_1$ by $\delta$. Then,

\[\max_{i} \lambda_i \leq 2.4 \]
% \gilad{I think this part can be removed if we don't show the constants in the theorem statement itself} Particularly, if $\bw$ is a KKT point, i.e. satisfying \defref{def:epsdeltakkt} for $\epsilon=\delta =0$, then
% \[\max_{i} \lambda_i \leq 1.3\]
\end{lemma}

\begin{proof}
We look at $\lambda_r = \max_{i} \lambda_i$. If $\lambda_r = 0$ we are done, since the r.h.s is non-negative. Otherwise, we define $\bv_\epsilon = \bw - \sum_{i=1}^m\lambda_iy_i\bx_i$, and by item (2) from \defref{def:epsdeltakkt} we have that $\norm{\bv_\epsilon}\leq \epsilon$. Hence, we have
\[\bw = \sum\limits_{i=0}^{m}\lambda_i y_i \bx_i + \bv_\epsilon~,\]
and from item (3) of \defref{def:epsdeltakkt} and $\lambda_r>0$, we have $1 + \frac{\delta}{\lambda_r}\geq y_r N(\bw, \bx_r) \geq 1$. Therefore, 

\begin{align*}
1 + \frac{\delta}{\lambda_r}\geq y_r N(\bw, \bx_r) =  y_r \sum\limits_{i=0}^{m}\lambda_i y_i \inner{\bx_i, \bx_r} + y_r \inner{\bx_r, \bv_\epsilon} =& \lambda_r \norm{\bx_r}^2 + y_r \sum\limits_{i \neq r \in [m]}\lambda_i y_i \inner{\bx_i, \bx_r}+ y_r \inner{\bx_r, \bv_\epsilon}\\
\geq& \lambda_r (1-\psi) - \sum\limits_{i \neq r \in [m]} \lambda_i |\inner{\bx_i,\bx_r}| - \norm{\bx_r} \norm{\bv_\epsilon} \\
    \geq& \lambda_r (1-\psi)  - \lambda_r \cdot \phi(m-1) -\epsilon\sqrt{1+\psi}
\end{align*}
where the last two inequalities holds due to \assref{ass:data} and Cauchy-Schwartz inequality. 

Solving for $\lambda_r$ leads to to 
\[
\lambda_r^2 \left((1-\psi) - \phi(m-1) \right) -(1+\epsilon\sqrt{1+\psi}) \lambda_r - \delta \leq 0~.
\]

Since $\psi \leq 0.1$ and $\phi \leq \frac{\epsilon_d}{4m}$ we get 
\begin{align*}
(1-\psi) - \phi(m-1) \geq 0.9 - (m-1) \frac{\epsilon_d}{4m} \geq 0.9 - \frac{\epsilon_d}{4} > 0~,
\end{align*} 
and we get that 
\begin{align}\label{eq:lin_boundedlambda}
    \lambda_r \leq \frac{(1+\epsilon\sqrt{1+\psi}) + \sqrt{(1+\epsilon\sqrt{1+\psi})^2+4((1-\psi) - \phi(m-1))\delta}}{2((1-\psi) - \phi(m-1))}
\end{align}
Plugging in $\epsilon, \delta \leq 0.5$, $\psi \leq 0.1$ and $\phi \leq \frac{\epsilon_d}{4m}$, we get
\begin{align*}
    \lambda_r \leq &\frac{(1+\epsilon\sqrt{1+\psi}) + \sqrt{(1+\epsilon\sqrt{1+\psi})^2+4((1-\psi) - \phi(m-1))\delta}}{2((1-\psi) - \phi(m-1))} \leq \\
    \leq& \frac{(1+0.5\sqrt{1.1}) + \sqrt{(1+0.5\sqrt{1.1})^2+2}}{2(0.9 - \frac{\epsilon_d}{4m}(m-1))}\\
    \leq& \frac{(1+0.5\sqrt{1.1}) + \sqrt{(1+0.5\sqrt{1.1})^2+2}}{2(0.9 - \frac{1}{8})} \leq \frac{3.61}{1.55} \leq 2.4~.
\end{align*}

% Note, if we further have $\epsilon = \delta = 0$, from \eqref{eq:lin_boundedlambda} we get 
% \begin{align}\label{eq:lambdastar_constantbound}
%     \lambda_r \leq \frac{1 + \sqrt{1}}{2((1-\psi) - \phi(m-1))} \leq \frac{1}{0.9- \frac{\epsilon_d}{4m} (m-2)} \leq \frac{1}{0.9-\frac{0.5}{4}} \leq 1.3~.
% \end{align}

\end{proof}

\begin{lemma}
\label{lem:lin_boundedlambda_vanish}
Let $\epsilon_d, \epsilon, \delta \leq 0.5$ and let $N(\bw, \bx)$ be a linear classifier  trained on a dataset $S = \{(\bx_i,y_i)\}_{i=1}^m$, and assume that $\bw$ is an $(\epsilon, \delta)$-approximate KKT point satisfying \defref{def:epsdeltakkt}, and $S$ satisfies \assref{ass:data} for $\psi \leq 0.1, \phi \leq \frac{\epsilon_d}{4m}$. Let $t \in [m]$.Then,
\[\frac{1}{\norm{\bx_t}^2} - \frac{0.6 \epsilon_d + 1.1\epsilon}{\norm{\bx_t}^2} \leq \lambda_t \leq \frac{1}{\norm{\bx_t}^2} + \frac{1.2\epsilon_d + 2.15\epsilon + 2.2\delta }{\norm{\bx_t}^2}~.\]

% \gilad{Also this part can be removed. We can just put $\epsilon=\delta=0$ and get these bounds with slightly worse constants.}
% Particularly, if $\bw$ is a KKT point, i.e. satisfying \defref{def:epsdeltakkt} for $\epsilon=\delta =0$, then

% \[\frac{1}{\norm{\bx_t}^2} - \frac{0.33\epsilon_d}{\norm{\bx_t}^2} \leq\lambda^*_t \leq \frac{1}{\norm{\bx_t}^2} +  \frac{0.33\epsilon_d}{\norm{\bx_t}^2}~.\]

\end{lemma}
\begin{proof}
We begin showing the result for the more general case of $\epsilon, \delta \leq 0.5$.
Let $t \in [m]$. Looking at an upper bound of the margin, we have

\begin{align*}
    1 \leq y_t N(\bw, \bx_t) =  y_t \sum\limits_{i=1}^{m} \lambda_i y_i \inner{\bx_i, \bx_t} + y_t\inner{\bv_\epsilon, \bx_t} \leq& \lambda_t \norm{\bx_t}^2 + \sum_{i\neq t \in [m]}\lambda_i |\inner{\bx_i,\bx_t}| + \inner{\bv_\epsilon, \bx_t} \\
    \leq&\lambda_t \norm{\bx_t}^2 + \phi(m-1)\max_p \lambda_p + \inner{\bv_\epsilon, \bx_t}\\
        \leq& \lambda_t \norm{\bx_t}^2 + 2.4\phi(m-1) +\epsilon \norm{\bx_t}~,
        \end{align*}
where the last inequality hold due to \lemref{lem:lin_boundedlambda} and Cauchy-Schwartz inequality. We solve it for $\lambda_t$ with plugging in $\phi \leq \frac{\epsilon_d}{4m}$ getting a lower bound for it

\begin{align*}
     \lambda_t \geq \frac{1}{\norm{\bx_t}^2} - \frac{ 2.4\phi(m-1)}{\norm{\bx_t}^2} - \frac{\epsilon}{\norm{\bx_t}} \geq \frac{1}{\norm{\bx_t}^2} - \frac{0.6 \epsilon_d + 1.1\epsilon}{\norm{\bx_t}^2}~.
\end{align*} 

We note that $1 - 0.6 \epsilon_d - 1.1 \epsilon \geq 0.15 > 0$, the therefore $\lambda_t > 0$. Next, to find an upper bound for $\lambda_t$, we look at a lower bound of the margin
        \begin{align*}
    1 + \frac{\delta}{\lambda_t} \geq y_t N(\bw, \bx_t) =  y_t \sum\limits_{i=1}^{m} \lambda_i y_i \inner{\bx_i, \bx_t} + y_t\inner{\bv_\epsilon, \bx_t}  \geq& \lambda_t\norm{\bx_t}^2 -\sum_{i\neq t \in [m]}\lambda_i |\inner{\bx_i,\bx_t}| -\inner{\bv_\epsilon, \bx_t}  \\
    \geq& \lambda_t\norm{\bx_t}^2 - \phi(m-1)\max_p \lambda_p -\inner{\bv_\epsilon, \bx_t} \\
        \geq& \lambda_t\norm{\bx_t}^2 - 2.4\phi(m-1)  -\epsilon \norm{\bx_t}~,
\end{align*}
where again the last inequalities holds due to \lemref{lem:lin_boundedlambda} Cauchy-Schwartz inequality. We get \[\lambda_t^2 \norm{\bx_t}^2 - \lambda_t (1 +2.4\phi(m-1)  +\epsilon \norm{\bx_t}) - \delta \leq 0\] and solve for $\lambda_t$ with plugging in $\phi \leq \frac{\epsilon_d}{4m},\norm{\bx_t}^2 \leq (1-\psi), \psi \leq 0.1$ we get an upper bound for $\lambda_t$
\begin{align*} 
\lambda_t &\leq \frac{(1 +2.4\phi(m-1)  +\epsilon \norm{\bx_t}) + \sqrt{(1 +2.4\phi(m-1)  +\epsilon \norm{\bx_t})^2 + 4 \norm{\bx_t}^2 \delta }}{2\norm{\bx_t}^2}\\
&\leq \frac{1 + 2.4 \frac{\epsilon_d}{4m}(m-1)+ \epsilon\sqrt{1+\psi}+1 + 2.4 \frac{\epsilon_d}{4m}(m-1) + \epsilon (1+\psi) + 4\delta(1+\psi) }{2\norm{\bx_t}^2}\\
&\leq \frac{1}{\norm{\bx_t}^2} + \frac{2.4 \frac{\epsilon_d}{4m}(m-1)+ \epsilon\sqrt{1+\psi}+ 2.4 \frac{\epsilon_d}{4m}(m-1) + \epsilon (1+\psi) + 4\delta(1+\psi) }{2\norm{\bx_t}^2}\\
&\leq \frac{1}{\norm{\bx_t}^2} + \frac{2.4 \frac{\epsilon_d}{4}+ \epsilon\sqrt{1.1}+ 2.4 \frac{\epsilon_d}{4} + \epsilon (1.1) + 4\delta(1.1) }{2\norm{\bx_t}^2} \\
&\leq \frac{1}{\norm{\bx_t}^2} + \frac{1.2\epsilon_d + 2.15\epsilon + 2.2\delta }{\norm{\bx_t}^2}~. 
\end{align*}

% If we further have the $\epsilon = \delta = 0$, we get for the upper bound of the margin the slightly different bound. We use \lemref{lem:lin_boundedlambda} and have
% \begin{align*}
%     1 \leq y_t N(\bw, \bx_t) \leq &\lambda_t \norm{\bx_t}^2 + \phi(m-1)\max_p \lambda_p + \inner{\bv_\epsilon, \bx_t}\\
%         \leq& \lambda_t \norm{\bx_t}^2 + 1.3\phi(m-1) ~,
%         \end{align*}
% that yields \[ \lambda_t \geq \frac{1}{\norm{\bx_t}^2} - \frac{0.33\epsilon_d}{\norm{\bx_t}^2}~,\]
% and similarly for the lower we use \lemref{lem:lin_boundedlambda} and have
%         \begin{align*}
%     1  \geq y_t N(\bw, \bx_t) \geq& \lambda_t\norm{\bx_t}^2 - \phi(m-2)\max_p \lambda_p  \\
%         \geq& \lambda^*_t\norm{\bx_t}^2 - 1.3\phi(m-1)~,
% \end{align*}
% resulting in an upper bound 
% \[\lambda_t \leq \frac{1}{\norm{\bx_t}^2} +  \frac{0.33\epsilon_d}{\norm{\bx_t}^2}~.\]
which finishes the proof.
\end{proof}

We next define an $(\epsilon, \delta, \gamma)$-approximate KKT. It is very similar to the $(\epsilon, \delta)$-approximate KKT definition given in \defref{def:epsdeltakkt}, with an extra $\gamma$ relaxation of the margin.
\begin{definition}
\label{def:epsdeltagammakkt}
A $(\epsilon, \delta, \gamma)$-approximate KKT for $\min_\btheta \frac{1}{2} \norm{\btheta}^2 \text{s.t.} \forall i \in [m], y_i N(\btheta, \bx_i) \geq 1$: $\exists \lambda_1,...,\lambda_m$ such that
\begin{enumerate}
    \item $\lambda_1,..., \lambda_m \geq 0$
    \item $\norm{\btheta - \sum\limits_{i=1}^{m} \lambda_i y_i \nabla_{\btheta} N(\btheta, \bx_i)}_2 \leq \epsilon$
    \item $\forall i \in [m]$, $\lambda_i \left( y_i N(\btheta, \bx_i)-1 \right) \leq \delta$
    \item $\forall i \in [m]$, $y_i N(\btheta, \bx_i) \geq 1-\gamma$
\end{enumerate}
\end{definition}

Now, we show that scaling an $(\epsilon, \delta, \gamma)$-approximate KKT can result in an $(\epsilon', \delta')$-approximate KKT, and determine the scaling effect on the approximation parameters.

\begin{lemma}
\label{lem:fixthemargin}
    Let a network $N(\btheta, \bx)$ be such that $N(\btheta, \bx)$ is a 1-homogeneous function with respect to the weights. Let $S = \{(\bx_i,y_i)\}_{i=1}^m$ be a dataset. Then, if $\btheta$ is a $(\epsilon, \delta, \gamma)$-approximate KKT (according to the above \defref{def:epsdeltagammakkt}) w.r.t $S$ with corresponding $\{\lambda_i\}_{i=1}^m$, then $\frac{1}{1-\gamma} \btheta$ is a $(\frac{1}{1-\gamma}\epsilon, \max_p \lambda_p \frac{\gamma}{1-\gamma}+ \frac{1}{1-\gamma} \delta)$-approximate KKT (according to \defref{def:epsdeltakkt}) w.r.t $S$ with with the corresponding $\lambda'_i = C \lambda_i$ .
\end{lemma}

\begin{proof}
    Let $N(\btheta, \bx)$ a 1-homogeneous function with respect to the weights, and $\btheta$ be a $(\epsilon, \delta, \gamma)$-approximate KKT. From $1$-homogeneity, for all $C>0$
    \[N(C\btheta, \bx)  = CN(\btheta, \bx)\]
    and the gradient is $0$-homogeneous, meaning
    % \gilad{I don't think this is true. If $N$ is 1-homogeneous then $\nabla N$ is $0$-homogeneous.}
    \[
    \nabla_{\btheta} N(C\btheta, \bx)  =  \nabla_{\btheta} N(\btheta, \bx)~.
    \]

We denote $C = \frac{1}{1-\gamma} $, and show that $C\btheta$ satisfies the conditions in \defref{def:epsdeltakkt}.
\begin{enumerate}
    \item $\norm{C\btheta - \sum\limits_{i=2}^{m} C\lambda_i y_i \nabla_{\btheta} N(C\btheta, \bx_i)} = C \norm{\btheta - \sum\limits_{i=2}^{m} \lambda_i y_i  \nabla_{\btheta} N(\btheta, \bx_i)} \leq C \epsilon$.
    \item Let $i \in [m]$. Then, $y_i N(C\btheta, \bx_i) = C y_i N(\btheta, \bx_i) \geq C(1-\gamma) = 1$  
    \item Let $i \in [m]$. Assume $\lambda_i \left( y_i N(\btheta, \bx_i)-1 \right) \leq \delta$. If $\lambda_i = 0$ we are done. Else, $\lambda_i > 0$ and $y_i N(\btheta, \bx_i) \leq 1+ \frac{\delta}{\lambda_i} $. 
    Then, 
    \begin{align*}
        &\lambda_i \left( y_i N(C\btheta, \bx_i)-1 \right) = \lambda_i \left(C y_i N(\btheta, \bx_i)-1 \right) \leq \\
        &\leq \lambda_i \left(C (1+ \frac{\delta}{\lambda_i})-1 \right)  = \lambda_i (C-1) + C\delta \leq \max_p \lambda_p \frac{\gamma}{1-\gamma}+ \frac{1}{1-\gamma} \delta~,
    \end{align*}
    which finishes the proof.
    % \om{to edit linear according to fix}
\end{enumerate}
\end{proof}

\subsection{Proof for \thmref{thm:linunlearning}}
\label{app:linunlearning}
\begin{proof}
   Note, for readability of the proof we denote $\epsilon_1$ by $\epsilon$ and $\delta_1$ by $\delta$.

% Let $0<\epsilon_d, \epsilon, \delta \leq 0.5$.  Let $N(\bw,\bx)  = \inner{\bw, \bx}$ a linear predictor trained on dataset $S = \{(\bx_1,y_1),...,(\bx_m,y_m)\}$, where $S$ satisfies \assref{ass:data} for $\psi \leq 0.1$ and $\phi \leq \frac{\epsilon_d}{4m}$. Let the weight vector $\bw$ be an $(\epsilon, \delta)$-approximate KKT point for the margin maximization problem (\eqref{eq:margmax}) w.r.t. $S$ according to \defref{def:epsdeltakkt} for $\lambda_1,...,\lambda_m$. Given $l \in [m]$, we denote by $\hat{\bw}$ an unlearned linear predictor created by performing gradient ascent unlearning of the data sample $(\bx_l,y_l) \in S$ with step size $\lambda_l$ (\eqref{eq:gradascent}) \gilad{When we write the main part of the paper I suggest writing explicitly the unlearning algorithm and reference to it here.}.
Using the stationarity condition in \defref{def:epsdeltakkt} for $\bw$, we denote $\bv_{\epsilon} = \bw-\sum\limits_{i=1}^{m} \lambda_i y_i \nabla_{\bw} N(\bw, \bx_i)$, so we get that $\norm{\bv_{\epsilon}} \leq \epsilon$ and
\[
\bw = \sum\limits_{i=1}^{m} \lambda_i y_i \nabla_{\bw} N(\bw, \bx_i) + \bv_{\epsilon} = \sum\limits_{i=1}^{m} \lambda_i y_i \bx_i + \bv_{\epsilon}~.
\]

Let $l \in [m]$, we wish to take a negative gradient step of size $\beta$, such that 
\[
\beta \nabla_{\bw} \ell(y_l N(\bw, \bx_l)) = - \lambda_l y_l \nabla_{\bw} N(\bw, \bx_l)
\]
so we pick a step size $\beta = \frac{- \lambda_l}{\ell'(y_l N(\bw, \bx_l))}$.
Then, when taking one gradient ascent step for $(\bx_l, y_l)$ of size $\beta$, we get the following $\hat{\bw}$
\[
\hat{\bw} = \sum\limits_{i=1}^{m} \lambda_i y_i \nabla_{\bw} N(\bw, \bx_i) +\bv_{\epsilon} - \lambda_l y_l \nabla_{\bw} N(\bw, \bx_r) = \sum\limits_{i \in [m]_{-l}} \lambda_i y_i \bx_i +\bv_{\epsilon}~.
\]

\subsubsection{Proof of 1. $\hat{\bw}$ has the direction of an $(\epsilon + \frac{\epsilon\epsilon_d}{m-\epsilon_d} , \delta +\frac{\delta\epsilon_d}{m-\epsilon_d} +\frac{7.2 \epsilon_d}{m})$-approximate KKT point for the margin maximization problem for $S \setminus (\bx_l,y_l)$.}

For readability, we show that $\hat{\bw}$ satisfies the conditions for $(\epsilon, \delta + \frac{1.44\epsilon_d}{m}, \frac{0.6\epsilon_d}{m})$-approximate KKT by \defref{def:epsdeltagammakkt}, and then use \lemref{lem:fixthemargin} to deduce that $\frac{1}{1-\frac{0.6\epsilon_d}{m}}\hat{\bw}$  satisfies the conditions for  $(\epsilon + \frac{\epsilon\epsilon_d}{m-\epsilon_d} , \delta +\frac{\delta\epsilon_d}{m-\epsilon_d} +\frac{7.2 \epsilon_d}{m})$-approximate KKT according to \defref{def:epsdeltakkt}.\\

% The weights $\bw$ satisfies the $(\epsilon, \delta)$-approximate KKT conditions in \defref{def:epsdeltakkt} for $\lambda_1,...,\lambda_m$. We denote $\bv_{\epsilon} = \bw-\sum\limits_{i=1}^{m} \lambda_i y_i \nabla_{\bw} N(\bw, \bx_i)$, so we get that $\norm{\bv_{\epsilon}} \leq \epsilon$ and
% \[
% \bw = \sum\limits_{i=1}^{m} \lambda_i y_i \nabla_{\bw} N(\bw, \bx_i) + \bv_{\epsilon} = \sum\limits_{i=1}^{m} \lambda_i y_i \bx_i + \bv_{\epsilon}~.
% \]
% We will now show that $\hat{\bw}$ satisfies all the conditions of \defref{def:epsdeltagammakkt}.
% \gilad{Write a sentence here describing what you are going to prove. Something like: "We will now show that $\hat{\bw}$ satisfies all the conditions of Definition 18." Also, change the name from step to condition and don't write it here, since it may confuse  things we want to show and things we prove.}
% \gilad{Sorry for changing my mind here, but I now think it would be clearer to state which condition we refer to (i.e. primal feasibility, stationarity etc.) instead of referring to them as condition 1,2,.. }
\paragraph{(1) Dual Feasibility: For all $i \in [m]_{-l}$, $\lambda_i \geq 0$.} directly from dual feasibility for $\bw$ (\defref{def:epsdeltakkt}). 

\paragraph{(2) Stationarity: $\norm{\hat{\bw} - \sum\limits_{i=1}^{m} \lambda_i y_i \nabla_{\bw} N(\hat{\bw}, \bx_i)} \leq \epsilon$.} 
Since $\nabla_{\bw}N(\hat{\bw}, \bx) = \nabla_{\bw}N(\bw, \bx) = x$, one can write 
\[\hat{\bw} = \sum\limits_{i \in [m]_{-l}} \lambda_i y_i \bx_i + \bv_{\epsilon} =\sum\limits_{i \in [m]_{-l}} \lambda_i y_i \nabla_{\bw} N(\hat{\bw}, \bx_i) + \bv_{\epsilon}\]
and the claim follows from (2) stationarity for $\bw$ (\defref{def:epsdeltakkt}).

Let $t \in [m]_{-l}$. Using the definitions of $\bw$ and $\hat{\bw}$, we can write the margin as

\begin{align}\label{eq:linearmargin}
    y_t N(\bw, \bx_t) = y_t \sum\limits_{i=1}^{m} \lambda_i y_i \inner{\bx_i, \bx_t} + y_t\inner{\bv_\epsilon, \bx_t}= y_t N(\hat{\bw}, \bx_t) + y_t  \lambda_l  y_l \inner{\bx_l,\bx_t}~.
\end{align}

Using this equality we prove the next two conditions:

\paragraph{(3) Complementarity Slackness: For all $t \in [m]_{-l}$, $\lambda_t \left( y_t N(\hat{\bw}, \bx_t)-1 \right) \leq \delta + \frac{1.44\epsilon_d}{m}$.}

If $\lambda_t =0$ we are done. Else, $\lambda_t > 0$. From complementarity slackness of $\bw$ being an $(\epsilon, \delta)$-approximate KKT, we know that $y_t N(\bw, \bx_t) \leq 1+ \frac{\delta}{\lambda_t}$. We use \eqref{eq:linearmargin} to lower bound the margin of $y_t N(\bw, \bx_t)$, getting

\begin{align*}
    1+ \frac{\delta}{\lambda_t} \geq y_t N(\bw, \bx_t) =& y_t N(\hat{\bw}, \bx_t) +  y_t  \lambda_l  y_l |\inner{\bx_l,\bx_t}|\\
    \geq&  y_t N(\hat{\bw}, \bx_t)  - \lambda_l |\inner{\bx_l, \bx_t}| \\
    \geq&  y_t N(\hat{\bw}, \bx_t)  - \phi \max_p \lambda_p~,
\end{align*}

plugging in $\phi \leq \frac{\epsilon_d}{4m}$ and the $\lambda_p$ upper bound from \lemref{lem:lin_boundedlambda} we get
\[  y_t N(\hat{\bw}, \bx_t) -\phi \max_p \lambda_p \geq y_t N(\hat{\bw}, \bx_t) -\frac{\epsilon_d}{4m} 2.4 \geq y_t N(\hat{\bw}, \bx_t) -\frac{0.6\epsilon_d}{m} ~.\]

We deduce an upper bound for the margin of $ N(\hat{\bw}, \bx_t)$-
\[y_t N(\hat{\bw}, \bx_t) \leq 1+ \frac{\delta}{\lambda_t} + \frac{0.6\epsilon_d}{m} = 1+ \frac{\delta + \frac{3}{5m} \lambda_t \epsilon_d}{\lambda_t}  \leq 1+\frac{\delta + \frac{3}{5m} 2.4  \epsilon_d}{\lambda_t} \leq 1+\frac{\delta + \frac{1.44\epsilon_d}{m} }{\lambda_t}\] as desired.

\paragraph{(4) Primal Feasibility: For all $t \in [m]_{-l}$, $y_t N(\hat{\bw}, \bx_t) \geq 1-\frac{0.6\epsilon_d}{m}$.}

We use \eqref{eq:linearmargin} to lower bound the margin of $N(\hat{\bw}, \bx_t)$, and use primal feasibility for $\bw$ (\defref{def:epsdeltakkt}), getting
\begin{align*}
    y_t N(\hat{\bw}, \bx_t) = y_t N(\bw, \bx_t) -  y_t  \lambda_l  y_l |\inner{\bx_l,\bx_t}| \geq 1- \lambda_l |\inner{\bx_l,\bx_t}| \geq 1 - \phi \max_p \lambda_p ~.
\end{align*}

Plugging in $\phi \leq \frac{\epsilon_d}{4m}$ and the $\lambda_p$ upper bound from \lemref{lem:lin_boundedlambda} we get that
\begin{align*}
     \phi \max_p \lambda_p \leq \frac{2.4\epsilon_d}{4m} \leq \frac{0.6\epsilon_d}{m} ~.
\end{align*}
Hence,  $y_t N(\hat{\bw}, \bx_t) \geq 1- \frac{0.6\epsilon_d}{m}$.

To conclude, we showed that $\hat{\bw}$ is an $(\epsilon, \delta + \frac{1.44\epsilon_d}{m}, \frac{0.6\epsilon_d}{m})$-approximate KKT by \defref{def:epsdeltagammakkt} . Finally, we look at the scaled weights $\frac{1}{1-\frac{0.6\epsilon_d}{m}}\hat{\bw}$. For $\epsilon_d \leq 1$ We calculate

% \[\frac{1}{1-\frac{0.6\epsilon_d}{m}}\epsilon \leq \frac{1}{0.4}\epsilon \leq 2.5 \epsilon~, \]
\[\frac{1}{1-\frac{0.6\epsilon_d}{m}}\epsilon \leq \frac{m}{m-\epsilon_d}\epsilon = \left(1 + \frac{\epsilon_d}{m-\epsilon_d} \right) \epsilon = \epsilon + \frac{\epsilon\epsilon_d}{m-\epsilon_d}~, \]

and
% \[\max_p \lambda_p  \frac{\frac{0.6\epsilon_d}{m}}{1-\frac{0.6\epsilon_d}{m}} + \frac{\delta + \frac{1.44\epsilon_d}{m}}{1-\frac{0.6\epsilon_d}{m}}  \leq 2.4  \frac{\frac{0.6\epsilon_d}{m}}{0.4} + \frac{\delta + \frac{1.44\epsilon_d}{m}}{0.4} \leq 2.5 \delta + \frac{7.2 \epsilon_d}{m} \]

\[\max_p \lambda_p  \frac{\frac{0.6\epsilon_d}{m}}{1-\frac{0.6\epsilon_d}{m}} + \frac{\delta + \frac{1.44\epsilon_d}{m}}{1-\frac{0.6\epsilon_d}{m}}  \leq  \delta +\frac{\delta\epsilon_d}{m-\epsilon_d} +\frac{7.2 \epsilon_d}{m} \]

and get from \lemref{lem:fixthemargin} that $\frac{1}{1-\frac{0.6\epsilon_d}{m}}\hat{\bw}$ is a $(\epsilon + \frac{\epsilon\epsilon_d}{m-\epsilon_d} , \delta +\frac{\delta\epsilon_d}{m-\epsilon_d} +\frac{7.2 \epsilon_d}{m})$-approximate KKT by \defref{def:epsdeltakkt} w.r.t. $S\setminus {(\bx_l, y_l)}$. We note that $\hat{\bw}$ and $\frac{1}{1-\frac{0.6\epsilon_d}{m}}\hat{\bw}$ have the same direction, which finishes the proof.

\subsubsection{Proof of 2. $ Cosine-Similarity(\hat{\bw}, \bw^*) \geq 1 - C(\sqrt{\epsilon_d} + \sqrt{\epsilon_d} + \sqrt{\delta})$ for some $C>0$.}

Let $N(\bw^*, \bx)$ be a max-margin linear predictor w.r.t. the remaining training set $S \setminus (\bx_l,y_l)$. Hence, $\bw^*$ is a KKT point of the margin maximization problem (\eqref{eq:margmax}) w.r.t. $\{\bx_i,y_i\}_{i \in [m]_{-l}}$, as in \defref{def:epsdeltakkt} (with $\epsilon=\delta=0$). From the stationarity condition we denote $\bw^* = \sum_{i \in [m]_{-l}}\lambda_i^* y_i\bx_i$.
% \gilad{This should be the assumption in the theorem, no?} This assumption is reasonable since the training converge to a vector in the direction of $\bw^*$, and the following proof holds regardless of the vector norm.\\

Let $t \in [m]_{-l}$. We use \lemref{lem:lin_boundedlambda_vanish} to prove tight bounds for $\lambda_t$ and $\lambda_t^*$. For a given $t$, $\lambda_t$ and $\lambda_t^*$ are close up to a small additive factor depend on $\epsilon_d, \epsilon $ and $ \delta$. For $\lambda_t$ we can use the results from \lemref{lem:lin_boundedlambda_vanish} directly, having 
\begin{align}\label{eq:lambda_vanishbound}
\frac{1}{\norm{\bx_t}^2} - \frac{0.6 \epsilon_d + 1.1\epsilon}{\norm{\bx_t}^2} \leq \lambda_t \leq \frac{1}{\norm{\bx_t}^2} + \frac{1.2\epsilon_d + 2.15\epsilon + 2.2\delta }{\norm{\bx_t}^2}~.
        \end{align} 

For $\lambda^*_t$, since $\bw^*$ is a KKT point of \eqref{eq:margmax} w.r.t. $S \setminus (\bx_l,y_l)$, we have a dataset of size $m-1$ and $\epsilon=\delta=0$. To accommodate the different parameter, we note that $\phi \leq \frac{\epsilon_d}{4m} \leq \frac{\epsilon_d}{4(m-1)}$, conclude that 

\begin{align}\label{eq:lambdastar_vanishbound}
\frac{1}{\norm{\bx_t}^2} - \frac{0.6 \epsilon_d}{\norm{\bx_t}^2} \leq \lambda^*_t \leq \frac{1}{\norm{\bx_t}^2} + \frac{1.2\epsilon_d}{\norm{\bx_t}^2}~.
        \end{align} 

And, similar note hold for \ref{lem:lin_boundedlambda} resulting in $\lambda^* \leq 2.4$.
We are now ready to prove the cosine similarity lower bound. For $\hat{\bw} = \sum_{i \in [m]_{-l}}\lambda_iy_i\bx_i+ \bv_{\epsilon}$ and $\bw^* = \sum_{i \in [m]_{-l}}\lambda_i^* y_i\bx_i$, we have 
\begin{align*}
    \frac{\inner{\hat{\bw}, \bw^*}}{\norm{\hat{\bw}}\norm{\bw^*}} = \frac{\inner{\sum_{i \in [m]_{-l}} \lambda_i y_i \bx_i +\bv_{\epsilon}, \sum_{i \in [m]_{-l}} \lambda_i^*y_i \bx_i } }{{\norm{\hat{\bw}}\norm{\bw^*}}}~.
    % \leq\approx \sum_{i=2}^m (\lambda_i - \lambda_i')^2~,
\end{align*}

We upper bound the norm of the predictors, when using \eqref{eq:lambda_vanishbound} and \eqref{eq:lambdastar_vanishbound} for any $i \in [m]_{-l}$ separately, bounding $\lambda_i \norm{\bx_i}^2$ and $\lambda_i^* \norm{\bx_i }^2$ respectively. Upper bounding $\norm{\hat{\bw}}^2$ we get

\begin{align*}
    &\norm{\hat{\bw}}^2 = \norm{\sum_{i \in [m]_{-l}}\lambda_iy_i\bx_i+ \bv_{\epsilon}}^2 = \inner{\sum_{i \in [m]_{-l}}\lambda_iy_i\bx_i+ \bv_{\epsilon}, \sum_{i \in [m]_{-l}}\lambda_iy_i\bx_i+ \bv_{\epsilon}} = \\
    &= \inner{\sum_{i \in [m]_{-l}}\lambda_iy_i\bx_i,\sum_{i \in [m]_{-l}}\lambda_iy_i\bx_i} + 2\inner{\sum_{i \in [m]_{-l}}\lambda_iy_i\bx_i, \bv_{\epsilon}} + \inner{\bv_\epsilon, \bv_\epsilon} \\
    &\leq \sum_{i \in [m]_{-l}}\lambda_i^2 \norm{x_i}^2 + \sum_{i \neq k \in [m]_{-l}}\lambda_i \lambda_k \inner{\bx_i,\bx_k} + 2 \sum_{i \in [m]_{-l}}\lambda_i\inner{\bx_i,\bv_\epsilon} + \epsilon^2
    \end{align*}
    
From \eqref{eq:lambda_vanishbound} we get that $\lambda_i \norm{x_i}^2 \leq  \left( 1 + 1.2\epsilon_d + 2.15\epsilon + 2.2\delta\right)$, from \lemref{lem:lin_boundedlambda} we get that for all $i$, $\lambda_i \leq 2.4$ and by \assref{ass:data} we get that for all $i,k \in [m]$ $\inner{\bx_i,\bx_k} \leq \phi$. Using Cauchy–Schwarz inequality we get that for all $i \in [m]$, $\inner{\bx_i,\bv_\epsilon} \leq \norm{\bx_i}\norm{\bv_\epsilon} \leq \epsilon\sqrt{1+\psi}$. Plug it all in we have

\begin{align*}
   \norm{\hat{\bw}}^2  &\leq \sum_{i \in [m]_{-l}}\lambda_i^2 \norm{x_i}^2 + \sum_{i \neq k \in [m]_{-l}}\lambda_i \lambda_k \inner{\bx_i,\bx_k} + 2 \sum_{i \in [m]_{-l}}\lambda_i\inner{\bx_i,\bv_\epsilon} + \epsilon^2\\
   &\leq \left( 1 + 1.2\epsilon_d + 2.15\epsilon + 2.2\delta\right) \sum_{i \in [m]_{-l}}\lambda_i + 2.4 m \phi \sum_{i \in [m]_{-l}}\lambda_i + \epsilon\sqrt{1+\psi} \sum_{i \in [m]_{-l}}\lambda_i + \epsilon^2\\
   &\leq \sum_{i \in [m]_{-l}}\lambda_i \left(\left( 1 + 1.2\epsilon_d + 2.15\epsilon + 2.2\delta\right) + 2.4 m \phi + \epsilon\sqrt{1+\psi} \right) + \epsilon^2
\end{align*}

We denote $\Lambda = \sum_{i \in [m]_{-l}}\lambda_i$ and plug in $\phi \leq \frac{\epsilon_d}{4m}$ and $\psi \leq 0.1$ and get
\begin{align*}
    \norm{\hat{\bw}}^2  &\leq \sum_{i \in [m]_{-l}}\lambda_i \left(\left( 1 + 1.2\epsilon_d + 2.15\epsilon + 2.2\delta\right) + 2.4 m \phi + \epsilon\sqrt{1+\psi} \right) + \epsilon^2 \\
    &\leq \Lambda \left(\left( 1 + 1.2\epsilon_d + 2.15\epsilon + 2.2\delta\right) + 0.6 \epsilon_d + 1.1\epsilon \right) + \epsilon^2 \\
    &\leq \Lambda \left( 1 + 1.8\epsilon_d + 3.25\epsilon + 2.2\delta\right) + \epsilon^2 
\end{align*}

For the upper bound of $\norm{\bw^*}^2$ we do similar calculations, using \eqref{eq:lambdastar_vanishbound} and \lemref{lem:lin_boundedlambda} getting

\begin{align*}
    \norm{\bw^*}^2 =& \norm{\sum_{i \in [m]_{-l}} \lambda_i^*y_i \bx_i }^2 = \inner{\sum_{i \in [m]_{-l}} \lambda_i^*y_i \bx_i , \sum_{i \in [m]_{-l}} \lambda_i^*y_i \bx_i }\\
    &\leq \sum_{i \in [m]_{-l}} (\lambda^*_i)^2 \norm{\bx_i}^2 + \sum_{i \neq k \in [m]_{-l}} \lambda_i^*  \lambda_k^*\inner{\bx_i,\bx_k}\\
    &\leq \left( 1 + 1.2\epsilon_d + 2.15\epsilon + 2.2\delta\right) \sum_{i \in [m]_{-l}} \lambda^*_i + 2.4 m\phi \sum_{i \in [m]_{-l}} \lambda^*_i
\end{align*}

W.L.O.G, we assume that $\sum_{i \in [m]_{-l}}\lambda_i \geq \sum_{i \in [m]_{-l}} \lambda_i^*$ (the other direction is proven similarly). This allow as to upper bound $\norm{\bw^*}^2$ using $\lambda_i$, with plugging in $\phi \leq \frac{\epsilon_d}{4m}$, we get
\begin{align*}
    \norm{\bw^*}^2 &\leq \left( 1 + 1.2\epsilon_d \right) \sum_{i \in [m]_{-l}} \lambda^*_i + 2.4m \phi \sum_{i \in [m]_{-l}} \lambda^*_i\\
    &\leq  \left( 1 + 1.2\epsilon_d \right) \sum_{i \in [m]_{-l}} \lambda_i + 2.4m \phi \sum_{i \in [m]_{-l}} \lambda_i\\
    &\leq \Lambda \left( 1 + 1.8\epsilon_d \right) 
\end{align*}

For the norm multiplication we have

\begin{align*}
    \norm{\hat{\bw}}\norm{\bw^*} &= \sqrt{\norm{\hat{\bw}}^2\norm{\bw^*}^2} = \sqrt{\left[\Lambda \left( 1 + 1.8\epsilon_d + 3.25\epsilon + 2.2\delta\right) + \epsilon^2\right] \left[\Lambda \left( 1 + 1.8\epsilon_d \right)\right]}\\
    &\leq \Lambda \sqrt{(1 + C ( \epsilon_d + \epsilon + \delta)) + \frac{\epsilon^2}{\Lambda} \left( 1 + C\epsilon_d \right)}\\
    &\leq \Lambda \sqrt{1 + C ( \epsilon_d + \epsilon + \delta) + \frac{\epsilon^2}{\Lambda} + \frac{\epsilon^2}{\Lambda}C\epsilon_d}\\
    &\leq \Lambda+ \Lambda\sqrt{C ( \epsilon_d + \epsilon + \delta) + \frac{\epsilon^2}{\Lambda} + \frac{\epsilon^2}{\Lambda}C\epsilon_d}
\end{align*}

for some constant $C>0$, where the last inequality hold since $1+\sqrt{x} \geq \sqrt{1+x}$ for all $x>0$. We next lower bound the inner product of $\hat{\bw}$ and $\bw^*$

\begin{align*}
    \inner{\hat{\bw}, \bw^*} &= \inner{\sum_{i \in [m]_{-l}}\lambda_iy_i\bx_i+ \bv_{\epsilon}, \sum_{i \in [m]_{-l}} \lambda_i^*y_i \bx_i } =\\
    &=\inner{\sum_{i \in [m]_{-l}}\lambda_iy_i\bx_i, \sum_{i \in [m]_{-l}} \lambda_i^*y_i \bx_i } + \inner{\sum_{i \in [m]_{-l}} \lambda_i^*y_i \bx_i, \bv_{\epsilon}}\\
    &\geq \sum_{i \in [m]_{-l}}\lambda_i^*\lambda_i \norm{\bx_i}^2 -  \sum_{i \neq k \in [m]_{-l}} \lambda_i^*\lambda_k \inner{\bx_i,\bx_k} - \sum_{i \in [m]_{-l}} \lambda_i^* \inner{\bx_i ,\bv_\epsilon}\\
\end{align*}

Here, we use the lower bound for $\lambda_i^* \norm{\bx_i}^2 \geq \left( 1-0.6\epsilon_d\right)$, the upper bound $\lambda_i^* \leq 2.4$ from \lemref{lem:lin_boundedlambda}, and the Cauchy–Schwarz inequality, having

\begin{align*}
    \inner{\hat{\bw}, \bw^*} &\geq \sum_{i \in [m]_{-l}}\lambda_i^*\lambda_i \norm{\bx_i}^2 -  \sum_{i \neq k \in [m]_{-l}} \lambda_i^*\lambda_k \inner{\bx_i,\bx_k} - \sum_{i \in [m]_{-l}} \lambda_i^* \inner{\bx_i ,\bv_\epsilon}\\
    &\geq \left( 1-0.6\epsilon_d\right)\sum_{i \in [m]_{-l}}\lambda_i - 2.4 m\phi \sum_{i \in [m]_{-l}}\lambda_i - \epsilon\sqrt{1+\psi} \sum_{i \in [m]_{-l}}\lambda_i
\end{align*}

and by plugging in $\phi \leq \frac{\epsilon_d}{4m}$, $\psi \leq 0.1$ we have

\begin{align*}
    \inner{\hat{\bw}, \bw^*} &\geq \left( 1-0.6\epsilon_d\right)\sum_{i \in [m]_{-l}}\lambda_i - 2.4 m \phi \sum_{i \in [m]_{-l}}\lambda_i - \epsilon\sqrt{1+\psi} \sum_{i \in [m]_{-l}}\lambda_i\\
    &\geq \Lambda \left( 1 -0.6 \epsilon_d - 0.6 \epsilon_d - 1.1 \epsilon\right)\\
    &\geq \Lambda - \Lambda\left( 1.2 \epsilon_d + 1.1 \epsilon\right)
\end{align*}

Join all the bounds toghter, we get for the cosine similarity 

\begin{align*}
    \frac{\inner{\hat{\bw}, \bw^*}}{\norm{\hat{\bw}}\norm{\bw^*}} &\geq \frac{\Lambda - \Lambda\left( 1.2 \epsilon_d + 1.1 \epsilon\right)}{\Lambda+ \Lambda\sqrt{C ( \epsilon_d + \epsilon + \delta) + \frac{\epsilon^2}{\Lambda} + \frac{\epsilon^2}{\Lambda}C\epsilon_d}}\\
    &\geq 1 - \frac{\Lambda\left( 1.2 \epsilon_d + 1.1 \epsilon\right) + \Lambda\sqrt{C ( \epsilon_d + \epsilon + \delta) + \frac{\epsilon^2}{\Lambda} + \frac{\epsilon^2}{\Lambda}C\epsilon_d}}{\Lambda+ \Lambda\sqrt{C ( \epsilon_d + \epsilon + \delta) + \frac{\epsilon^2}{\Lambda} + \frac{\epsilon^2}{\Lambda}C\epsilon_d}}\\
    &\geq 1 - \frac{\left( 1.2 \epsilon_d + 1.1 \epsilon\right) + \sqrt{C ( \epsilon_d + \epsilon + \delta) + \frac{\epsilon^2}{\Lambda} + \frac{\epsilon^2}{\Lambda}C\epsilon_d}}{1+ \sqrt{C ( \epsilon_d + \epsilon + \delta) + \frac{\epsilon^2}{\Lambda} + \frac{\epsilon^2}{\Lambda}C\epsilon_d}}\\
    &\geq 1 -\left( 1.2 \epsilon_d + 1.1 \epsilon\right) - \sqrt{C ( \epsilon_d + \epsilon + \delta) + \frac{\epsilon^2}{\Lambda} + \frac{\epsilon^2}{\Lambda}C\epsilon_d}
\end{align*}

We note that by \lemref{lem:lin_boundedlambda_vanish}
\begin{align*}
\Lambda = \sum_{i \in [m]_{-l}} \lambda_i \geq& (m-1)\left(\frac{1}{\norm{\bx_t}^2} - \frac{0.6 \epsilon_d + 1.1\epsilon}{\norm{\bx_t}^2}\right) \\
\geq&(m-1)0.9\left(1 -0.6 \epsilon_d - 1.1\epsilon\right)\\
\geq& 0.1(m-1)~,
\end{align*}

Concluding,

\begin{align*}
    \frac{\inner{\hat{\bw}, \bw^*}}{\norm{\hat{\bw}}\norm{\bw^*}} &\geq 1 -\left( 1.2 \epsilon_d + 1.1 \epsilon\right) - \sqrt{C ( \epsilon_d + \epsilon + \delta) + \frac{\epsilon^2}{0.1(m-1)} + \frac{\epsilon^2}{0.1(m-1)}C\epsilon_d}\\
    &\geq 1 -C_2\left( \sqrt{\epsilon_d} + \sqrt{\epsilon} + \sqrt{\delta}\right)
\end{align*}

for some constant $C_2 > 0$.
\end{proof}

\subsection{Proof for forgetting subset of points  using $\mathcal{A}_{\text{k-GA}}$ -- linear predictors 
%(\secref{sec:kset})
}\label{app:linunlearning_subset}
We formalize and prove the statement for unlearning a subset of data points. 
Here, the term \emph{successful unlearning} is the natural extension of Definition~\ref{def:suc_unlearning} to unlearning a subset, rather than a single point.
\begin{theorem}
\label{thm:linsubset}
    In the same settings as Theorem~\ref{thm:linunlearning}, let $S_{\text{forget}} \subseteq S$ be a subset of size $k$.

    Then, 
% unlearning $(\bx_l,y_l)$ using 
the extended algorithm $\mathcal{A}_{\text{K-GA}}$, with appropriate coefficients $\{\beta_r\}$, is an \textbf{$(\epsilon, \delta, \tau)$-successful} unlearning algorithm  w.r.t. $\bw$ and $S$, 
% for an example $(\bx_l,y_l)$ 
where:
\begin{enumerate}
    \item \textbf{The case of $\epsilon=\epsilon_1 +\frac{\epsilon_1\epsilon_d}{\frac{m}{k}-\epsilon_d}$, $\delta=\delta_1 + \frac{\delta_1\epsilon_d}{\frac{m}{k}-\epsilon_d}+ \frac{7.2 \epsilon_d}{m}$, $\tau =0$:}\\
    The predictor $\mathcal{A}_{\text{k-GA}}(\bw, S, l)$ has the direction of an $(\epsilon, \delta)$-approximate KKT point for the margin maximization problem (\eqref{eq:margmax}) w.r.t. $S \setminus (\bx_l,y_l)$.
    % \item The predictor $\hat{\bw}$ has the direction of an $(\epsilon +\frac{\epsilon\epsilon_d}{m-\epsilon_d}, \delta + \frac{\delta\epsilon_d}{m-\epsilon_d}+ \frac{7.2 \epsilon_d}{m})$-approximate KKT point for the margin maximization problem (Problem~\eqref{eq:margmax}) w.r.t. $S \setminus (\bx_l,y_l)$.
    \item \textbf{The case of $\epsilon=\delta=0$, $\tau =C(\sqrt{\epsilon_d} + \sqrt{\epsilon_1} + \sqrt{\delta_1})$ for some universal constant $C>0$:} \\
    Let $\bw^*$ be a max-margin linear predictor w.r.t. the remaining training set $S \setminus (\bx_l,y_l)$, i.e. the global optimum of the \eqref{eq:margmax} w.r.t. $S \setminus (\bx_l,y_l)$. Then, $ \cossim(\mathcal{A}_{\text{k-GA}}(\bw, S, l), \bw^*) \geq 1 -\tau$.
    \end{enumerate}
    % we unlearn $S_{\text{forget}} \subset S$ a subset of size $k$, using a step consists of the $k$ gradient directions with an appropriate sizes, we obtain a predictor $\hat{\bw}$ such that:
    % \begin{enumerate}
    % \item The predictor $\hat{\bw}$ has the direction of an $(\epsilon +\frac{\epsilon\epsilon_d}{\frac{m}{k}-\epsilon_d}, \delta + \frac{\delta\epsilon_d}{\frac{m}{k}-\epsilon_d}+ \frac{7.2k\epsilon_d}{m})$-approximate KKT point for the margin maximization problem (Problem~\eqref{eq:margmax}) w.r.t. $S \setminus (\bx_l,y_l)$.
    % \item Let $\bw^*$ be a max-margin linear predictor w.r.t. the remaining training set $S \setminus (\bx_l,y_l)$, i.e. the global optimum of the Problem~\eqref{eq:margmax} w.r.t. $S \setminus (\bx_l,y_l)$.
    % Then, $ \cossim(\hat{\bw}, \bw^*) \geq 1 -C(\sqrt{\epsilon_d} + \sqrt{\epsilon} + \sqrt{\delta})$ for some constant $C>0$.
    % \end{enumerate}
\end{theorem}

\begin{proof}
    Let a forget set $S_f \subset S$ such that $|S_f| = k$. We denote $I_f = \{i : (\bx_i,y_i) \in S_f\}$. We denote $S_r = S \setminus S_f$ and $I_r = \{i : (\bx_i,y_i) \in S_r\}$. The proof is highly similar to the proof for unlearning single point in \ref{app:linunlearning}.

    Similarly, we denote $\bv_{\epsilon} = \bw-\sum\limits_{i=1}^{m} \lambda_i y_i \nabla_{\bw} N(\bw, \bx_i)$, so we get that $\norm{\bv_{\epsilon}} \leq \epsilon$ and
\[
\bw = \sum\limits_{i=1}^{m} \lambda_i y_i \nabla_{\bw} N(\bw, \bx_i) + \bv_{\epsilon} = \sum\limits_{i=1}^{m} \lambda_i y_i \bx_i + \bv_{\epsilon}~.
\]

According to the algorithm $\mathcal{A}_{\text{k-GA}}$, we take a step consists of the sum of $k$ gradients w.r.t. data points in $S_f$ with the following sizes- For any $(\bx_l,y_l) \in S_f$, we sum a gradient of size $\beta = \frac{- \lambda_l}{\ell'(y_l N(\bw, \bx_l))}$. We get 

\[
\hat{\bw} = \sum\limits_{i=1}^{m} \lambda_i y_i \nabla_{\bw} N(\bw, \bx_i) +\bv_{\epsilon} - \sum\limits_{l \in I_f} \lambda_l y_l \nabla_{\bw} N(\bw, \bx_r) = \sum\limits_{i \in I_r} \lambda_i y_i \bx_i +\bv_{\epsilon}~.
\]

\subsubsection*{Proof of 1. $\hat{\bw}$ has the direction of an $(\epsilon + \frac{\epsilon\epsilon_d}{\frac{m}{k}-\epsilon_d} , \delta +\frac{\delta\epsilon_d}{\frac{m}{k}-\epsilon_d} +\frac{7.2 k\epsilon_d}{m})$-approximate KKT point for the margin maximization problem for $S \setminus (\bx_l,y_l)$.}

\paragraph{(1) Dual Feasibility: For all $i \in [m]_{-l}$, $\lambda_i \geq 0$.} Same. directly from dual feasibility for $\bw$ (\defref{def:epsdeltakkt}). 

\paragraph{(2) Stationarity: $\norm{\hat{\bw} - \sum\limits_{i=1}^{m} \lambda_i y_i \nabla_{\bw} N(\hat{\bw}, \bx_i)} \leq \epsilon.$} Same as in \ref{app:linunlearning}.

\paragraph{(3) Complementarity Slackness: For all $t \in [m]_{-l}$, $\lambda_t \left( y_t N(\hat{\bw}, \bx_t)-1 \right) \leq \delta + \frac{1.44k\epsilon_d}{m}$.}
Using the same Equation~\eqref{eq:linearmargin} we get 

\begin{align*}
    1+ \frac{\delta}{\lambda_t} \geq y_t N(\bw, \bx_t) =& y_t N(\hat{\bw}, \bx_t) +  y_t  \sum\limits_{l \in I_f} \lambda_l  y_l |\inner{\bx_l,\bx_t}|\\
    \geq&  y_t N(\hat{\bw}, \bx_t) -  \sum\limits_{l \in I_f} \lambda_l |\inner{\bx_l,\bx_t}| \\
    \geq&  y_t N(\hat{\bw}, \bx_t)  - k\phi  \max_p \lambda_p~,
\end{align*}

plugging in $\phi \leq \frac{\epsilon_d}{4m}$ and the $\lambda_p$ upper bound from \lemref{lem:lin_boundedlambda} we get
\[  y_t N(\hat{\bw}, \bx_t) -k\phi \max_p \lambda_p \geq y_t N(\hat{\bw}, \bx_t) -k\frac{\epsilon_d}{4m} 2.4 \geq y_t N(\hat{\bw}, \bx_t) -\frac{0.6k\epsilon_d}{m} ~.\]

We deduce an upper bound for the margin of $ N(\hat{\bw}, \bx_t)$-
\[y_t N(\hat{\bw}, \bx_t) \leq 1+ \frac{\delta}{\lambda_t} + \frac{0.6k\epsilon_d}{m} = 1+ \frac{\delta + \frac{3}{5m} k\lambda_t \epsilon_d}{\lambda_t}  \leq 1+\frac{\delta + \frac{3}{5m} k2.4  \epsilon_d}{\lambda_t} \leq 1+\frac{\delta + \frac{1.44k\epsilon_d}{m} }{\lambda_t}\] as desired.

\paragraph{(4) Primal Feasibility: For all $t \in [m]_{-l}$, $y_t N(\hat{\bw}, \bx_t) \geq 1-\frac{0.6k\epsilon_d}{m}$.}

We use \eqref{eq:linearmargin} to lower bound the margin of $N(\hat{\bw}, \bx_t)$, and use primal feasibility for $\bw$ (\defref{def:epsdeltakkt}), getting
\begin{align*}
    y_t N(\hat{\bw}, \bx_t) = y_t N(\bw, \bx_t) -  y_t \sum\limits_{l \in I_f} \lambda_l  y_l |\inner{\bx_l,\bx_t}|  \geq 1 - k\phi \max_p \lambda_p ~.
\end{align*}

Plugging in $\phi \leq \frac{\epsilon_d}{4m}$ and the $\lambda_p$ upper bound from \lemref{lem:lin_boundedlambda} we get that
\begin{align*}
     k\phi \max_p \lambda_p \leq \frac{2.4k\epsilon_d}{4m} \leq \frac{0.6k\epsilon_d}{m} ~.
\end{align*}
Hence,  $y_t N(\hat{\bw}, \bx_t) \geq 1- \frac{0.6k\epsilon_d}{m}$.

We showed that $\hat{\bw}$ is an $(\epsilon, \delta + \frac{1.44k\epsilon_d}{m}, \frac{0.6k\epsilon_d}{m})$-approximate KKT by \defref{def:epsdeltagammakkt} . Finally, we look at the scaled weights $\frac{1}{1-\frac{0.6k\epsilon_d}{m}}\hat{\bw}$. For $\epsilon_d \leq 1$ We calculate

% \[\frac{1}{1-\frac{0.6\epsilon_d}{m}}\epsilon \leq \frac{1}{0.4}\epsilon \leq 2.5 \epsilon~, \]
\[\frac{1}{1-\frac{0.6k\epsilon_d}{m}}\epsilon \leq \frac{\frac{m}{k}}{\frac{m}{k}-\epsilon_d}\epsilon = \left(1 + \frac{\epsilon_d}{\frac{m}{k}-\epsilon_d} \right) \epsilon = \epsilon + \frac{\epsilon\epsilon_d}{\frac{m}{k}-\epsilon_d}~, \]

and
% \[\max_p \lambda_p  \frac{\frac{0.6\epsilon_d}{m}}{1-\frac{0.6\epsilon_d}{m}} + \frac{\delta + \frac{1.44\epsilon_d}{m}}{1-\frac{0.6\epsilon_d}{m}}  \leq 2.4  \frac{\frac{0.6\epsilon_d}{m}}{0.4} + \frac{\delta + \frac{1.44\epsilon_d}{m}}{0.4} \leq 2.5 \delta + \frac{7.2 \epsilon_d}{m} \]

\[\max_p \lambda_p  \frac{\frac{0.6k\epsilon_d}{m}}{1-\frac{0.6k\epsilon_d}{m}} + \frac{\delta + \frac{1.44k\epsilon_d}{m}}{1-\frac{0.6k\epsilon_d}{m}}  \leq  \delta +\frac{\delta\epsilon_d}{\frac{m}{k}-\epsilon_d} +\frac{7.2 k\epsilon_d}{m} \]

and get from \lemref{lem:fixthemargin} that $\frac{1}{1-\frac{0.6k\epsilon_d}{m}}\hat{\bw}$ is a $(\epsilon + \frac{\epsilon\epsilon_d}{\frac{m}{k}-\epsilon_d} , \delta +\frac{\delta\epsilon_d}{\frac{m}{k}-\epsilon_d} +\frac{7.2 k\epsilon_d}{m})$-approximate KKT by \defref{def:epsdeltakkt} w.r.t. $S\setminus {(\bx_l, y_l)}$. We note that $\hat{\bw}$ and $\frac{1}{1-0.6k\frac{\epsilon_d}{m}}\hat{\bw}$ have the same direction, which finishes the proof.

\subsubsection*{Proof of 2. $ Cosine-Similarity(\hat{\bw}, \bw^*) \geq 1 - C(\sqrt{\epsilon_d} + \sqrt{\epsilon_d} + \sqrt{\delta})$ for some $C>0$.}

Let $N(\bw^*, \bx)$ be a max-margin linear predictor w.r.t. the remaining training set $S \setminus S_f$. Hence, $\bw^*$ is a KKT point of the margin maximization problem (\eqref{eq:margmax}) w.r.t. $\{\bx_i,y_i\}_{i \in I_f}$, as in \defref{def:epsdeltakkt} (with $\epsilon=\delta=0$). From the stationarity condition we denote $\bw^* = \sum_{i \in I_f}\lambda_i^* y_i\bx_i$. We have same bounds for $\lambda_i$ and $\lambda^*_i$, since it is independent of the unlearning.

The rest of the proof remains the same but the substitution of $\sum_{i \in [m]_{-l}}\lambda_i$ in $\sum_{i \in I_r}\lambda_i$, and the lower bound for it - by \lemref{lem:lin_boundedlambda_vanish}
\begin{align*}
\Lambda = \sum_{i \in I_r} \lambda_i \geq& (m-k)\left(\frac{1}{\norm{\bx_t}^2} - \frac{0.6 \epsilon_d + 1.1\epsilon}{\norm{\bx_t}^2}\right) \\
\geq&(m-k)0.9\left(1 -0.6 \epsilon_d - 1.1\epsilon\right)\\
\geq& 0.1(m-k)~,
\end{align*}

That have no significant effect on the final bound
\begin{align*}
    \frac{\inner{\hat{\bw}, \bw^*}}{\norm{\hat{\bw}}\norm{\bw^*}} &\geq 1 -\left( 1.2 \epsilon_d + 1.1 \epsilon\right) - \sqrt{C ( \epsilon_d + \epsilon + \delta) + \frac{\epsilon^2}{0.1(m-k)} + \frac{\epsilon^2}{0.1(m-k)}C(\epsilon_d +\epsilon + \delta)}\\
    &\geq 1 -C_2\left( \sqrt{\epsilon_d} + \sqrt{\epsilon} + \sqrt{\delta}\right)
\end{align*}

for some constant $C_2 > 0$.
\end{proof}

\subsection{The Identity is an Unsuccessful Unlearning Algorithm}\label{app:linidentity}
To complement Theorem~\ref{thm:linunlearning}, we provide the following remark, that shows that keeping the original predictor is not a successful unlearning algorithm. Particularly, for any $\epsilon', \delta' > 0$, we show that for the predictor as defined in \thmref{thm:linunlearning}, its cosine similarity to any $(\epsilon', \delta')$-approximate KKT point for $S \setminus \{(\bx_l, y_l)\}$ is relatively large.

\begin{remark}\label{remark:linidentity}
    In the same settings as \ref{thm:linunlearning}, the algorithm $\mathcal{A}_I(\btheta, S, r) = \btheta$, is $(\epsilon, \delta, \tau)$-successful only for $\tau \geq \frac{C}{m}  - C(\epsilon_d + \epsilon)$ for some $C>0$.
\end{remark}

As a short intuition for the proof, we note that the original network weight parameter, denoted as 
\[
\bw = \sum\limits_{i=1}^{m} \lambda_i y_i \nabla_{\bw} N(\bw, \bx_i) + \bv_{\epsilon} = \sum\limits_{i=1}^{m} \lambda_i y_i \bx_i + \bv_{\epsilon_1}~,
\]
 consists of a sum of $m$ summons, while any other KKT point w.r.t. $S \setminus \{(\bx_l, y_l)\}$, $\widetilde{\bw}$, consists of a sum of the $(m-1)$ gradients of the remaining dataset. This gap creates an inevitable angle between the two vectors. 
 
\begin{proof}
% \gilad{In the remark, there should be a reference to this part of the appendix}.
In this section, we show that the original network $\bw$ is not a good candidate for the unlearning tasks according to the $(\epsilon, \delta, \tau)$-successful definition (\defref{def:suc_unlearning}). Formally, we look at the simple unlearning algorithm $\mathcal{A}_I(\bw, S, r) = \bw$. We show that for any $(\epsilon', \delta')$-approximate KKT point $\widetilde{\bw}$, where $\epsilon', \delta' < 0.5$ and $\epsilon_d < 0.1$, there exists $C>0$ such that
\[
\cossim(\bw, \widetilde{\bw}) \leq 1 - \frac{C}{m} +  C(\epsilon_d + \epsilon + \widetilde{\epsilon})~,
\]

leading to 
\[ \tau \geq \frac{C}{m}  - C(\epsilon_d + \epsilon + \widetilde{\epsilon})~. \]

% \gilad{The definition of $\tau$ above is unclear, we also use $\tau$ sometime in the paper instead. It will be better to write it explicitly. Either that the cossim is upper bounded, or the algorithm is $(\epsilon,\delta,\frac{C}{m})$-successful where the last part is independent of $\epsilon$ and $\delta$.}
We recall that due to the stationary condition for the original network $\bw$ w.r.t. the full dataset $S$ we have
\[ \bw = \sum\limits_{i \in [m]} \lambda_i y_i \nabla_{\bw} N(\bw, \bx_i) + \bv_\epsilon = \sum\limits_{i=1}^{m} \lambda_i y_i \bx_i + \bv_\epsilon~.\]

We denote an $(\widetilde{\epsilon}, \widetilde{\delta})$-approximate KKT point of the margin maximization problem w.r.t. the retain dataset $S \setminus (\bx_l, y_l)$ by $\widetilde{\bw}$. From the stationarity condition we get that
\[
\widetilde{\bw} = \sum\limits_{i \in [m]_{-l}} \widetilde{\lambda_i} y_i \bx_i + \bv_{\widetilde{\epsilon}}~.
\]
% \gilad{Remind the readers what is $\widetilde{\bw}$.}
Next, we show that the cosine similarity between $\bw$ and $\widetilde{\bw}$ is lower bounded by $\frac{C}{m}  + C(\epsilon_d + \epsilon + \widetilde{\epsilon})$.
We denote $\underline{\bw} = \bw - \bv_\epsilon$ and $\underline{\widetilde{\bw}} = \bw - \bv_{\widetilde{\epsilon}}$. 
For the cosine similarity between $\bw$ and $\widetilde{\bw}$ we have

\[
\cossim(\bw, \widetilde{\bw}) = \frac{\inner{\bw, \widetilde{\bw}}}{\norm{\bw} \norm{\widetilde{\bw}}} = \frac{\inner{\underline{\bw}+ \bv_\epsilon, \underline{\widetilde{\bw}} + \bv_{\widetilde{\epsilon}}}}{\norm{\bw} \norm{\widetilde{\bw}}}
\]

We first use Cauchy–Schwarz inequality and separate it into two expressions
% \gilad{You need to have absolute value to use CS, otherwise the inner products may be negative.}:

\begin{align}
\cossim(\bw, \widetilde{\bw}) &= \frac{\inner{\underline{\bw}+ \bv_\epsilon, \underline{\widetilde{\bw}} + \bv_{\widetilde{\epsilon}}}}{\norm{\bw} \norm{\widetilde{\bw}}} \nonumber \\
&\leq \frac{\inner{\underline{\bw}, \underline{\widetilde{\bw}}}}{\norm{\bw} \norm{\widetilde{\bw}}} +   \frac{|\inner{\bv_\epsilon,\underline{\widetilde{\bw}}}|+ |\inner{\bv_{\widetilde{\epsilon}},\underline{\bw}}|+ |\inner{\bv_\epsilon,\bv_{\widetilde{\epsilon}}}|}{\norm{\bw} \norm{\widetilde{\bw}}} \nonumber\\
&\leq \frac{\inner{\underline{\bw}, \underline{\widetilde{\bw}}}}{\norm{\bw} \norm{\widetilde{\bw}}} +   \frac{\norm{\bv_\epsilon}\norm{\underline{\widetilde{\bw}}}+ \norm{\bv_{\widetilde{\epsilon}}}\norm{\underline{\bw}}+ \norm{\bv_\epsilon}\norm{\bv_{\widetilde{\epsilon}}}}{\norm{\bw} \norm{\widetilde{\bw}}} \label{eq:lincossimsum}
\end{align}

% \begin{align*}
%     \norm{\underline{\bw}}^2 &= \norm{\sum_{i \in [m]}\lambda_iy_i\bx_i}^2 = \inner{\sum_{i \in [m]}\lambda_iy_i\bx_i, \sum_{i \in [m]}\lambda_iy_i\bx_i} = \\
%     &\geq \sum_{i \in [m]}\lambda_i^2 \norm{x_i}^2 - \sum_{i \neq k \in [m]}\lambda_i \lambda_k \inner{\bx_i,\bx_k}\\
%     &\geq \sum_{i \in [m]}\lambda_i^2 \norm{x_i}^2 - \phi \sum_{i \neq k \in [m]}\lambda_i \lambda_k~,
% \end{align*}

% and by Cauchy–Schwarz
% \[
% \norm{\bw}^2 = \norm{\underline{\bw} + \bv_\epsilon}^2 \geq \norm{\underline{\bw}}^2 - 2 \norm{\underline{\bw}} \epsilon - \epsilon^2
% \]

% and similarly
% \begin{align*}
%     \norm{\widetilde{\bw}}^2 =& \norm{\sum_{i \in [m]_{-l}} \widetilde{\lambda_i}y_i \bx_i }^2 = \inner{\sum_{i \in [m]_{-l}} \widetilde{\lambda_i}y_i \bx_i  , \sum_{i \in [m]_{-l}} \widetilde{\lambda_i}y_i \bx_i }\\
%     &\geq \sum_{i \in [m]}\widetilde{\lambda_i}^2 \norm{x_i}^2 - \phi \sum_{i \neq k \in [m]}\widetilde{\lambda_i}\widetilde{\lambda_k}~,
% \end{align*}

% and 
% \[
% \norm{\widetilde{\bw}}^2 = \norm{\underline{\widetilde{\bw}} + \bv_{\widetilde{\epsilon}}}^2 \geq \norm{\underline{\widetilde{\bw}}}^2 - 2 \norm{\underline{\widetilde{\bw}}} \widetilde{\epsilon} - \widetilde{\epsilon}^2
% \]

% @@@@@@@@@@@@@@@@@@@@@@@@@@@@@@@@@@@@@@@@@@@@
We next lower bound the norm of the parameter vectors. 
% If we denote $p_\epsilon = 2 \epsilon\sum_{i \in [m]}\lambda_i \norm{\bx_i} - \epsilon^2$ and 
We note that 

\[
\norm{\bw} = \norm{\underline{\bw}+ \bv_{\epsilon}} \geq \norm{\underline{\bw}} - \epsilon
\]
and

\begin{align*}
    \norm{\underline{\bw}}^2 = \norm{\sum_{i \in [m]}\lambda_iy_i\bx_i}^2 &= \inner{\sum_{i \in [m]}\lambda_iy_i\bx_i, \sum_{i \in [m]}\lambda_iy_i\bx_i} = \\
    % &= \inner{\sum_{i \in [m]}\lambda_iy_i\bx_i,\sum_{i \in [m]}\lambda_iy_i\bx_i} + 2\inner{\sum_{i \in [m]}\lambda_iy_i\bx_i, \bv_{\epsilon}} + \inner{\bv_\epsilon, \bv_\epsilon} \\
    &\geq \sum_{i \in [m]}\lambda_i^2 \norm{x_i}^2 - \sum_{i \neq k \in [m]}\lambda_i \lambda_k \inner{\bx_i,\bx_k} \\%- 2 \sum_{i \in [m]}\lambda_i\inner{\bx_i,\bv_\epsilon} + \epsilon^2\\
    &\geq \sum_{i \in [m]}\lambda_i^2 \norm{x_i}^2 - \phi \sum_{i \neq k \in [m]}\lambda_i \lambda_k ~. %- 2 \epsilon\sum_{i \in [m]}\lambda_i \norm{\bx_i}  + \epsilon^2\\
    % &\geq \sum_{i \in [m]}\lambda_i^2 \norm{x_i}^2 - \phi \sum_{i \neq k \in [m]}\lambda_i \lambda_k % - p_\epsilon~,
\end{align*}

% Where the last inequality holds due to $\psi \leq 0.1$ and \lemref{lem:lin_boundedlambda}.
Similarly $\norm{\widetilde{\bw}} \geq \norm{\underline{\widetilde{\bw}}} - \widetilde{\epsilon}$ and 

\begin{align*}
    \norm{\underline{\widetilde{\bw}}}^2 =& \norm{\sum_{i \in [m]_{-l}} \widetilde{\lambda_i}y_i \bx_i}^2 = \inner{\sum_{i \in [m]_{-l}} \widetilde{\lambda_i}y_i \bx_i , \sum_{i \in [m]_{-l}} \widetilde{\lambda_i}y_i \bx_i  }\\
    % &\geq \sum_{i \in [m]}\widetilde{\lambda_i}^2 \norm{x_i}^2 - \phi \sum_{i \neq k \in [m]}\widetilde{\lambda_i}\widetilde{\lambda_k}  - 2 \widetilde{\epsilon}\sum_{i \in [m]}\widetilde{\lambda_i} \norm{\bx_i} + \widetilde{\epsilon}^2 \\ 
    &\geq \sum_{i \in [m]_{-l}}\widetilde{\lambda_i}^2 \norm{x_i}^2 - \phi \sum_{i \neq k \in [m]_{-l}}\widetilde{\lambda_i}\widetilde{\lambda_k} ~.
\end{align*}

We now upper bound the inner product $\inner{\underline{\bw}, \underline{\widetilde{\bw}}}$, having 
% \gilad{Also here we need absolute value to bound this way}

\begin{align*}
    \inner{\underline{\bw}, \underline{\widetilde{\bw}}} &= \inner{\sum_{i \in [m]}\lambda_iy_i\bx_i, \sum_{i \in [m]_{-l}} \widetilde{\lambda_i}y_i \bx_i } =\\
    &=\inner{\sum_{i \in [m]_{-l}}\lambda_iy_i\bx_i, \sum_{i \in [m]_{-l}} \widetilde{\lambda_i}y_i \bx_i } + \inner{\sum_{i \in [m]_{-l}} \widetilde{\lambda_i}y_i \bx_i, \lambda_l y_l \bx_l}\\
    &\leq |\inner{\sum_{i \in [m]_{-l}}\lambda_iy_i\bx_i, \sum_{i \in [m]_{-l}} \widetilde{\lambda_i}y_i \bx_i }| + |\inner{\sum_{i \in [m]_{-l}} \widetilde{\lambda_i}y_i \bx_i, \lambda_l y_l \bx_l}|\\
    &\leq \sum_{i \in [m]_{-l}}\widetilde{\lambda_i}\lambda_i \norm{\bx_i}^2 +  \sum_{i \neq k \in [m]_{-l}} \widetilde{\lambda_i}\lambda_k \inner{\bx_i,\bx_k} + \sum_{i \in [m]_{-l}} \widetilde{\lambda_i}\lambda_l \inner{\bx_i,\bx_l} \\
    &\leq \sum_{i \in [m]_{-l}}\widetilde{\lambda_i}\lambda_i \norm{\bx_i}^2 +  \phi \sum_{i \neq k \in [m]_{-l}} \widetilde{\lambda_i}\lambda_k  + \phi \sum_{i \in [m]_{-l}} \widetilde{\lambda_i}\lambda_l  \\
\end{align*}

Plug it all in, we get for the first summon at \eqref{eq:lincossimsum}

\begin{align*}
    &\frac{\inner{\underline{\bw}, \underline{\widetilde{\bw}}}}{\norm{\bw} \norm{\widetilde{\bw}}} \leq \\
    &\frac{\sum_{i \in [m]_{-l}}\widetilde{\lambda_i}\lambda_i \norm{\bx_i}^2 +  \phi \sum_{i \neq k \in [m]_{-l}} \widetilde{\lambda_i}\lambda_k  + \phi \sum_{i \in [m]_{-l}} \widetilde{\lambda_i}\lambda_l}{\left(\sqrt{\sum_{i \in [m]}\lambda_i^2 \norm{x_i}^2 - \phi \sum_{i \neq k \in [m]}\lambda_i \lambda_k } - \epsilon\right) \left(\sqrt{\sum_{i \in [m]_{-l}}\widetilde{\lambda_i}^2 \norm{x_i}^2 - \phi \sum_{i \neq k \in [m]_{-l}}\widetilde{\lambda_i}\widetilde{\lambda_k}  }-{\widetilde{\epsilon}}\right)}~.
\end{align*}

We first note that by Cauchy–Schwarz 
\[
\sum_{i \in [m]_{-l}}\widetilde{\lambda_i}\lambda_i \norm{\bx_i}^2 \leq \sqrt{\sum_{i \in [m]_{-l}}\widetilde{\lambda_i}^2 \norm{\bx_i}^2} \sqrt{\sum_{i \in [m]_{-l}}\lambda_i^2 \norm{\bx_i}^2}~,
\]
and
\[
\sum_{i \in [m]_{-l}}\widetilde{\lambda_i}\lambda_i \leq \sqrt{\sum_{i \in [m]_{-l}}\widetilde{\lambda_i}^2 } \sqrt{\sum_{i \in [m]_{-l}}\lambda_i^2}~.
\]

We now reduce the nominator and denominator by $\sqrt{\sum_{i \in [m]_{-l}}\widetilde{\lambda_i}^2 \norm{\bx_i}^2} \sqrt{\sum_{i \in [m]_{-l}}\lambda_i^2 \norm{\bx_i}^2}$. We denote $b = (1+1.2 \epsilon_d + 2.15 \epsilon + 2.2 \delta), a = (1- 0.6\epsilon_d - 1.1 \epsilon)$, and use \lemref{lem:lin_boundedlambda_vanish} in which for all $i$, $a<\lambda_i \norm{\bx_i}^2<b$. We calculate the summons in the nominator after reduction, having

\[
\frac{\sum_{i \in [m]_{-l}}\widetilde{\lambda_i}\lambda_i \norm{\bx_i}^2}{\sqrt{\sum_{i \in [m]_{-l}}\widetilde{\lambda_i}^2 \norm{\bx_i}^2} \sqrt{\sum_{i \in [m]_{-l}}\lambda_i^2 \norm{\bx_i}^2}} \leq 1~,
\]

\[
\frac{\phi \sum_{i \neq k \in [m]_{-l}} \widetilde{\lambda_i}\lambda_k}{ \sqrt{\sum_{i \in [m]_{-l}}\widetilde{\lambda_i}^2 \norm{\bx_i}^2} \sqrt{\sum_{i \in [m]_{-l}}\lambda_i^2 \norm{\bx_i}^2}} \leq \frac{\phi \sqrt{\sum_{i\neq k \in [m]_{-l}}\widetilde{\lambda_i}^2 } \sqrt{\sum_{i\neq k \in [m]_{-l}}\lambda_i^2}}{ \sqrt{\sum_{i \in [m]_{-l}}\widetilde{\lambda_i}^2 \norm{\bx_i}^2} \sqrt{\sum_{i \in [m]_{-l}}\lambda_i^2 \norm{\bx_i}^2}} \leq \frac{\epsilon_d}{3.6}~,
\]
\[
\frac{\phi \sum_{i \in [m]_{-l}} \widetilde{\lambda_i}\lambda_l}{ \sqrt{\sum_{i \in [m]_{-l}}\widetilde{\lambda_i}^2 \norm{\bx_i}^2} \sqrt{\sum_{i \in [m]_{-l}}\lambda_i^2 \norm{\bx_i}^2}} \leq \frac{\phi \sum_{i \in [m]_{-l}} \widetilde{\lambda_i}\lambda_l}{\sum_{i \in [m]_{-l}}\widetilde{\lambda_i}\lambda_i \norm{\bx_i}^2 }\leq \frac{1.2b\epsilon_d}{4ma} ~.
\]

and for the denominator we have
% \gilad{I don't understand the calculations above}
\[
\frac{\sum_{i \in [m]}\lambda_i^2 \norm{x_i}^2}{\sum_{i \in [m]}\lambda_i^2 \norm{x_i}^2} = 1~,
\]
\[
\frac{\phi \sum_{i \neq k \in [m]_{-l}}\lambda_i \lambda_k}{{\sum_{i \in [m]_{-l}}\lambda_i^2 \norm{x_i}^2}} \leq \frac{\phi \sqrt{\sum_{i\neq k \in [m]_{-l}}{\lambda_i}^2 } \sqrt{\sum_{i\neq k \in [m]_{-l}}\lambda_i^2}}{{\sum_{i \in [m]_{-l}}\lambda_i^2 \norm{x_i}^2}}
\leq \frac{\phi (m-1)\sum_{i \in [m]_{-l}}{\lambda_i}^2}{{\sum_{i \in [m]_{-l}}\lambda_i^2 \norm{x_i}^2}} \leq \frac{\epsilon_d}{3.6}~,
\]

\[
\frac{{\epsilon}}{\sqrt{\sum_{i \in [m]_{-l}}\lambda_i^2 \norm{x_i}^2}} \leq \frac{\epsilon}{0.9a\sqrt{m}}~,
\]
the same for $\widetilde{\lambda_i}$ and $\widetilde{\epsilon}$, and finally
\[
\frac{\widetilde{\lambda_l}^2 \norm{x_l}^2}{\sum_{i \in [m]_{-l}}\widetilde{\lambda_i}^2 \norm{x_i}^2} \leq \frac{2.4 b}{0.91a^2m} \leq \frac{2.64b}{am}~.
\]

Plug it all in we have

\begin{align*}
    &\frac{\inner{\underline{\bw}, \underline{\widetilde{\bw}}}}{\norm{\bw} \norm{\widetilde{\bw}}} \leq \\
    &\frac{\sum_{i \in [m]_{-l}}\widetilde{\lambda_i}\lambda_i \norm{\bx_i}^2 +  \phi \sum_{i \neq k \in [m]_{-l}} \widetilde{\lambda_i}\lambda_k  + \phi \sum_{i \in [m]_{-l}} \widetilde{\lambda_i}\lambda_l}{\sqrt{\sum_{i \in [m]}\lambda_i^2 \norm{x_i}^2 - \phi \sum_{i \neq k \in [m]}\lambda_i \lambda_k  - \epsilon} \sqrt{\sum_{i \in [m]}\widetilde{\lambda_i}^2 \norm{x_i}^2 - \phi \sum_{i \neq k \in [m]}\widetilde{\lambda_i}\widetilde{\lambda_k}  -{\widetilde{\epsilon}}}} \\
    &\leq \frac{1 + 0.28\epsilon_d + \frac{1.2b\epsilon_d}{4ma}}{\sqrt{1 - 0.28\epsilon_d - \frac{\epsilon}{0.9a\sqrt{m}}} \sqrt{1 - 0.28\epsilon_d - \frac{\widetilde{\epsilon}}{0.9a\sqrt{m}} + \frac{2.64b}{am}}}\\
    &\leq \frac{1 + 0.28\epsilon_d + \frac{1.2b\epsilon_d}{4ma}}{\left(1 - 0.28\epsilon_d - \frac{\epsilon+\widetilde{\epsilon}}{0.9a\sqrt{m}}\right)\sqrt{1+ \frac{2.64b}{am}}} \\
    % &\leq 1 + \frac{1 + 0.28\epsilon_d + \frac{1.2b\epsilon_d}{4ma} - \left(1 - 0.28\epsilon_d - \frac{2.2(\epsilon+\widetilde{\epsilon})}{a}\right)\left(1+ \frac{2.32b}{a^2m}\right) }{\left(1 - 0.28\epsilon_d - \frac{2.2(\epsilon+\widetilde{\epsilon})}{a}\right)\left(1+ \frac{2.32b}{a^2m}\right)}
\end{align*}

for any $0<x<1$ we get that 
\[
\frac{1}{\sqrt{1+x}} \leq 1 - \frac{x}{4} %+ \frac{3x^2}{8} 
\]
and thus in conclusion we have 
\begin{align*}
    &\frac{\inner{\underline{\bw}, \underline{\widetilde{\bw}}}}{\norm{\bw} \norm{\widetilde{\bw}}} \leq \\
    &\leq \frac{1 + 0.28\epsilon_d + \frac{1.2b\epsilon_d}{4ma}}{\left(1 - 0.28\epsilon_d - \frac{\epsilon+\widetilde{\epsilon}}{0.9a\sqrt{m}}\right)\sqrt{1+ \frac{2.64b}{am}}} \\
    &\leq \frac{1 + 0.28\epsilon_d + \frac{1.2b\epsilon_d}{4ma}}{1 - 0.28\epsilon_d - \frac{\epsilon+\widetilde{\epsilon}}{0.9a\sqrt{m}}}\left(1 - \frac{0.66b}{am}\right)\\
    &\leq \left(1+\frac{0.56\epsilon_d + \frac{1.2b\epsilon_d}{4ma} + \frac{\epsilon+\widetilde{\epsilon}}{0.9a\sqrt{m}}}{1 - 0.28\epsilon_d - \frac{\epsilon+\widetilde{\epsilon}}{0.9a\sqrt{m}}} \right)\left(1 - \frac{0.66b}{am}\right)\\
    &\leq 1 - \frac{C}{m}  + C(\epsilon_d + \epsilon + \widetilde{\epsilon})~,
\end{align*}

which finishes the upper bounded the first summon of the cosine similarity at \eqref{eq:lincossimsum}. We now upper bound the second summon, we recall that $\underline{\bw} = \bw - \bv_\epsilon$ and therefore $\norm{\underline{\bw}} \leq \norm{\bw} + \epsilon$, and similar for $\widetilde{\bw}$, and thus,
\begin{align*}
    \frac{\norm{\bv_\epsilon}\norm{\underline{\widetilde{\bw}}}+ \norm{\bv_{\widetilde{\epsilon}}}\norm{\underline{\bw}}+ \norm{\bv_\epsilon}\norm{\bv_{\widetilde{\epsilon}}}}{\norm{\bw} \norm{\widetilde{\bw}}} \leq \frac{\epsilon\norm{{\widetilde{\bw}}} + \epsilon^2 + \widetilde{\epsilon}\norm{{\bw}}+ \widetilde{\epsilon}^2 + \epsilon \widetilde{\epsilon}}{\norm{\bw} \norm{\widetilde{\bw}}} = \frac{\epsilon}{\norm{\bw}} + \frac{\widetilde{\epsilon}}{\norm{\widetilde{\bw}}} + \frac{\epsilon^2 + \widetilde{\epsilon}^2 + \epsilon\widetilde{\epsilon}}{\norm{\bw} \norm{\widetilde{\bw}}}
\end{align*}

We look at the norm lower bound. We note that 

\[
\norm{\bw} = \norm{\underline{\bw}+ \bv_{\epsilon}} \geq \norm{\underline{\bw}} - \epsilon~,
\]
and
\begin{align*}
    \norm{\underline{\bw}}^2  &= \inner{\sum_{i \in [m]}\lambda_iy_i\bx_i, \sum_{i \in [m]}\lambda_iy_i\bx_i} = \\
    &\geq \sum_{i \in [m]}\lambda_i^2 \norm{x_i}^2 - \phi \sum_{i \neq k \in [m]}\lambda_i \lambda_k \\
    &\geq \sum_{i \in [m]}\lambda_i \left[ a - \phi m b \right]\\
    &\geq m 0.9 a \left[ a - 0.6\epsilon_d \right]\\
    &\geq m 0.9 a \left[  1- 1.2\epsilon_d - 1.1 \epsilon \right] \geq 0.1 m~,
\end{align*}

and similarly $\norm{\underline{\widetilde{\bw}}}^2 \geq 0.1 (m-1)$. Plug in to the denominator of the above fraction we get 

\begin{align*}
    \frac{\epsilon}{\norm{\bw}} + \frac{\widetilde{\epsilon}}{\norm{\widetilde{\bw}}} + \frac{\epsilon^2 + \widetilde{\epsilon}^2 + \epsilon\widetilde{\epsilon}}{\norm{\bw} \norm{\widetilde{\bw}}} \leq \frac{\epsilon}{0.1m - \epsilon} + \frac{\widetilde{\epsilon}}{0.1(m-1) - \widetilde{\epsilon}}+ \frac{\epsilon^2 + \widetilde{\epsilon}^2 + \epsilon\widetilde{\epsilon}}{(0.1(m-1) -\epsilon)^2} \leq C_1(\epsilon_d+ \epsilon+ \widetilde{\epsilon})
\end{align*}
% \gilad{What is the above calculation related to?}

which means that there exists $C$  such that 
\begin{align*}
    \cossim(\bw, \widetilde{\bw}) \leq 1 - \frac{C}{m}  + C(\epsilon_d + \epsilon + \widetilde{\epsilon})~,
\end{align*}

Thus, concluding the proof.
\end{proof}

\section{Proofs for \secref{sec:nonlin}}
\label{app:nonlin}

\subsection{lemmas for Proof \ref{app:nonlinunlearning} of \thmref{thm:nonlinunlearning}}
\begin{lemma}
\label{lem:fixbounds}
% \gilad{This lemma is very straightforward and only used once. I suggest just adding the claim to the proof.}
Let $S = \{(\bx_1,y_1),...,(\bx_m,y_m)\}$ such that $\forall i \in [m], \bx_i \in \reals^d$ and let $\{\bw_j\}_{j=1}^n$, $\forall j \in [n], \bw_j \in \reals^d$. Assume the data distribution $\mathcal{D}$ satisfies \assref{ass:data} for some $\psi, \phi$. Given $l \in [m]$ and $c \in \reals$, for $j \in [n]$ and $r \in [m]_{-l}$, we denote 
\[\Delta_{r,j}= \sum\limits_{k \in [m]_{-l}} c \inner{\bx_k,\bx_r} \sign(\inner{\bx_k, \bw_j})~.\] Then,
\label{lem:additionbound}
    \begin{align*}
        \bw_j^\top \bx_r \geq 0 \Rightarrow& \\
        &c(1-\psi) - (m-2)c \phi \leq \Delta_{r,j} \leq c(1+\psi) + (m-2)c \phi\\
    \bw_j^\top \bx_r < 0 \Rightarrow&\\
    &-c(1+\psi) - (m-2)c \phi \leq \Delta_{r,j} \leq -c(1-\psi) + (m-2)c \phi
    \end{align*}
\end{lemma}

\begin{proof}
    \begin{align*}
        &\sum\limits_{k \in [m]_{-l}} c \inner{\bx_k,\bx_r} \sign(\inner{\bx_k, \bw_j}) =\\
        &= c \norm{\bx_r}^2 \sign(\inner{\bx_r, \bw_j})  + \sum\limits_{k \in [m]_{-l}, k\neq r} c \inner{\bx_k, \bx_r} \sign(\inner{\bx_k, \bw_j}) \\
    \end{align*}
    From \assref{ass:data} we know that $(1-\psi) \leq \norm{\bx_r}^2 \leq (1+\psi)$, for $ k \neq r$, $ -\phi  \leq  \inner{\bx_k,\bx_r} \leq \phi$ which finishes the proof.  
\end{proof}

\begin{lemma}
\label{lem:fixupperbounds}
Let $S = \{(\bx_1,y_1),...,(\bx_m,y_m)\}$ such that $\forall i \in [m], \bx_i \in \reals^d$ and let $\{\bw_j\}_{j=1}^n$ $\forall j \in [n], \bw_j \in \reals^d$. Assume the data distribution $\mathcal{D}$ satisfies \assref{ass:data} for some $\psi \leq 0.1, \phi \leq \frac{\epsilon_d}{4mn}$. Given $l \in [m]$, and $c= \frac{\epsilon_d}{2mn}$, for $j \in [n]$ and $r \in [m]_{-l}$, we denote  
\[\Delta_{j}= \sum\limits_{k \in [m]_{-l}} c \bx_k \sign(\inner{\bx_k, \bw_j})\] Then for $j \in [n]$,
\[ \norm{|u_j| \lambda_l \sigma'_{l,j} \Delta_{j}} \leq  \frac{22 \epsilon_d}{\sqrt{mn}} \]
% \gilad{What is $\epsilon_1$? Also, the notation of $F_j$ is redundant.}
% \gilad{!!!! The bound, and also all the bounds following it can be significantly simpler if we just care about the dependence on the $\epsilon,\epsilon_d,\delta$. Since $n,m \geq 1$ and even $m \geq 2$ whenever you divide by them, since this is an upper bound you can just drop them. For example, here you would get: $\leq \epsilon_d(1.5 + 4\epsilon^2 + 1.4\delta + \frac{1}{2}\sqrt{1.1 + \epsilon_d}) \leq C\cdot \epsilon_d (1 + \epsilon^2 + \delta)$ for some constant $C > 0$ (note that I used $\epsilon_d \leq 1$). This last bound is much nicer but also looks a bit weird, so  Since $\epsilon,\delta \leq 1$ this terms is bounded by $C'\cdot \epsilon_d$ for some other constant $C'$.}
\end{lemma}

\begin{proof}
We first look at the norm of $\Delta_j$, having

\begin{align*}
    &\norm{\Delta_j}^2 = \norm{\sum\limits_{k\in [m]_{-l}} c \bx_k \sign(\inner{\bx_k, \bw_j})}^2 = \\
    &= \inner{\sum\limits_{k \in [m]_{-l}} c \bx_k \sign(\inner{\bx_k, \bw_j}), \sum\limits_{k \in [m]_{-l}} c \bx_k \sign(\inner{\bx_k, \bw_j})} \\
    &\leq c^2 \inner{\sum\limits_{k \in [m]_{-l}}  \bx_k , \sum\limits_{k\in [m]_{-l}}  \bx_k }\\
    &\leq c^2 \left[ \sum\limits_{k\in [m]_{-l}} \norm{\bx_i}^2 + \sum\limits_{s \neq k \in [m]_{-l}} \inner{\bx_k,\bx_s} \right]\\
    &\leq c^2 \left( m (1+\psi) + m^2 \phi \right)
\end{align*}

we plug in $\psi \leq 0.1, \phi \leq \frac{\epsilon_d}{4mn}, c= \frac{\epsilon_d}{2mn}$ and get 
\[
\norm{\Delta_j}^2 \leq \frac{\epsilon_d^2}{4m^2n^2} \left( 1.1m + m^2 \frac{\epsilon_d}{4mn}\right) = \frac{\epsilon_d^2 \left(1.1 + \frac{\epsilon_d}{4n}\right)}{4 m n^2} 
\]
and 
\[
\norm{\Delta_j} \leq \frac{\epsilon_d \sqrt{1.1+\frac{\epsilon_d}{n}}}{2\sqrt{m}n}
\]

% \gilad{!!! Check this transition again. Opening a norm will have terms including both $\norm{\bx_i}$ and $\inner{\bx_i,\bx_j}$. Also, to make this calculation easier consider bounding $\norm{\Delta_j}^2$ and taking a square root at the end. }
% \begin{align*}
%      \norm{\Delta_j} \leq  m \cdot c (1+\psi) \leq m \frac{\epsilon_d}{2mn} 1.1 \leq \frac{1.1\epsilon_d}{2n}
% \end{align*}
From \lemref{lem:lambda_upperbound} we have that $\max_{i \in [m]}\lambda_i \leq  20.4n$. As for all $j\in [n], |u_j| = \frac{1}{\sqrt{n}}$, and $\sigma_{l,j}' \geq 0 $,
joining all together we have
\begin{align*}
   \norm{|u_j| \lambda_l \sigma'_{l,j} \Delta_{j}} = |u_j| \lambda_l \sigma_{l,j}' \norm{\Delta_j} \leq& \frac{1}{\sqrt{n}}20.4n \frac{\epsilon_d \sqrt{1.1+\frac{\epsilon_d}{n}}}{2\sqrt{m}n} \\
    \leq& \frac{1}{\sqrt{n}}20.4 \frac{\epsilon_d \sqrt{1.1+\frac{\epsilon_d}{n}}}{2\sqrt{m}} \\
    \leq& \frac{\epsilon_d \left(20.4 + \frac{1}{2}\sqrt{1.1+\frac{\epsilon_d}{n}} \right)}{\sqrt{nm}} \leq \frac{22 \epsilon_d}{\sqrt{mn}} ~,
    % < \frac{1.6 \epsilon_1}{\sqrt{n}}
\end{align*}
as desired.
\end{proof}

\begin{lemma}
\label{lem:lambda_upperbound}
    Let $N(\btheta,\bx)  = \sum\limits_{j=1}^{n} u_j \sigma(\bw_j^\top \bx)$ be a two-layer fully connected neural network, trained on $S = \{(\bx_1,y_1),...,(\bx_m,y_m)\}$, and let $0 < \epsilon_d, \epsilon, \delta \leq 1$ such that $\btheta$ is an $(\epsilon, \delta)$-approximate $KKT$ point for the margin maximization problem for $S$ according to \defref{def:epsdeltakkt} for $\lambda_1,...,\lambda_m$, and $S$ satisfies \assref{ass:data} for $\psi = 0,1$, and $\phi \leq \frac{\epsilon_d}{4mn}$. Assume $\forall j \in [n], u_j \sim \mathcal{U}\{-\frac{1}{\sqrt{n}}, \frac{1}{\sqrt{n}}\}$. Then, For $i \in [m]$ we have
    \[
     \max \left\{ \sum\limits_{j \in J_+} u_j^2 \lambda_i \sigma'_{i,j} , \sum\limits_{j \in J_-} u_j^2 \lambda_i \sigma'_{i,j} \right \} \leq 2.5 +5.25\epsilon + 2.4\delta \leq 10.2~,
     \]
     and therefore also
    \[
     \sum\limits_{j =1}^{n}u_j^2 \lambda_i \sigma'_{i,j}  \leq 5 +10.5\epsilon + 4.8\delta \leq 20.4~,
    \]
    and 
    \[
     \lambda_i \leq  n \left(5 +10.5\epsilon + 4.8\delta\right)\leq 20.4n~.
    \]
    
    % \gilad{!!! This term can be much simplified. First, the terms $\frac{\epsilon^2}{n}$ can be bounded by $1$. Second, you can use that $\sqrt{a + b} \leq \sqrt{a} + \sqrt{b}$ for positive numbers to take the constant out of the root. Finally, you can replace all the numbers by some large constant $C > 0$. I think the final bound can look like $C\cdot \frac{1 + \delta x}{x}$ if we denote $x = (1-\psi)  - 2\phi(m-1)$. This all depends on $x \geq 0$ which isn't currently part of the assumptions of the lemma. But I think you do use this assumption in the prove when you solve the polynomial.}
\end{lemma}

\begin{proof}
    Let $J_+ = \{j \in [n]: u_j > 0\}$ and  $J_- = \{j \in [n]: u_j < 0\}$. Denote $\alpha_+ = \max_{i \in [m]}\left( \sum\limits_{j \in J_+} u_j^2 \lambda_i \sigma'_{i,j} \right)$ and $\alpha_- = \max_{i \in [m]}\left( \sum\limits_{j \in J_-} u_j^2 \lambda_i \sigma'_{i,j} \right)$. w.l.o.g. we assume $\alpha_+ \geq \alpha_-$ (the other direction is proven similarly). We denote $\alpha = \alpha_+ = \max_{i \in [m]} \left( \sum\limits_{j \in J_+} u_j^2 \lambda_i \sigma'_{i,j}\right)$, and $k=\argmax_{i \in [m]}\left( \sum\limits_{j \in J_+} u_j^2 \lambda_i \sigma'_{i,j} \right)$. If $\lambda_k = 0$ the claim follows. 

    Using the stationarity condition in \defref{def:epsdeltakkt} for $\btheta$, we denote $\bv_{\epsilon} = \btheta - \sum_{i=1}^m\lambda_iy_i \nabla_{\btheta} N(\btheta, \bx_i)$, and $\bv_{\epsilon,j} = \bw_j - \sum\limits_{i=1}^{m} u_j  \lambda_i y_i \sigma'_{i,j} \bx_i$, such that $\bv_{\epsilon}$ is the concatenation of all $\bv_{\epsilon,j}$ and $\norm{\bv_{\epsilon}} = \epsilon$. Using this notation we have for all $j \in [n]$ the inner product
    \begin{align*}
    \bw_j^\top \bx_k =& u_j \sum\limits_{i=1}^{m}  \lambda_i y_i \sigma'_{i,j} \inner{\bx_i,\bx_k} + \inner{\bv_{\epsilon,j}, \bx_k}\\
    =& u_j \lambda_k y_k \sigma'_{k,j} \norm{\bx_k}^2 + u_j \sum\limits_{i=1, i\neq k}^{m}  \lambda_i y_i \sigma'_{i,j} \inner{\bx_i,\bx_k} + \inner{\bv_{\epsilon,j}, \bx_k}~.
    \end{align*}

To upper bound $\alpha$, we use the Complementarity Slackness condition in \defref{def:epsdeltakkt} to first bound the margin, and then solve for $\alpha$. First, since for all $j \in [n]$ and $k \in [m]$,  $|u_j| = \frac{1}{\sqrt{n}}$ and $\sigma'_{k,j} \leq 1$, we get
    that $\alpha \leq \lambda_k \frac{1}{n}\sum\limits_{j=1}^{n} \sigma'_{k,j} \leq \lambda_k$, so $\frac{1}{\lambda_k} \leq \frac{1}{\alpha}$.

Then, using the Complementarity Slackness condition for $\btheta$ we get that $y_k N(\btheta, \bx_k) \leq 1+ \frac{\delta}{\lambda_k} \leq 1+ \frac{\delta}{\alpha}$. To use the $\alpha$ notation we express the margin with in terms of sums over $J_+$ and $J_-$
    \[1 + \frac{\delta}{\alpha} \geq y_k N(\btheta, \bx_k) = y_k\sum\limits_{j=1}^{n} u_j \sigma(\bw_j^\top \bx_k)
        =  y_k \left[ \sum\limits_{j \in J_+} u_j \sigma(\bw_j^\top \bx_k) +\sum\limits_{j \in J_-} u_j \sigma(\bw_j^\top \bx_k) \right]~. \]
    
    Now, to divide both sides of the inequality by $y_k$, we need to know its sign. We separate to two cases for $y_k$:
    
    \textbf{Case 1:} $y_k =1$\\
    We lower bound the margin
    \begin{align*}
        1 + \frac{\delta}{\alpha} \geq N(\btheta, \bx_k) =& \sum\limits_{j \in J_+} u_j \sigma(\bw_j^\top \bx_k) +\sum\limits_{j \in J_-} u_j \sigma(\bw_j^\top \bx_k)\\
        \geq&  \sum\limits_{j \in J_+} u_j \bw_j^\top \bx_k + \sum\limits_{j \in J_-} u_j \sigma(\bw_j^\top \bx_k)~,
    \end{align*}
    Where the last inequality hold since for all $y \in \reals$, $y \leq \sigma(y)$.
    We lower bound separately the first summand, getting
     \begin{align*}
         \sum\limits_{j \in J_+} u_j \bw_j^\top \bx_k =& \sum\limits_{j \in J_+} u_j \left( u_j \lambda_k \sigma'_{k,j} \norm{\bx_k}^2 + u_j \sum\limits_{i=1, i\neq k}^{m}  \lambda_i y_i \sigma'_{i,j} \inner{\bx_i,\bx_k} + \inner{\bv_{\epsilon,j}, \bx_k} \right)\\
         \geq& (1-\psi) \sum\limits_{j \in J_+} u_j^2 \lambda_k \sigma'_{k,j}  - \phi \sum\limits_{j \in J_+} \sum\limits_{i=1, i\neq k}^{m} u_j^2 \lambda_i \sigma'_{i,j} -  \sum\limits_{j \in J_+} u_j |\inner{\bv_{\epsilon,j}, \bx_k}|\\
         \geq& (1-\psi) \alpha - \phi (m-1) \alpha -  \sum\limits_{j \in J_+} u_j |\inner{\bv_{\epsilon,j}, \bx_k}|~.
     \end{align*}

    Using Cauchy–Schwarz inequality we have
     \begin{align*}
         \sum\limits_{j \in J_+} u_j |\inner{\bv_{\epsilon,j}, \bx_k}| = \frac{1}{\sqrt{n}}\sum\limits_{j \in J_+} |\inner{\bv_{\epsilon,j}, \bx_k}| \leq \frac{1}{\sqrt{n}} \norm{\bv_\epsilon} \sqrt{n} \max_{p \in [m]} \norm{x_p} \leq \epsilon \sqrt{1+ \psi}~,
     \end{align*}

     getting
     \[\sum\limits_{j \in J_+} u_j \bw_j^\top \bx_k \geq (1-\psi) \alpha - \phi (m-1) \alpha - \epsilon \sqrt{1+ \psi}~.\]

     Bounding the second summand we have
    \begin{align*}
         \sum\limits_{j \in J_-} u_j \sigma(\bw_j^\top \bx_k) =& \sum\limits_{j \in J_-} u_j \sigma \left( u_j \lambda_k \sigma'_{k,j} \norm{\bx_k}^2 + u_j \sum\limits_{i=1, i\neq k}^{m}  \lambda_i y_i \sigma'_{i,j} \inner{\bx_i,\bx_k} + \inner{\bv_{\epsilon,j}, \bx_k}\right)\\
        &\geq \sum\limits_{j \in J_-} u_j \sigma \left(u_j \sum\limits_{i=1, i\neq k}^{m}  \lambda_i y_i \sigma'_{i,j} \inner{\bx_i,\bx_k} + |\inner{\bv_{\epsilon,j}, \bx_k}| \right)\\
        &\geq \sum\limits_{j \in J_-} u_j \sigma \left(|u_j| \sum\limits_{i=1, i\neq k}^{m}  \lambda_i \sigma'_{i,j} |\inner{\bx_i,\bx_k}| + |\inner{\bv_{\epsilon,j}, \bx_k}| \right)\\
        &\geq \sum\limits_{j \in J_-} u_j \sigma \left(|u_j| \sum\limits_{i=1, i\neq k}^{m}  \lambda_i \sigma'_{i,j} \phi + |\inner{\bv_{\epsilon,j}, \bx_k}|\right)\\
        &\geq -\phi \sum\limits_{j \in J_-} \sum\limits_{i=1, i\neq k}^{m} u_j^2 \lambda_i \sigma'_{i,j}- \sum\limits_{j \in J_-} |u_j| |\inner{\bv_{\epsilon,j}, \bx_k}|\\
        &\geq -\phi (m-1) \alpha - \epsilon \sqrt{1+ \psi}~,
     \end{align*}

and combining the two results we have
\begin{align*}
        1+ \frac{\delta}{\alpha} \geq 1 +\frac{\delta}{\lambda_k} \geq y_k N(\btheta, \bx_k) \geq& \sum\limits_{j \in J_+} u_j \bw_j^\top \bx_k + \sum\limits_{j \in J_-} u_j \sigma(\bw_j^\top \bx_k)\\
        \geq& (1-\psi) \alpha - \phi (m-1) \alpha - \phi (m-1) \alpha\\
        =& \alpha\left((1-\psi) - 2\phi (m-1)\right) - 2\epsilon \sqrt{1+ \psi}~,
    \end{align*}
    getting
    \[
    \alpha^2\left((1-\psi) - 2\phi (m-1)\right) - \alpha\left(1 + 2\epsilon \sqrt{1+ \psi} \right) - \delta \leq 0~.
    \]
    Note, for our setting $\psi \leq 0.1$ and $\phi \leq \frac{\epsilon_d}{4mn}$ so we get that 
    \[
    (1-\psi) - 2\phi (m-1) \geq 0.9 - 2 \frac{\epsilon_d}{4mn} (m-1) \geq 0.9 -\frac{\epsilon_d}{2n} > 0~,
    \]
    hence solving for $\alpha$ we get 
    \[\alpha \leq \frac{1 + 2\epsilon \sqrt{1+ \psi} + \sqrt{(1 + 2\epsilon \sqrt{1+ \psi})^2 + 4\delta\left((1-\psi) - 2\phi (m-1)\right)}}{2\left((1-\psi) - 2\phi (m-1)\right)}~.\]\\

    \textbf{Case 2:} $y_k =-1$ is very similar.\\
    First we have
    \begin{align*}
        -1 -\frac{\delta}{\alpha} = N(\btheta, \bx_k)
        \leq&  \sum\limits_{j \in J_+} u_j \sigma(\bw_j^\top \bx_k) + \sum\limits_{j \in J_-} u_j \bw_j^\top \bx_k~,
    \end{align*}

    for the first summand we get
     \begin{align*}
         \sum\limits_{j \in J_+} u_j \sigma \left( \bw_j^\top \bx_k \right) &= \sum\limits_{j \in J_+} u_j \sigma\left( -u_j \lambda_k \sigma'_{k,j} \norm{\bx_k}^2 + u_j \sum\limits_{i=1, i\neq k}^{m}  \lambda_i y_i \sigma'_{i,j} \inner{\bx_i,\bx_k}  + \inner{\bv_{\epsilon,j}, \bx_k} \right)\\
         &\leq \sum\limits_{j \in J_+} u_j \sigma \left(u_j \sum\limits_{i=1, i\neq k}^{m}  \lambda_i  \sigma'_{i,j} |\inner{\bx_i,\bx_k}|  + |\inner{\bv_{\epsilon,j}, \bx_k}|\right)\\
         &\leq \sum\limits_{j \in J_+} u_j \sigma \left(u_j \sum\limits_{i=1, i\neq k}^{m}  \lambda_i  \sigma'_{i,j} \phi + |\inner{\bv_{\epsilon,j}, \bx_k}| \right)\\
         &\leq \phi \sum\limits_{i=1, i\neq k}^{m}  \sum\limits_{j \in J_+} u_j^2  \lambda_i  \sigma'_{i,j} + \sum\limits_{j \in J_+} u_j |\inner{\bv_{\epsilon,j}, \bx_k}| \\
         &\leq  \phi (m-1) \alpha + \epsilon \sqrt{1+ \psi}
     \end{align*}

 and for the second
    \begin{align*}
         \sum\limits_{j \in J_-} u_j \bw_j^\top \bx_k =& \sum\limits_{j \in J_-} u_j  \left( -u_j \lambda_k \sigma'_{k,j} \norm{\bx_k}^2 + u_j \sum\limits_{i=1, i\neq k}^{m}  \lambda_i y_i \sigma'_{i,j} \inner{\bx_i,\bx_k} + \inner{\bv_{\epsilon,j}, \bx_k}  \right)\\
        \leq& -(1-\psi) \sum\limits_{j \in J_-} u_j^2 \lambda_k \sigma'_{k,j}  + \phi \sum\limits_{j \in J_-} \sum\limits_{i=1, i\neq k}^{m} u_j^2 \lambda_i \sigma'_{i,j} + \sum\limits_{j \in J_+} u_j |\inner{\bv_{\epsilon,j}, \bx_k}| \\
         \leq& -(1-\psi) \alpha + \phi (m-1) \alpha  + \epsilon \sqrt{1+ \psi}
     \end{align*}

combining the two results leads to the same upper bound
\[
\alpha \leq \frac{1 + 2\epsilon \sqrt{1+ \psi} + \sqrt{(1 + 2\epsilon \sqrt{1+ \psi})^2 + 4\delta\left((1-\psi) - 2\phi (m-1)\right)}}{2\left((1-\psi) - 2\phi (m-1)\right)}
\]

We plug in $\psi \leq 0.1$ and $\phi \leq \frac{\epsilon_d}{4mn}$, $\epsilon_d \leq 1$, and get 
\begin{align*}
    \alpha &\leq \frac{1 + 2\epsilon \sqrt{1+ \psi} + \sqrt{(1 + 2\epsilon \sqrt{1+ \psi})^2 + 4\delta\left((1-\psi) - 2\phi (m-1)\right)}}{2\left((1-\psi) - 2\phi (m-1)\right)}\\
    &\leq \frac{1 + 2.1\epsilon + \sqrt{(1 + 2.1\epsilon)^2 + 4\delta(0.9)}}{2(0.9 - 2\frac{\epsilon_d}{4mn} (m-1))}\\
    &\leq \frac{1 + 2.1\epsilon  + (1 + 2.1\epsilon)+ 1.9 \delta }{2(0.9 - 2\frac{\epsilon_d}{4})}\\
    &\leq \frac{2 + 4.2\epsilon + 1.9 \delta}{0.8} \leq 2.5 +5.25\epsilon + 2.4\delta \leq 10.2\\
\end{align*}
     meaning for all $i \in [m]$ we have
     \[
     \max \left\{ \sum\limits_{j \in J_+} u_j^2 \lambda_i \sigma'_{i,j} , \sum\limits_{j \in J_-} u_j^2 \lambda_i \sigma'_{i,j} \right \} \leq 2.5 +5.25\epsilon + 2.4\delta
     \]
     so
     \[
     \sum\limits_{j \in [n]} u_j^2 \lambda_i \sigma'_{i,j}  \leq 5 +10.5\epsilon + 4.8\delta \leq 20.4
     \]
     using the fact that for all $j \in [n]$ and $k \in [m]$,  $|u_j| = \frac{1}{\sqrt{n}}$ and $\sigma'_{k,j} \leq 1$ we also get that 
     \[
     \lambda_i \leq \frac{5 +10.5\epsilon + 4.8\delta}{\sum\limits_{j \in [n]} u_j^2 \sigma'_{i,j}} \leq \frac{5 +10.5\epsilon + 4.8\delta}{\frac{1}{n}} \leq n \left(5 +10.5\epsilon + 4.8\delta\right)\leq 20.4n
     \]
\end{proof}

\begin{lemma}
\label{lem:lambda_lowerbound}
    Let $N(\btheta,\bx)  = \sum\limits_{j=1}^{n} u_j \sigma(\bw_j^\top \bx)$ be a two-layer fully connected neural network, trained on $S = \{(\bx_1,y_1),...,(\bx_m,y_m)\}$, and let $0 < \epsilon_d, \epsilon, \delta \leq 1$ such that $\btheta$ is an $(\epsilon, \delta)$-approximate $KKT$ point for the margin maximization problem (\eqref{eq:margmax}) for $S$ according to \defref{def:epsdeltakkt} for $\lambda_1,...,\lambda_m$, and $S$ satisfies \assref{ass:data} for $\psi = 0,1$, and $\phi \leq \frac{\epsilon_d}{4mn}$. Assume $\forall j \in [n], u_j \sim \mathcal{U}\{-\frac{1}{\sqrt{n}}, \frac{1}{\sqrt{n}}\}$. We denote $\alpha_{\max} = \max_{i \in [m]}\left(\max \left\{ \sum\limits_{j \in J_+} u_j^2 \lambda_i \sigma'_{i,j} , \sum\limits_{j \in J_-} u_j^2 \lambda_i \sigma'_{i,j} \right \}\right)$. Then, For $i \in [m]$ we have
         \[
     \min \left\{ \sum\limits_{j \in J_+} u_j^2 \lambda_i \sigma'_{i,j} , \sum\limits_{j \in J_-} u_j^2 \lambda_i \sigma'_{i,j} \right \} \geq 0.45 - 2.32\frac{\epsilon_d}{n} - 0.96\epsilon
     \]
     and therefore also
    \[
     \sum\limits_{j =1}^{n} u_j^2 \lambda_i \sigma'_{i,j} \geq 0.9 - 4.64\frac{\epsilon_d}{n} - 1.92\epsilon
    \]
    and 
    \[
     \lambda_i \geq 0.9 - 4.64\frac{\epsilon_d}{n} - 1.92\epsilon
    \]
\end{lemma}

\begin{proof}
% \gilad{See my comments from Lemma 22, they also apply here.}
        Let $J_+ = \{j \in [n]: u_j > 0\}$ and  $J_- = \{j \in [n]: u_j < 0\}$. Denote $\alpha_+ = \min_{i \in [m]}\left( \sum\limits_{j \in J_+} u_j^2 \lambda_i \sigma'_{i,j} \right)$ and $\alpha_- = \min_{i \in [m]}\left( \sum\limits_{j \in J_-} u_j^2 \lambda_i \sigma'_{i,j} \right)$. w.l.o.g. we assume $\alpha_+ \leq \alpha_-$ (the other direction is proven similarly). We denote $\alpha = \alpha_+ = \min_{i \in [m]} \left( \sum\limits_{j \in J_+} u_j^2 \lambda_i \sigma'_{i,j}\right)$, and $k=\argmin_{i \in [m]}\left( \sum\limits_{j \in J_+} u_j^2 \lambda_i \sigma'_{i,j} \right)$.  

        Using the stationarity condition in \defref{def:epsdeltakkt} for $\btheta$, we denote $\bv_{\epsilon} = \btheta - \sum_{i=1}^m\lambda_iy_i \nabla_{\btheta} N(\btheta, \bx_i)$, and $\bv_{\epsilon,j} = \bw_j - \sum\limits_{i=1}^{m} u_j  \lambda_i y_i \sigma'_{i,j} \bx_i$, such that $\bv_{\epsilon}$ is the concatenation of all $\bv_{\epsilon,j}$ and $\norm{\bv_{\epsilon}} = \epsilon$. Using this notation we have for all $j \in [n]$ the inner product
    \begin{align*}
    \bw_j^\top \bx_k =& u_j \sum\limits_{i=1}^{m}  \lambda_i y_i \sigma'_{i,j} \inner{\bx_i,\bx_k} + \inner{\bv_{\epsilon,j}, \bx_k}\\
    =& u_j \lambda_k y_k \sigma'_{k,j} \norm{\bx_k}^2 + u_j \sum\limits_{i=1, i\neq k}^{m}  \lambda_i y_i \sigma'_{i,j} \inner{\bx_i,\bx_k} + \inner{\bv_{\epsilon,j}, \bx_k}~.
    \end{align*}

    To lower bound $\alpha$, we use the primal feasibility condition in \defref{def:epsdeltakkt} to first bound the margin, and then solve for $\alpha$. To use the $\alpha$ notation we express the margin with in terms of sums over $J_+$ and $J_-$
    \[1 + \frac{\delta}{\alpha} \geq y_k N(\btheta, \bx_k) = y_k\sum\limits_{j=1}^{n} u_j \sigma(\bw_j^\top \bx_k)
        =  y_k \left[ \sum\limits_{j \in J_+} u_j \sigma(\bw_j^\top \bx_k) +\sum\limits_{j \in J_-} u_j \sigma(\bw_j^\top \bx_k) \right]~. \]
    
    Now, to divide both sides of the inequality by $y_k$, we need to know its sign. We separate to two cases for $y_k$:

    \textbf{Case 1:} $y_k =1$\\
    We upper bound the margin
    \begin{align*}
        1 \leq N(\btheta, \bx_k)
        \leq&  \sum\limits_{j \in J_+} u_j \sigma(\bw_j^\top \bx_k) + \sum\limits_{j \in J_-} u_j \bw_j^\top \bx_k\\
    \end{align*}
Where the last inequality hold since for all $y \in \reals$, $y \leq \sigma(y)$.
    We lower bound separately the first summand, getting
    \begin{align*}
         \sum\limits_{j \in J_+} u_j \sigma \left( \bw_j^\top \bx_k \right) &= \sum\limits_{j \in J_+} u_j \sigma\left( u_j \lambda_k \sigma'_{k,j} \norm{\bx_k}^2 + u_j \sum\limits_{i=1, i\neq k}^{m}  \lambda_i y_i \sigma'_{i,j} \inner{\bx_i,\bx_k} + \inner{\bv_{\epsilon,j}, \bx_k}  \right)\\
         &\leq \sum\limits_{j \in J_+} u_j \sigma \left( u_j \lambda_k \sigma'_{k,j} \norm{\bx_k}^2 +u_j \sum\limits_{i=1, i\neq k}^{m}  \lambda_i  \sigma'_{i,j} |\inner{\bx_i,\bx_k}| + |\inner{\bv_{\epsilon,j}, \bx_k}| \right)\\
         &\leq \sum\limits_{j \in J_+} u_j \sigma \left( u_j \lambda_k \sigma'_{k,j} (1+\psi) + u_j \sum\limits_{i=1, i\neq k}^{m}  \lambda_i  \sigma'_{i,j} \phi + |\inner{\bv_{\epsilon,j}, \bx_k}|\right)\\
         &\leq  (1+\psi) \sum\limits_{j \in J_+} u_j^2 \lambda_k \sigma'_{k,j} + \phi \sum\limits_{i=1, i\neq k}^{m}  \sum\limits_{j \in J_+} u_j^2  \lambda_i  \sigma'_{i,j} +\sum\limits_{j \in J_+} u_j |\inner{\bv_{\epsilon,j}, \bx_k}|~.
     \end{align*}
 Using Cauchy–Schwarz inequality we have
     \begin{align*}
         \sum\limits_{j \in J_+} u_j |\inner{\bv_{\epsilon,j}, \bx_k}| = \frac{1}{\sqrt{n}}\sum\limits_{j \in J_+} |\inner{\bv_{\epsilon,j}, \bx_k}| \leq \frac{1}{\sqrt{n}} \norm{\bv_\epsilon} \sqrt{n} \max_{p \in [m]} \norm{x_p} \leq \epsilon \sqrt{1+ \psi}~,
     \end{align*}

     getting
     \[
     \sum\limits_{j \in J_+} u_j \sigma \left( \bw_j^\top \bx_k \right) \leq (1+\psi)\alpha + \phi (m-1) \alpha_{\max} + \epsilon \sqrt{1+ \psi}~.
     \]
For the upper bound of the second summand we have
    \begin{align*}
         \sum\limits_{j \in J_-} u_j \bw_j^\top \bx_k =& \sum\limits_{j \in J_-} u_j  \left( u_j \lambda_k \sigma'_{k,j} \norm{\bx_k}^2 + u_j \sum\limits_{i=1, i\neq k}^{m}  \lambda_i y_i \sigma'_{i,j} \inner{\bx_i,\bx_k} + \inner{\bv_{\epsilon,j}, \bx_k} \right)\\
         \leq& \sum\limits_{j \in J_-} u_j \left( u_j \lambda_k \sigma'_{k,j} (1+\psi) + u_j \sum\limits_{i=1, i\neq k}^{m}  \lambda_i \sigma'_{i,j} \phi - |\inner{\bv_{\epsilon,j}, \bx_k}|  \right)\\
         &\leq  (1+\psi)\alpha + \phi (m-1) \alpha_{\max} + \epsilon \sqrt{1+ \psi}~,
     \end{align*}

and combining the two results we have
\begin{align*}
        1 \leq N(\btheta, \bx_k) \leq&  \sum\limits_{j \in J_+} u_j \sigma(\bw_j^\top \bx_k) + \sum\limits_{j \in J_-} u_j \bw_j^\top \bx_k\\
        \leq& 2(1+\psi)\alpha + 2\phi (m-1) \alpha_{\max} + 2\epsilon \sqrt{1+ \psi}\\
    \end{align*}
    and solving for $\alpha$ we have
    \[\alpha \geq \frac{1-  2\phi (m-1) \alpha_{\max}- 2\epsilon \sqrt{1+ \psi}}{2(1+\psi)}~. \]

     \textbf{Case 2:} $y_k =-1$\\
    First we have
    \begin{align*}
        -1 \geq N(\btheta, \bx_k)
        \geq&  \sum\limits_{j \in J_+} u_j \bw_j^\top \bx_k + \sum\limits_{j \in J_-} u_j \sigma(\bw_j^\top \bx_k)\\
    \end{align*}
we get for the first summand
     \begin{align*}
         \sum\limits_{j \in J_+} u_j \left( \bw_j^\top \bx_k \right) &= \sum\limits_{j \in J_+} u_j \left( -u_j \lambda_k \sigma'_{k,j} \norm{\bx_k}^2 + u_j \sum\limits_{i=1, i\neq k}^{m}  \lambda_i y_i \sigma'_{i,j} \inner{\bx_i,\bx_k} + \inner{\bv_{\epsilon,j}, \bx_k}  \right)\\
         &\geq -(1+\psi) \sum\limits_{j \in J_+} u_j^2 \lambda_k \sigma'_{k,j} - \phi\sum\limits_{i=1, i\neq k}^{m}  \sum\limits_{j \in J_+} u_j^2 \lambda_i \sigma'_{i,j}  -  \sum\limits_{j \in J_+} u_j |\inner{\bv_{\epsilon,j}, \bx_k}| \\
         &\geq -(1+\psi) \alpha - \phi (m-1)\alpha_{\max}-  \epsilon \sqrt{1+ \psi}
     \end{align*}

And for the second summand
\begin{align*}
         \sum\limits_{j \in J_-} u_j \sigma(\bw_j^\top \bx_k) =& \sum\limits_{j \in J_-} u_j \sigma \left( -u_j \lambda_k \sigma'_{k,j} \norm{\bx_k}^2 + u_j \sum\limits_{i=1, i\neq k}^{m}  \lambda_i y_i \sigma'_{i,j} \inner{\bx_i,\bx_k} + \inner{\bv_{\epsilon,j}, \bx_k}  \right)\\
         &\geq -(1+\psi) \sum\limits_{j \in J_+} u_j^2 \lambda_k \sigma'_{k,j} - \phi\sum\limits_{i=1, i\neq k}^{m}  \sum\limits_{j \in J_+} u_j^2 \lambda_i \sigma'_{i,j} +  \sum\limits_{j \in J_-} u_j |\inner{\bv_{\epsilon,j}, \bx_k}| \\
         &\geq -(1+\psi) \alpha - \phi (m-1)\alpha_{\max}-  \epsilon \sqrt{1+ \psi}
     \end{align*}

combining the two results leads to the same lower bound \[\alpha \geq \frac{1-  2\phi (m-1) \alpha_{\max}- 2\epsilon \sqrt{1+ \psi}}{2(1+\psi)}~. \]

From \ref{lem:lambda_upperbound} we have that $\alpha_{\max} \leq 10.2$, and we plug in $\psi \leq 0.1$ and $\phi \leq \frac{\epsilon_d}{4mn}$, getting

\begin{align*}
    \alpha &\geq \frac{1-  2\phi (m-1) \alpha_{\max}- 2\epsilon \sqrt{1+ \psi}}{2(1+\psi)}\\
    &\geq \frac{1-  2\frac{\epsilon_d}{4mn} (m-1) 10.2- 2.1\epsilon}{2.2}\\
    &\geq \frac{1-  5.1\frac{\epsilon_d}{n} - 2.1\epsilon}{2.2}  \geq 0.45 - 2.32\frac{\epsilon_d}{n} - 0.96\epsilon
\end{align*}

meaning for all $i \in [m]$ we have
     \[
     \min \left\{ \sum\limits_{j \in J_+} u_j^2 \lambda_i \sigma'_{i,j} , \sum\limits_{j \in J_-} u_j^2 \lambda_i \sigma'_{i,j} \right \} \geq 0.45 - 2.32\frac{\epsilon_d}{n} - 0.96\epsilon
     \]
     so
\[\sum\limits_{j =1}^{n}u_j^2 \lambda_i \sigma'_{i,j} \geq 0.9 - 4.64\frac{\epsilon_d}{n} - 1.92\epsilon\]

% \om{now it is negative :( need to check why we need it to be > 0 and change $\epsilon$'s upper bounds accordingly. or $\phi$ denominator}

using the fact that for all $j \in [n]$ and $k \in [m]$,  $|u_j| = \frac{1}{\sqrt{n}}$ and $\sigma'_{k,j} \leq 1$ we also get that 
\begin{align*}
        \lambda_i \geq \frac{ 0.9 - 4.64\frac{\epsilon_d}{n} - 1.92\epsilon}{\frac{1}{n}\sum\limits_{j =1}^{n} \sigma'_{i,j}} \geq  0.9 - 4.64\frac{\epsilon_d}{n} - 1.92\epsilon
    \end{align*}

\end{proof}

%same sigma
\begin{lemma}
\label{lem:samesigma}
Let $N(\btheta,\bx)  = \sum\limits_{j=1}^{n} u_j \sigma(\bw_j^\top \bx)$ be a two-layer fully connected neural network, trained on $S = \{(\bx_1,y_1),...,(\bx_m,y_m)\}$, and let $0 < \epsilon_d, \epsilon, \delta \leq 1$ such that $\btheta$ is an $(\epsilon, \delta)$-approximate $KKT$ point for the margin maximization problem (\eqref{eq:margmax}) for $S$ according to \defref{def:epsdeltakkt} for $\lambda_1,...,\lambda_m$, and $S$ satisfies \assref{ass:data} for $\psi = 0,1$, and $\phi \leq \frac{\epsilon_d}{4mn}$. Given $l \in [m]$, we denote by $\hat{\btheta}$ the parameters created by performing gradient ascent on the first layer weights, for the data sample $(\bx_l,y_l) \in S$ with step size determined by $\lambda_l$ (\eqref{eq:gradascent}).
We denote by $\widetilde{\btheta}$ the weight vector such that for $j \in [n]$
\[\widetilde{\bw}_j = \hat{\bw}_j + |u_j| \lambda_l \sigma'_{l,j} \Delta_{j}~,\] 
for $\Delta_{j}= \sum\limits_{k \in [m]_{-l}} c \bx_k \sign(\inner{\bx_k, \bw_j})$ and $c=\frac{\epsilon_d}{2mn}$. 
Then, for all $r \in [m]_{-l}$ and $j \in [n]$, 
\[ \sign(\widetilde{\bw}_j^\top \bx_r) = \sign(\bw_j^\top \bx_r)\]
\end{lemma}
\begin{proof}
    Let $r \in [m]_{-l}$, and $j \in [n]$. Looking at the inner product, we denote
\[\Delta_{r,j} = \inner{\Delta_j, \bx_r}= \sum\limits_{k \in [m]_{-l}} c \inner{\bx_k,\bx_r} \sign(\inner{\bx_k, \bw_j})~,\]
and have
\begin{align*}
    \bw_j^\top \bx_r &= u_j \sum\limits_{i=1}^{m}  \lambda_i y_i \sigma'_{i,j} \inner{\bx_i,\bx_r} = \\
    &= u_j \sum\limits_{i \in [m]_{-l}}  \lambda_i y_i \sigma'_{i,j} \inner{\bx_i,\bx_r} +u_j \lambda_l y_l \sigma'_{l,j} \inner{\bx_l,\bx_r} + \inner{\bv_{\epsilon,j}, \bx_r}
\end{align*}
and 
\begin{align*}
    \widetilde{\bw}_j^\top \bx_r = u_j \sum\limits_{i \in [m]_{-l}}  \lambda_i y_i \sigma'_{i,j} \inner{\bx_i,\bx_r}+  |u_j| \lambda_l \sigma'_{l,j} \Delta_{j,r} + \inner{\bv_{\epsilon,j}, \bx_r}~,
\end{align*}
where one can see that the difference between the inner products is 
\[\bw_j^\top \bx_r - \widetilde{\bw}_j^\top \bx_r = u_j \lambda_l y_l \sigma'_{l,j} \inner{\bx_l,\bx_r} - |u_j| \lambda_l \sigma'_{l,j} \Delta_{j,r} =  \lambda_l \sigma'_{l,j} \left(u_j\inner{\bx_l,\bx_r} - |u_j|\Delta_{j,r}\right)~.\]
To show they have the same sign, it's enough to show that the difference is either negative to positive, depending on $\bw_j^\top \bx_r$ sign. If it is positive, we show the difference in negative, hence $\widetilde{\bw}_j^\top \bx_r$ is bigger and also positive, and if it's negative we show the a positive difference to conclude equal sign.

Note, if $\lambda_l =0 $ we are done, and particularly we have not change $\btheta$ by unlearning or adding our fix, meaning $\btheta = \hat{\btheta} = \widetilde{\btheta}$. In addition, if $\sigma'_{l,j}=0$ for some $j$, we haven't change the neuron $\bw_j$, and the claim follows. For the rest of the proof we assume $\lambda_l > 0 $ and $\sigma'_{l,j} = 1$, so to show the difference's sign it's enough to show the sign of $\left(u_j\inner{\bx_l,\bx_r} - |u_j|\Delta_{j,r}\right)$.

Case 1:  $\bw_j^\top \bx_r \geq 0$. We show that $\left(u_j\inner{\bx_l,\bx_r} - |u_j|\Delta_{j,r}\right) \leq 0$\\

By \lemref{lem:additionbound} 
\[
|u_j|   \Delta_{j,r} \geq |u_j|   \left( c(1-\psi) - (m-2)c \phi \right)
\]

And using \assref{ass:data} we get that $|\inner{\bx_l,\bx_r}| \leq \phi$ we have 

\begin{align*}
    u_j \inner{\bx_l,\bx_r} - |u_j|\Delta_{j,r} &\leq |u_j|   \phi - |u_j|   \left( c(1-\psi) - (m-2)c \phi \right)\\
    &\leq |u_j|   \left(\phi -  \left( c(1-\psi) - (m-2)c \phi \right)\right)~.
\end{align*}

We left to show that $\left(\phi -  c(1-\psi) + (m-2)c \phi \right) \leq 0 $ and indeed plugging in $\psi = 0.1$, $ \phi \leq \frac{\epsilon_d}{4mn}$, $ c =\frac{\epsilon_d}{2mn}$ we have 

\begin{align*}
    \phi -  c(1-\psi) + (m-2)c \phi  \leq \frac{\epsilon_d}{4mn} -  \frac{\epsilon_d}{2mn}(0.9) + (m-2)\frac{\epsilon_d}{2mn} \frac{\epsilon_d}{4mn} \leq \frac{0.25\epsilon_d - 0.45\epsilon_d + 0.125\epsilon_d^2}{mn} < 0
\end{align*}

which finishes this case.

Case 2:  $\bw_j^T\bx_r < 0$. We show that $\left(u_j\inner{\bx_l,\bx_r} - |u_j|\Delta_{j,r}\right) \geq 0$\\

By \lemref{lem:additionbound} 
\[
|u_j| \Delta_{j,r} \leq |u_j| \left( -c(1-\psi) + (m-2)c \phi \right)
\]
And using \assref{ass:data} we get that $|\inner{\bx_l,\bx_r}| \leq \phi$ we have 

\begin{align*}
    u_j \inner{\bx_l,\bx_r} - |u_j|\Delta_{j,r} &\geq -|u_j|   \phi - |u_j|   \left( -c(1-\psi) + (m-2)c \phi \right)\\
    &\geq |u_j|   \left(-\phi +  \left( c(1-\psi) - (m-2)c \phi \right)\right)~.
\end{align*}

Now, It's enough to show that $-\phi +  c(1-\psi) + (m-2)c \phi \geq 0$, which has already proven in the previous case.

\end{proof}

% diff bounds
\begin{lemma}
\label{lem:diff_margin_bound}
Let $0< \epsilon_d, \epsilon, \delta \leq 0.4$. Let $N(\bx,\btheta)  = \sum\limits_{j=1}^{n} u_j \sigma(\bw_j^\top \bx)$ be a two-layer fully connected neural network, trained on $S = \{(\bx_1,y_1),...,(\bx_m,y_m)\}$, and assume that $\btheta$ is an $(\epsilon, \delta)$-approximate $KKT$ point for the margin maximization problem (\eqref{eq:margmax}) for $S$ according to \defref{def:epsdeltakkt} for $\lambda_1,...,\lambda_m$, and $S$ satisfies \assref{ass:data} for $\psi \leq 0.1$ and $\phi \leq \frac{\epsilon_d}{4mn}$. Given $l \in [m]$, we denote by $\hat{\btheta}$ the parameters created by performing gradient ascent on the first layer weights, for the data sample $(\bx_l,y_l) \in S$ with step size $\lambda_l$ (\eqref{eq:gradascent}).
We denote by $\widetilde{\btheta}$ the weight vector such that for $j \in [n]$
\[\widetilde{\bw}_j = \hat{\bw}_j + |u_j| \lambda_l \sigma'_{l,j} \Delta_{j}~,\] 
for $\Delta_{j}= \sum\limits_{k \in [m]_{-l}} c \bx_k \sign(\inner{\bx_k, \bw_j})$ and $c=\frac{\epsilon_d}{2mn}$. Then, for all $r \in [m]_{-l}$,
\begin{align*}
-\frac{9\epsilon_d}{mn} \leq y_r \left[ N(\widetilde{\btheta}, \bx_r) - N(\btheta, \bx_r) \right] \leq \frac{9\epsilon_d}{mn}~,
\end{align*}

% \gilad{After you simplify all the terms this lemma should be much shorter. For example, the bound for the first terms is smaller than $C\cdot \epsilon_d$ for some universal constant $C > 0$.}
\end{lemma}

\begin{proof}
Let $r \in [m]_{-l}$. We look at the margins for $\bx_r$ with respect to $\btheta$ and $\widetilde{\btheta}$ and get the difference
\begin{align*}
y_r \left[ N(\widetilde{\btheta}, \bx_r) - N(\btheta, \bx_r) \right] = y_r \left[ \sum\limits_{j=1}^{n} u_j \sigma(\widetilde{\bw}_j^\top \bx_r) - \sum\limits_{j=1}^{n} u_j \sigma(\bw_j^\top \bx_r) \right]~.
\end{align*}
From \lemref{lem:samesigma} we get that for $j \in [n]$, $\sign(\widetilde{\bw}_j^\top \bx_r) = \sign(\bw_j^\top \bx_r)$. Then, if $\bw_j^\top \bx_r < 0$ we get that $\sigma(\widetilde{\bw}_j^\top \bx_r) =  \sigma(\bw_j^\top \bx_r) = 0$. Otherwise, $\bw_j^\top \bx_r \geq 0$, and we get that $\sigma(\widetilde{\bw}_j^\top \bx_r)= \widetilde{\bw}_j^\top \bx_r $ and $  \sigma(\bw_j^\top \bx_r)=\bw_j^\top \bx_r$.   We denote $J_+ = \{j \in [n]: \bw_j^\top \bx_r > 0 \text{ and } u_j > 0\}$ and $J_- = \{j \in [n]: \bw_j^\top \bx_r > 0  \text{ and } u_j < 0\}$, and get 
\begin{align*}
    &\sum\limits_{j=1}^{n} u_j \sigma(\widetilde{\bw}_j^\top \bx_r) - \sum\limits_{j=1}^{n} u_j \sigma(\bw_j^\top \bx_r) = \\
    =&  \sum\limits_{j=1}^{n} u_j \left( \sigma(\widetilde{\bw}_j^\top \bx_r) - \sigma(\bw_j^\top \bx_r) \right) \\
    =&  \sum\limits_{j \in J_+} u_j \left( \widetilde{\bw}_j^\top \bx_r - \bw_j^\top \bx_r \right) -  \sum\limits_{j \in J_-} |u_j| \left( \widetilde{\bw}_j^\top \bx_r - \bw_j^\top \bx_r \right)
\end{align*}

Following  \defref{def:epsdeltakkt}, we denote $\bv_{\epsilon} = \btheta - \sum_{i=1}^m\lambda_iy_i \nabla_{\btheta} N(\btheta, \bx_i)$ and for $j \in [n]$ we denote, 
    \[
    \bw_j = \sum\limits_{i=1}^{m} \lambda_i y_i \nabla_{\bw_j} N(\btheta, \bx_i) + \bv_{\epsilon,j} = u_j \sum\limits_{i=1}^{m}  \lambda_i y_i \sigma'_{i,j} \bx_i + \bv_{\epsilon,j}~,
    \]
    such that $\bv_{\epsilon} = (\bv_{\epsilon,1}, ..., \bv_{\epsilon,n})$ a concatenation of all $\bv_{\epsilon,j}$'s vectors. 
    Following the unlearning step in \eqref{eq:gradascent} for $(\bx_l, y_l)$, we denote 
    \[
    \hat{\bw}_j = \sum\limits_{i \in [m]_{-l}} u_j \lambda_i y_i \sigma'_{i,j} \bx_i + \bv_{\epsilon,j}~,
    \]
    and get
    \[
    \widetilde{\bw}_j = \sum\limits_{i \in [m]_{-l}} u_j \lambda_i y_i \sigma'_{i,j} \bx_i + |u_j| \lambda_l \sigma'_{l,j} \Delta_{j}+ \bv_{\epsilon,j}~.
    \]

When we look at the difference $\widetilde{\bw}_j^\top \bx_r - \bw_j^\top \bx_r$, we get that for $j \in J_+ \cup J_-$
\begin{align*}
&0 \leq \widetilde{\bw}_j^\top \bx_r - \bw_j^\top \bx_r =\\
&= \left[ u_j \sum\limits_{i \in [m]_{-l}}  \lambda_i y_i \sigma'_{i,j} \inner{\bx_i,\bx_r}+  |u_j| \lambda_l \sigma'_{l,j} \Delta_{j,r}+ \bv_{\epsilon,j} \right] -\left[ u_j \lambda_l y_l \sigma'_{l,j} \inner{\bx_l,\bx_r} + u_j \sum\limits_{i \in [m]_{-l}}  \lambda_i y_i \sigma'_{i,j} \inner{\bx_i,\bx_r}+ \bv_{\epsilon,j} \right]\\
&= |u_j| \lambda_l \sigma'_{l,j} \Delta_{j,r} - u_j \lambda_l y_l \sigma'_{l,j} \inner{\bx_l,\bx_r}~. \\
\end{align*}

We now use this equality for the margin difference, getting 
\begin{align*}
    &N(\widetilde{\btheta}, \bx_r) - N(\btheta, \bx_r) = \\
    =& \sum\limits_{j \in J_+} u_j \left( \widetilde{\bw}_j^\top \bx_r - \bw_j^\top \bx_r \right) -  \sum\limits_{j \in J_-} |u_j| \left( \widetilde{\bw}_j^\top \bx_r - \bw_j^\top \bx_r \right) \\
    =& \sum\limits_{j \in J_+} u_j \left( |u_j| \lambda_l \sigma'_{l,j} \Delta_{j,r} - u_j \lambda_l y_l \sigma'_{l,j} \inner{\bx_l,\bx_r} \right) -  \sum\limits_{j \in J_-} |u_j| \left( |u_j| \lambda_l \sigma'_{l,j} \Delta_{j,r} - u_j \lambda_l y_l \sigma'_{l,j} \inner{\bx_l,\bx_r} \right) \\
    =& \sum\limits_{j \in J_+} u_j^2 \lambda_l \sigma'_{l,j} \left(\Delta_{j,r} -y_l \inner{\bx_l,\bx_r} \right)  -  \sum\limits_{j \in J_-} u_j^2 \lambda_l \sigma'_{l,j} \left(  \Delta_{j,r} + y_l \inner{\bx_l,\bx_r} \right)~.\\
\end{align*}

We denote $\alpha_- = \sum\limits_{j \in J_-} u_j^2 \lambda_l \sigma'_{l,j}$ and by $\alpha_+ = \sum\limits_{j \in J_+} u_j^2 \lambda_l \sigma'_{l,j}$. So, we get that 
\begin{align*}
N(\widetilde{\btheta}, \bx_r) - N(\btheta, \bx_r) =& \alpha_+ \left(\Delta_{j,r} -y_l \inner{\bx_l,\bx_r} \right)  -  \alpha_- \left(  \Delta_{j,r} + y_l \inner{\bx_l,\bx_r} \right)\\
=& \alpha_+ \Delta_{j,r} - y_l\alpha_+\inner{\bx_l,\bx_r} - \alpha_-\Delta_{j,r} - y_l\alpha_-\inner{\bx_l,\bx_r}\\
=& \alpha_+ \Delta_{j,r} - \alpha_-\Delta_{j,r} -y_l \left( \alpha_+\inner{\bx_l,\bx_r} + \alpha_-\inner{\bx_l,\bx_r} \right)~.
\end{align*}

Since $\alpha_-, \alpha_+ , \Delta_{j,r}\geq 0$, for the upper bounds we get
\begin{align*}
N(\widetilde{\btheta}, \bx_r) - N(\btheta, \bx_r) =&  \alpha_+ \Delta_{j,r} - \alpha_-\Delta_{j,r} -y_l \left( \alpha_+\inner{\bx_l,\bx_r} + \alpha_-\inner{\bx_l,\bx_r} \right) \\
\leq&  \alpha_+ \Delta_{j,r} + \alpha_+\phi + \alpha_-\phi~.
\end{align*}

From \lemref{lem:lambda_upperbound} we get that $\alpha_-, \alpha_+ \leq 10.2$, from \lemref{lem:fixbounds} we get that $\Delta_{j,r} \leq c(1+\psi) + (m-2)c\phi$. Together with plugging in $\psi = 0.1$, $\phi \leq \frac{\epsilon_d}{4mn}$, $ c =\frac{\epsilon_d}{2mn}$ and $\epsilon_d \leq 1$, we get 

\begin{align*}
   N(\widetilde{\btheta}, \bx_r) - N(\btheta, \bx_r) &\alpha_+ \Delta_{j,r} + \alpha_+\phi + \alpha_-\phi  \\
\leq& \alpha_+ \left(c(1+\psi) + (m-2)c\phi + \phi \right) + \alpha_- \phi\\
    \leq& 10.2 \left(\frac{1.1\epsilon_d}{2mn} + (m-2)\frac{\epsilon_d}{2mn}\frac{\epsilon_d}{4mn} + \frac{\epsilon_d}{4mn} \right) + 10.2  \frac{\epsilon_d}{4mn}\\
    \leq& 10.2 \left(\frac{1.1\epsilon_d}{2mn} + \frac{\epsilon_d}{2n}\frac{\epsilon_d}{4mn} + \frac{\epsilon_d}{4mn} \right) + \frac{2.55\epsilon_d}{mn}\\
    \leq& \frac{\epsilon_d}{mn}\left[5.61 + \frac{0.125\epsilon_d}{n} + 0.25 + 2.55\right] \leq \frac{9\epsilon_d}{mn}~.
\end{align*}

For the lower bound of the margin we get 
\begin{align*}
N(\widetilde{\btheta}, \bx_r) - N(\btheta, \bx_r) =&  \alpha_+ \Delta_{j,r} - \alpha_-\Delta_{j,r} -y_l \left( \alpha_+\inner{\bx_l,\bx_r} + \alpha_-\inner{\bx_l,\bx_r} \right) \\
\geq& - \alpha_-\Delta_{j,r} -\alpha_+ \phi - \alpha_-\phi~,
\end{align*}

and the same calculations we did for the upper bound will yield 
\[N(\widetilde{\btheta}, \bx_r) - N(\btheta, \bx_r) \geq -\frac{9\epsilon_d}{mn}~.\]

\end{proof}

\subsection{Proof for \thmref{thm:nonlinunlearning}}
\label{app:nonlinunlearning}
\begin{proof}
    % Let $0<\epsilon_d, \epsilon , \delta \leq 0.4$. Let $N(\btheta, \bx)  = \sum\limits_{j=1}^{n} u_j \sigma(\bw_j^\top \bx)$ be a two-layer fully connected neural network, trained on $S = \{(\bx_1,y_1),...,(\bx_m,y_m)\}$, such that $\forall j \in [n], u_j \sim \mathcal{U}\{-\frac{1}{\sqrt{n}}, \frac{1}{\sqrt{n}}\}$, and $\btheta$ is an $(\epsilon, \delta)$-approximate $KKT$ point for the margin maximization problem for $S$ according to \defref{def:epsdeltakkt} for $\lambda_1,...,\lambda_m$, and $S$ satisfies \assref{ass:data} for $\psi \leq 0.1$ and $\phi \leq \frac{\epsilon_d}{4mn}$. \gilad{Until here everything is already written in the statement of the theorem, right?} 
   Note, for readability of the proof we denote $\epsilon_1$ by $\epsilon$ and $\delta_1$ by $\delta$.
    
    Using the stationarity condition in \defref{def:epsdeltakkt} for $\btheta$, we denote $\bv_{\epsilon} = \btheta - \sum_{i=1}^m\lambda_iy_i \nabla_{\btheta} N(\btheta, \bx_i)$ and for $j \in [n]$ we denote, 
    \[
    \bw_j = \sum\limits_{i=1}^{m} \lambda_i y_i \nabla_{\bw_j} N(\btheta, \bx_i) + \bv_{\epsilon,j} = u_j \sum\limits_{i=1}^{m}  \lambda_i y_i \sigma'_{i,j} \bx_i + \bv_{\epsilon,j}
    \]
    where $\bv_{\epsilon} = (\bv_{\epsilon,1}, ..., \bv_{\epsilon,n})$ the concatenation of all $\bv_{\epsilon,j}$ and $\norm{\bv_{\epsilon}} = \epsilon$.
    % \gilad{This is a concatenation of all the vectors? Is it used?}. 
    
    Let $l \in [m]$, we wish to take a negative gradient step of size $\beta$, such that 

    \[
    \beta \nabla_{\btheta} \ell(y_l N(\btheta, \bx_l)) = - \lambda_l y_l \nabla_{\btheta} N(\btheta, \bx_l)
    \]
    so we pick a step size $\beta = \frac{- \lambda_l}{\ell'(y_l N(\btheta, \bx_l))}$. We denote by $\hat{\btheta}$ the parameters created by performing gradient ascent on the first layer weights, for the data sample $(\bx_l,y_l) \in S$ with step size $\beta$ (\eqref{eq:gradascent}). As a result, for all $j \in [n]$ we have
    \begin{align*}
\hat{\bw}_j =& \bw_j - \lambda_l y_l \nabla_{\bw_j} N(\btheta, \bx_l) \\
=& \sum\limits_{i=1}^{m} \lambda_i y_i \nabla_{\bw_j} N(\btheta, \bx_i) + \bv_{\epsilon,j} - \lambda_l y_l \nabla_{\bw_j} N(\btheta, \bx_l) = \sum\limits_{i \in [m]_{-l}} u_j \lambda_i y_i \sigma'_{i,j} \bx_i + \bv_{\epsilon,j}~.
\end{align*}
  
Given $\hat{\btheta}$ and the unlearned sample index $l \in [m]$, we denote $c = \frac{\epsilon_d}{2mn}$, and for $j \in [n]$, we denote:
\[\Delta_{j} := \sum\limits_{k \in [m]_{-l}} c \bx_k \sign(\inner{\bx_k, \bw_j})~.\]
Using $\Delta_{j}$, we define a slightly modified weight vector $\widetilde{\btheta}$, such that for $j \in [n]$, 
\[\widetilde{\bw}_j = \hat{\bw}_j + |u_j| \lambda_l \sigma'_{l,j} \Delta_{j}~.\] 

% Here, we first prove that $\widetilde{\btheta}$ has the direction of a $(\frac{\hat{\epsilon}}{1-\hat{\gamma}}, \hat{\delta}\frac{1}{1-\hat{\gamma}}+\max_p\lambda_p\frac{\hat{\gamma}}{1-\hat{\gamma}})$-approximate KKT point of the margin maximization problem (\eqref{eq:margmax})  w.r.t. $S \setminus \{\bx_l,y_l\}$. We then show that $\cossim(\hat{\btheta} , \widetilde{\btheta}) \geq 1- \frac{2\Tilde{\epsilon}}{m} - \sqrt{\frac{2\Tilde{\epsilon}}{m}}$.

% \gilad{If you split the proof into steps, it would make it easier to understand it by describing at the beginning of each step what is proven in it.}

\subsubsection{Proof of $\widetilde{\theta}$ has the direction of a $(\epsilon + \frac{9\epsilon_d\epsilon}{m-9\epsilon_d} + \frac{23\epsilon_d}{\sqrt{m}}, \delta + \frac{9\epsilon_d\delta}{m-9\epsilon_d}+ \frac{22.6\epsilon_d}{m})$-approximate KKT point of the margin maximization problem (\eqref{eq:margmax})  w.r.t. $S \setminus \{\bx_l,y_l\}$}

It is enough to prove that $\widetilde{\btheta}$ is an $(\epsilon + \frac{22 \epsilon_d}{\sqrt{m}}, \delta + \frac{184\epsilon_d}{m}, \frac{9\epsilon_d}{mn})$-approximate KKT for the margin maximization problem (\eqref{eq:margmax}) w.r.t. $S \setminus (\bx_l,y_l)$ with the corresponding $\{\lambda_i\}_{i \in [m]_{-l}}$, according to \defref{def:epsdeltagammakkt}. Then, using \lemref{lem:fixthemargin}, we conclude the approximation parameters for $\frac{1}{1-\frac{9\epsilon_d}{mn}}\widetilde{\btheta}$, for the stationarity parameter, for $\epsilon_d \leq 0.01$ we have
\[
\frac{1}{1-\frac{9\epsilon_d}{mn}} \left( \epsilon + \frac{22 \epsilon_d}{\sqrt{m}}\right) \leq \left(1 + \frac{9\epsilon_d}{m-9\epsilon_d}\right) \left( \epsilon + \frac{22\epsilon_d}{\sqrt{m}}\right) \leq \epsilon + \frac{9\epsilon_d\epsilon}{m-9\epsilon_d} + \frac{23\epsilon_d}{\sqrt{m}}~,
\]
For the complementarity slackness parameter we use the upper bound for $\max_p\lambda_p$ from \ref{lem:lambda_upperbound}, and have 

\[
\frac{1}{1-\frac{9\epsilon_d}{mn}} \left(\delta + \frac{184\epsilon_d}{m} \right) + \max_p\lambda_p \frac{\frac{9\epsilon_d}{mn}}{1-\frac{9\epsilon_d}{mn}} \leq \delta + \frac{9\epsilon_d\delta}{m-9\epsilon_d}+  \frac{22.6\epsilon_d}{m}~,
\]
Finally we conclude that $\frac{1}{1-\frac{9\epsilon_d}{mn}}\widetilde{\btheta}$ is a $(\epsilon + \frac{9\epsilon_d\epsilon}{m-9\epsilon_d} + \frac{23\epsilon_d}{\sqrt{m}}, \delta + \frac{9\epsilon_d\delta}{m-9\epsilon_d}+ \frac{22.6\epsilon_d}{m})$-approximate $KKT$ for the margin maximization problem (\eqref{eq:margmax}) w.r.t. $S \setminus \{\bx_l,y_l\}$, according to \defref{def:epsdeltakkt}. We note that $\widetilde{\btheta}$ and $\frac{1}{1-\hat{\gamma}}\widetilde{\btheta}$ has the same direction, which finishes the proof. 

We start by showing $\widetilde{\btheta}$ is an $(\epsilon + \frac{22 \epsilon_d}{\sqrt{m}}, \delta + \frac{184\epsilon_d}{m}, \frac{9\epsilon_d}{mn})$-approximate KKT.

\subsubsection*{(1) Dual Feasibility: For all $r \in [m]_{-l}$, $\lambda_r \geq 0$.} Directly from dual feasibility for $\btheta$ (\defref{def:epsdeltakkt}).

\subsubsection*{(2) Stationarity: $\norm{\widetilde{\btheta} - \sum\limits_{i \in [m]_{-l}} \lambda_i y_i \nabla_{\btheta} N(\widetilde{\btheta}, \bx_i)} \leq \epsilon + \frac{22 \epsilon_d}{\sqrt{m}}$.} 

From stationarity for $\btheta$ (\defref{def:epsdeltakkt}) we get that $\btheta = \sum_{i=1}^m\lambda_iy_i \nabla_{\btheta} N(\btheta, \bx_i) + \bv_{\epsilon}$. By the difinition of $\hat{\btheta}$ we get that $\hat{\btheta} = \sum_{i \in [m]_{-l}}\lambda_iy_i \nabla_{\btheta} N(\btheta, \bx_i) + \bv_{\epsilon}$. For readability, we first denote $\bu =  (|u_1| \lambda_l \sigma'_{l,1} \Delta_1, ..,|u_n| \lambda_l \sigma'_{l,n} \Delta_n)$, such that $\bu \in \reals^{m\times n}$, and note that one can write $ \widetilde{\btheta} = \hat{\btheta} + \bu$. Thus,

\begin{align*}
    \norm{\widetilde{\btheta} - \sum\limits_{i \in [m]_{-l}} \lambda_i y_i \nabla_{\btheta} N(\widetilde{\btheta}, \bx_i)} = \norm{\sum_{i \in [m]_{-l}}\lambda_iy_i \nabla_{\btheta} N(\btheta, \bx_i) + \bv_{\epsilon} + \bu - \sum\limits_{i \in [m]_{-l}} \lambda_i y_i \nabla_{\btheta} N(\widetilde{\btheta}, \bx_i)}
\end{align*}

In \lemref{lem:samesigma} we showed that for $j \in [n], i \in [m]$, $\onefunc_{\{\widetilde{\bw}_j^T \bx_j \geq 0\}} = \onefunc_{\{\bw_j^T \bx_j \geq 0\}}$. Then, for $j \in [n]$ we have
\begin{align*}
    \nabla_{\bw_j} N(\widetilde{\btheta}, \bx_i) = u_j \onefunc_{\{\widetilde{\bw}_j^T \bx_j \geq 0\}} \bx_i  = 
     u_j \onefunc_{\{\bw_j^T \bx_j \geq 0\}} \bx_i = \nabla_{\bw_j} N(\btheta, \bx_i)~,
\end{align*}

which leads to 
\begin{align*}
    &\norm{\widetilde{\btheta} - \sum\limits_{i \in [m]_{-l}} \lambda_i y_i \nabla_{\btheta} N(\widetilde{\btheta}, \bx_i)} =\\
    =&\norm{\sum_{i \in [m]_{-l}}\lambda_iy_i \nabla_{\btheta} N(\btheta, \bx_i) + \bv_{\epsilon} + \bu - \sum\limits_{i \in [m]_{-l}} \lambda_i y_i \nabla_{\btheta} N(\widetilde{\btheta}, \bx_i)}\\
    =& \norm{ \bv_{\epsilon} + \bu} \leq \norm{\bv_\epsilon} + \norm{\bu}~.
\end{align*}

% By Cauchy–Schwarz inequality we get that 
% \[ \norm{ \bv_{\epsilon} + \bu}^2 \leq \norm{\bv_{\epsilon}}^2 + \norm{\bu}^2 + 2\norm{\bv_{\epsilon}}\norm{\bu}~. \]

Using the upper bound from \lemref{lem:fixupperbounds}, for $\norm{\bu}$ we have
\begin{align*}
\norm{\bu} = \norm{(|u_1| \lambda_l \sigma'_{l,1} \Delta_1, ..,|u_n| \lambda_l \sigma'_{l,n} \Delta_n)} =& \sqrt{\sum\limits_{j=1}^{n} \norm{u_j \lambda_l \sigma_{l,j}' \Delta_j} ^2} \leq \\
\leq& \sqrt{n} \max_{j \in [n]}|u_j| \lambda_l \sigma'_{l,j} \norm{\Delta_j} \\
\leq&\sqrt{n}  \frac{22 \epsilon_d}{\sqrt{mn}} \leq \frac{22 \epsilon_d}{\sqrt{m}}~,
\end{align*}

and plugging it in we have

\begin{align*}
    &\norm{\widetilde{\btheta} - \sum\limits_{i \in [m]_{-l}} \lambda_i y_i \nabla_{\btheta} N(\widetilde{\btheta}, \bx_i)} \leq \norm{\bv_\epsilon} + \norm{\bu} \leq \epsilon + \frac{22 \epsilon_d}{\sqrt{m}} ~.
\end{align*}

% \begin{align*}
%     \norm{\widetilde{\btheta} - \sum\limits_{i \in [m]_{-l}} \lambda_i y_i \nabla_{\btheta} N(\widetilde{\btheta}, \bx_i)}  =& \norm{ \bv_{\epsilon} + \bu} \\
%     \leq& \sqrt{\norm{\bv_{\epsilon}}^2 + \norm{\bu}^2 + 2\norm{\bv_{\epsilon}}\norm{\bu}} \\
%     \leq& \epsilon + \frac{7 \epsilon_d}{\sqrt{m}} + \sqrt{2 \epsilon \frac{7 \epsilon_d}{\sqrt{m}}} = \epsilon + \frac{22 \epsilon_d}{\sqrt{m}} + \frac{\sqrt{44\epsilon_d\epsilon}}{\sqrt[4]{m}}~,
% \end{align*}

as desired.

From \lemref{lem:diff_margin_bound} we get that $-\frac{9\epsilon_d}{mn} \leq y_rN(\widetilde{\btheta}, \bx_r) - y_r N(\btheta, \bx_r) \leq \frac{9\epsilon_d}{mn}$. Using it we prove the next conditions.

\subsubsection*{(3) Complementarity Slackness: For all $r \in [m]_{-l}$, $\lambda_r \left( y_r N(\widetilde{\btheta}, \bx_r)-1 \right) \leq \delta + \frac{184\epsilon_d}{m}$.} 
Let $r \in [m]_{-l}$. If $\lambda_r = 0$ we are done.
    Otherwise, from complementarity slackness condition for $\btheta$ we get that $\lambda_r \left( y_r N(\btheta, \bx_r)-1 \right) \leq \delta$. We use the fact that $y_rN(\widetilde{\btheta}, \bx_r) - \frac{9\epsilon_d}{mn} \leq y_r N(\btheta, \bx_r)$ to get that
    \begin{align*}
    &\delta \geq \lambda_r \left( y_r N(\btheta, \bx_r)-1 \right) = \lambda_r \left( y_r N(\widetilde{\btheta}, \bx_r)-\frac{9\epsilon_d}{mn}-1 \right) = \lambda_r \left( y_r N(\widetilde{\btheta}, \bx_r)-1 \right) - \lambda_r\frac{9\epsilon_d}{mn} \\
    &\geq \lambda_r \left( y_r N(\widetilde{\btheta}, \bx_r)-1 \right) - \max_p\lambda_p\frac{9\epsilon_d}{mn}
    \end{align*}
    and conclude that
    \[ \lambda_r \left( y_r N(\widetilde{\btheta}, \bx_r)-1 \right) \leq \delta + \max_p\lambda_p\frac{9\epsilon_d}{mn}~.\]
    From \lemref{lem:lambda_upperbound} we have an upper bound $\max_p\lambda_p \leq  20.4n$, so we get that 
    \begin{align*}
        &\lambda_r \left( y_r N(\widetilde{\btheta}, \bx_r)-1 \right) \leq \delta + \max_p\lambda_p\frac{9\epsilon_d}{mn} \leq \delta + 20.4n \frac{9\epsilon_d}{nm} \leq \delta + \frac{184\epsilon_d}{m}~.
    \end{align*}

\subsubsection*{(4) Primal Feasibility: For all $r \in [m]_{-l}$, $y_i N(\bx_i, \widetilde{\btheta}) \geq 1-\frac{9\epsilon_d}{mn}$.} 
Let $r \in [m]_{-l}$. From primal feasibility for $\btheta$ (\defref{def:epsdeltakkt}) we get that $y_r N(\btheta, \bx_r) \geq 1$, and from \lemref{lem:diff_margin_bound} we have that \[y_rN(\widetilde{\btheta}, \bx_r) - y_r N(\btheta, \bx_r) \geq - \frac{9\epsilon_d}{mn}\] 
which concludes the proof.

\subsubsection{Proof of $\cossim(\hat{\theta} , \widetilde{\theta}) \geq 1- \frac{82\epsilon_d}{m}$}

We begin with looking at the inner product $\inner{\hat{\btheta} , \widetilde{\btheta}}$. For readability, we first denote $\bu =  (|u_1| \lambda_l \sigma'_{l,1} \Delta_1, ..,|u_n| \lambda_l \sigma'_{l,n} \Delta_n)$, such that $\bu \in \reals^{m\times n}$, and note that one can write $ \widetilde{\btheta} = \hat{\btheta} + \bu$ and 
\[
\inner{\hat{\btheta} , \widetilde{\btheta}}  =  \inner{\hat{\btheta}, \hat{\btheta} + \bu} = \norm{\hat{\btheta}}^2 + \inner{\hat{\btheta}, \bu} \geq \norm{\hat{\btheta}}^2 - |\inner{\hat{\btheta}, \bu}| \geq \norm{\hat{\btheta}}^2 - \norm{\hat{\btheta}}\norm{\bu}~,
\]
where the last transition is due to Cauchy–Schwarz inequality. We now look at the weights vectors norm and get
\[\norm{\widetilde{\btheta}} = \norm{\hat{\btheta} + \bu} \leq \norm{\hat{\btheta}} + \norm{\bu}\]

% again using Cauchy–Schwarz inequality get that
% \[
% \norm{\widetilde{\btheta}}^2 = \inner{\widetilde{\btheta},\widetilde{\btheta}} = \inner{\hat{\btheta} + \bu,\hat{\btheta} + \bu} \leq \norm{\hat{\btheta}}^2 + 2|\inner{\hat{\btheta}, \bu}| + \norm{\bu}^2 \leq \norm{\hat{\btheta}}^2 + 2\norm{\hat{\btheta}}\norm{\bu} + \norm{\bu}^2~,
% \]

which leads to 

\begin{align*}
     \norm{\hat{\btheta}} \norm{\widetilde{\btheta}} = \norm{\hat{\btheta}} \left( \norm{\hat{\btheta}} + \norm{\bu} \right) =  \norm{\hat{\btheta}}^2 +  \norm{\hat{\btheta}}\norm{\bu} 
\end{align*}

% \begin{align*}
%      \norm{\hat{\btheta}} \norm{\widetilde{\btheta}} &= \sqrt{\norm{\hat{\btheta}}^2 \norm{\widetilde{\btheta}}^2 } = \\
%     &=\sqrt{\norm{\hat{\btheta}}^2 \left(\norm{\hat{\btheta}}^2 + 2\norm{\hat{\btheta}}\norm{\bu} + \norm{\bu}^2 \right)} \\
%     &=\norm{\hat{\btheta}}^2 \sqrt{1 + \frac{2\norm{\bu}}{\norm{\hat{\btheta}}}+\frac{\norm{\bu}^2}{\norm{\hat{\btheta}}^2}}\\
%     &\leq \norm{\hat{\btheta}}^2 \left( 1 + \sqrt{\frac{2\norm{\bu}}{\norm{\hat{\btheta}}}+\frac{\norm{\bu}^2}{\norm{\hat{\btheta}}^2}}\right)\\
%     &= \norm{\hat{\btheta}}^2 + \norm{\hat{\btheta}}^2 \sqrt{\frac{2\norm{\bu}}{\norm{\hat{\btheta}}}+\frac{\norm{\bu}^2}{\norm{\hat{\btheta}}^2}}
% \end{align*}

% where the last inequality hold since $1+\sqrt{x} \geq \sqrt{1+x}$ for all $x>0$. 
We are now ready to lower bound the cosine similarity, having

\begin{align*}
    \cossim(\hat{\btheta} , \widetilde{\btheta}) =& \frac{\inner{\hat{\btheta} , \widetilde{\btheta}}}{\norm{\hat{\btheta}}\norm{\widetilde{\btheta}}}\\
    \geq& \frac{\norm{\hat{\btheta}}^2 - \norm{\hat{\btheta}}\norm{\bu}}{\norm{\hat{\btheta}}^2 + \norm{\hat{\btheta}}\norm{\bu}}\\
    \geq& 1- \frac{2\norm{\hat{\btheta}}\norm{\bu}}{\norm{\hat{\btheta}}^2 + \norm{\hat{\btheta}}\norm{\bu}}\\
    \geq& 1- \frac{2 \norm{\bu}}{\norm{\hat{\btheta}}} ~.
\end{align*}

% \begin{align*}
%     \cossim(\hat{\btheta} , \widetilde{\btheta}) =& \frac{\inner{\hat{\btheta} , \widetilde{\btheta}}}{\norm{\hat{\btheta}}\norm{\widetilde{\btheta}}}\\
%     \geq& \frac{\norm{\hat{\btheta}}^2 - \norm{\hat{\btheta}}\norm{\bu}}{\norm{\hat{\btheta}}^2 + \norm{\hat{\btheta}}^2 \sqrt{\frac{2\norm{\bu}}{\norm{\hat{\btheta}}}+\frac{\norm{\bu}^2}{\norm{\hat{\btheta}}^2}}}\\
%     \geq& 1- \frac{\norm{\hat{\btheta}}\norm{\bu} - \norm{\hat{\btheta}}^2 \sqrt{\frac{2\norm{\bu}}{\norm{\hat{\btheta}}}+\frac{\norm{\bu}^2}{\norm{\hat{\btheta}}^2}}}{\norm{\hat{\btheta}}^2 + \norm{\hat{\btheta}}^2 \sqrt{\frac{2\norm{\bu}}{\norm{\hat{\btheta}}}+\frac{\norm{\bu}^2}{\norm{\hat{\btheta}}^2}}}\\
%     \geq& 1- \frac{\norm{\bu}}{\norm{\hat{\btheta}}} - \sqrt{\frac{2\norm{\bu}}{\norm{\hat{\btheta}}}+\frac{\norm{\bu}^2}{\norm{\hat{\btheta}}^2}}~.
% \end{align*}

% \begin{align*}
% |\inner{\hat{\btheta}, \bu}|
%     &=|\sum\limits_{j \in [n]}\inner{\hat{\bw}_j, |u_j| \lambda_l \sigma'_{l,j} \Delta_{j}}| \\
%     &=|\sum\limits_{j \in [n]}\inner{\sum\limits_{i \in [m]_{-l}} u_j \lambda_i y_i \sigma'_{i,j} \bx_i +\bv_{\epsilon,j},|u_j| \lambda_l \sigma'_{l,j} \Delta_{j}}|\\
%     &\geq|\sum\limits_{j \in [n]}\inner{\sum\limits_{i \in [m]_{-l}} u_j \lambda_i y_i \sigma'_{i,j} \bx_i,|u_j| \lambda_l \sigma'_{l,j} \Delta_{j}}| - |\inner{\bv_\epsilon, \bu}|\\
% \end{align*}

% \gilad{Try to simplify all the calculations below. I think everything should depend only on $\epsilon_d$ and perhaps $m$ in the denominator.}
To finish the proof, we upper bound $\frac{\norm{\bu}}{\norm{\hat{\btheta}}}$. We note that we can upper bound the norm of $\bu$ using the upper bound from \lemref{lem:fixupperbounds}:
\begin{align*}
\norm{\bu} = \norm{(|u_1| \lambda_l \sigma'_{l,1} \Delta_1, ..,|u_n| \lambda_l \sigma'_{l,n} \Delta_n)} =& \sqrt{\sum\limits_{j=1}^{n} \norm{u_j \lambda_l \sigma_{l,j}' \Delta_j} ^2} \leq \\
\leq& \sqrt{n} \max_{j \in [n]}|u_j| \lambda_l \sigma'_{l,j} \norm{\Delta_j} \\
\leq&\sqrt{n}  \frac{22 \epsilon_d}{\sqrt{mn}} \leq \frac{22 \epsilon_d}{\sqrt{m}}~,
\end{align*}

We now show a lower bound for $\norm{\hat{\btheta}}$, using that for all $j \in [n]$, $|u_j|  = \frac{1}{\sqrt{n}}$, and for \assref{ass:data}, for all $i,k \in [m]$ $\norm{\bx_i}^2 \geq (1-\psi)$, and $|\inner{\bx_i,\bx_k}|\leq \phi$. We have
\begin{align*}
    \norm{\hat{\btheta}}^2 &= \sum\limits_{j \in [n]} \norm{\hat{\bw}_j}^2 \\
    &= \sum\limits_{j \in [n]}\norm{\sum\limits_{i \in [m]_{-l}} u_j \lambda_i y_i \sigma'_{i,j} \bx_i}^2 \\
    &= \sum\limits_{j \in [n]}\inner{\sum\limits_{i \in [m]_{-l}} u_j \lambda_i y_i \sigma'_{i,j} \bx_i, \sum\limits_{i \in [m]_{-l}} u_j \lambda_i y_i \sigma'_{i,j} \bx_i}\\
    &\geq \sum\limits_{j \in [n]} \left( \sum\limits_{i \in [m]_{-l}} u_j^2 \lambda_i^2 \sigma'_{i,j} \norm{\bx_i}^2 - \sum\limits_{i \in [m]_{-l}} \sum\limits_{k \neq i \in [m]_{-l}} u_j^2 \lambda_i \lambda_k \sigma'_{i,j} \sigma'_{k,j} \inner{\bx_i, \bx_k} \right)\\
    &\geq \frac{1}{n} \sum\limits_{j \in [n]} \left((1-\psi)\sum\limits_{i \in [m]_{-l}}  \lambda_i^2 \sigma'_{i,j}  - \phi\sum\limits_{i \in [m]_{-l}} \sum\limits_{k \neq i \in [m]_{-l}}  \lambda_i \lambda_k \sigma'_{i,j} \sigma'_{k,j}  \right)\\
    &\geq \frac{1}{n}  \left((1-\psi)\sum\limits_{i \in [m]_{-l}}  \lambda_i^2 \sum\limits_{j \in [n]}\sigma'_{i,j}  - \phi\sum\limits_{i \in [m]_{-l}} \sum\limits_{k \neq i \in [m]_{-l}}  \lambda_i \lambda_k \sum\limits_{j \in [n]}\sigma'_{i,j}  \right)
        \end{align*}
        
We note that using \lemref{lem:lambda_lowerbound} and \lemref{lem:lambda_upperbound}, for all $i$, we have 
\[
\left(0.9 - 4.64\frac{\epsilon_d}{n} - 1.92\epsilon\right) \leq  \sum\limits_{j =1}^{n} u_j^2 \lambda_i \sigma'_{i,j} \leq 20.4
\]
hence since $|u_j|  = \frac{1}{\sqrt{n}}$
\[\left(0.9 - 4.64\frac{\epsilon_d}{n} - 1.92\epsilon\right)n \leq   \lambda_i \sum\limits_{j =1}^{n}  \sigma'_{i,j} \leq 20.4n
\]
Using these bounds we have

\begin{align*}
    \norm{\hat{\btheta}}^2 &\geq \frac{1}{n} \left((1-\psi)\sum\limits_{i \in [m]_{-l}}  \lambda_i \left(0.9 - 4.64\frac{\epsilon_d}{n} - 1.92\epsilon\right)n  - \phi\sum\limits_{i \in [m]_{-l}} \sum\limits_{k \neq i \in [m]_{-l}}  \lambda_i20.4n  \right)\\
    &\geq \left((1-\psi)\left(0.9 - 4.64\frac{\epsilon_d}{n} - 1.92\epsilon\right)\sum\limits_{i \in [m]_{-l}}  \lambda_i   - 20.4\phi\sum\limits_{i \in [m]_{-l}} \sum\limits_{k \neq i \in [m]_{-l}}  \lambda_i  \right)\\
    &\geq \sum\limits_{i \in [m]_{-l}}  \lambda_i \left[(1-\psi)\left(0.9 - 4.64\frac{\epsilon_d}{n} - 1.92\epsilon\right)   - 20.4\phi (m-2)  \right]
        \end{align*}

Plugging in $\psi \leq 0.1$, $\phi \leq \frac{\epsilon_d}{4mn}$, $\epsilon < 1$ and $\epsilon_d \leq 0.01$ we have

\begin{align*}
&\geq \sum\limits_{i \in [m]_{-l}}  \lambda_i \left[0.9\left(0.9 - 4.64\frac{\epsilon_d}{n} - 1.92\epsilon\right)   - 20.4\phi (m-2)  \right]\\
    &\geq (m-1) \left(0.9 - 4.64\frac{\epsilon_d}{n} - 1.92\epsilon\right) \left[0.9\left(0.9 - 4.64\frac{\epsilon_d}{n} - 1.92\epsilon\right)   - \frac{20.4\epsilon_d}{4n}  \right]\\
    &\geq (m-1) \left(0.9 (0.9 - 4.64\frac{\epsilon_d}{n} - 1.92\epsilon)^2 - \frac{5.1\epsilon_d}{n}\left(0.9 - 4.64\frac{\epsilon_d}{n} - 1.92\epsilon\right) \right)\\
    \geq& (m-1) \left(0.72 + 19 \frac{\epsilon_d^2}{n^2} + 3\epsilon + 8\frac{\epsilon_d \epsilon}{n} - 7.6\frac{\epsilon_d}{n} - 3.2\epsilon - \frac{4.6 \epsilon_d}{n}\right)\\
    \geq& (m-1) \left(0.72 - 12.2\frac{\epsilon_d}{n} - 0.2\epsilon \right)\\
    \geq& 0.3(m-1)
\end{align*}

and of course 
\[
\norm{\hat{\btheta}} \geq \sqrt{0.3(m-1)}~.
\]

We can know join the upper bound for $\norm{\bu}$ and lower bound of $\norm{\hat{\btheta}}$ getting 
\begin{align*}
    \frac{\norm{\bu}}{\norm{\hat{\btheta}}} \leq& \frac{\frac{22 \epsilon_d}{\sqrt{m}}}{\sqrt{0.3(m-1)}}\\
    \leq& \frac{41\epsilon_d}{m}
\end{align*}

and finally,

\begin{align*}
    \cossim(\hat{\btheta} , \widetilde{\btheta}) \geq& 1- \frac{2\norm{\bu}}{\norm{\hat{\btheta}}} \\
    \geq& 1- \frac{82\epsilon_d}{m} ~,
\end{align*}

as desired.
\end{proof}

\subsection{Proof for forgetting subset of points using $\mathcal{A}_{\text{k-GA}}$ -- two layer networks}
%(\secref{sec:kset})}
\label{app:nonlinunlearning_subset}
We formalize and prove the statement for unlearning a subset of data points. 
Recall that the term \emph{successful unlearning} here is the natural extension of Definition~\ref{def:suc_unlearning} to unlearning a subset, rather than a single point.

\begin{theorem}\label{theorem:nonlinsubset}
    In the same settings as \thmref{thm:nonlinunlearning}, let $S_{\text{forget}} \subseteq S$ a subset of size $k$.
    Then, the extended algorithm $\mathcal{A}_{\text{k-GA}}$, with appropriate coefficients $\{\beta_r\}$, 
    is a \textbf{$(\epsilon, \delta, \tau)$-successful} unlearning algorithm w.r.t. $\btheta$ and $S$, where $\epsilon=\epsilon_1 + \frac{9\epsilon_d\epsilon_1}{\frac{m}{k}-9\epsilon_d} + \frac{23k\epsilon_d}{\sqrt{m}}$, $\delta=\delta_1 + \frac{9\epsilon_d\delta_1}{\frac{m}{k}-9\epsilon_d}+ \frac{22.6k\epsilon_d}{m}$ and $\tau = \frac{82k\epsilon_d}{m-k}$.
    
    % we unlearn $S_{\text{forget}} \subset S$ a subset of size $k$, using one step consists of the $k$ gradients with an appropriate sizes. We denote by $\hat{\btheta}$ the parameters created by performing gradient ascent on the first layer weights. Then, there exists $\widetilde{\btheta}$, such that $\widetilde{\btheta}$ has the direction of a $(\epsilon + \frac{9\epsilon_d\epsilon}{\frac{m}{k}-9\epsilon_d} + \frac{23k\epsilon_d}{\sqrt{m}}, \delta + \frac{9\epsilon_d\delta}{\frac{m}{k}-9\epsilon_d}+ \frac{22.6k\epsilon_d}{m})$-approximate KKT point for the margin maximization problem (\eqref{eq:margmax}) w.r.t. $S \setminus (\bx_l,y_l)$ and $\cossim(\hat{\btheta} , \widetilde{\btheta}) \geq 1- \frac{82k\epsilon_d}{m-k}$. 
\end{theorem}

\begin{proof}
Let a forget set $S_f \subset S$ such that $|S_f| = k$. We denote $I_f = \{i : (\bx_i,y_i) \in S_f\}$. We denote $S_r = S \setminus S_f$ and $I_r = \{i : (\bx_i,y_i) \in S_r\}$. This proof widely relies the proof in \ref{app:nonlinunlearning}.

Using the stationarity condition in \defref{def:epsdeltakkt} for $\btheta$, we denote $\bv_{\epsilon} = \btheta - \sum_{i=1}^m\lambda_iy_i \nabla_{\btheta} N(\btheta, \bx_i)$ and for $j \in [n]$ we denote, 
\[
\bw_j = \sum\limits_{i=1}^{m} \lambda_i y_i \nabla_{\bw_j} N(\btheta, \bx_i) + \bv_{\epsilon,j} = u_j \sum\limits_{i=1}^{m}  \lambda_i y_i \sigma'_{i,j} \bx_i + \bv_{\epsilon,j}
\]
where $\bv_{\epsilon} = (\bv_{\epsilon,1}, ..., \bv_{\epsilon,n})$ the concatenation of all $\bv_{\epsilon,j}$ and $\norm{\bv_{\epsilon}} = \epsilon$.

According to the algorithm $\mathcal{A}_{\text{k-GA}}$, we take a step consists of the sum of $k$ gradients w.r.t. data points in $S_f$ with the following sizes- for any $(\bx_l,y_l) \in S_f$, we take a step size $\beta = \frac{- \lambda_l}{\ell'(y_l N(\btheta, \bx_l))}$. As a result, for all $j \in [n]$ we have
    \begin{align*}
\hat{\bw}_j =& \bw_j - \lambda_l y_l \nabla_{\bw_j} N(\btheta, \bx_l) \\
=& \sum\limits_{i=1}^{m} \lambda_i y_i \nabla_{\bw_j} N(\btheta, \bx_i) + \bv_{\epsilon,j} - \sum\limits_{l \in I_f} \lambda_l y_l \nabla_{\bw_j} N(\btheta, \bx_l) = \sum\limits_{i \in I_r} u_j \lambda_i y_i \sigma'_{i,j} \bx_i + \bv_{\epsilon,j}~.
\end{align*}

Given $\hat{\btheta}$ and the unlearned sample indices $l \in I_f$, we denote $c = \frac{\epsilon_d}{2mn}$, and for $j \in [n]$, we denote:
\[\Delta_{j} := \sum\limits_{k \in S_r} c \bx_k \sign(\inner{\bx_k, \bw_j})~.\]
Using $\Delta_{j}$, we define a slightly modified weight vector $\widetilde{\btheta}$, such that for $j \in [n]$, 
\[\widetilde{\bw}_j = \hat{\bw}_j + \sum\limits_{l \in I_f} |u_j| \lambda_l \sigma'_{l,j} \Delta_{j}~.\]

The first main challenge of this proof is \lemref{lem:samesigma}, that is proven for a single point unlearning. However, browsing through the proof one can see that its main observation is about the difference between the inner product of some training sample $\bx_r$ in either the original or the fixed unlearn weight voters. Looking at the difference our case -
\begin{align*}
    \inner{\bw_j , \bx_r} - \inner{\widetilde{\bw}_j, \bx_r} =& \sum\limits_{l \in I_f} u_j \lambda_l y_l \sigma'_{l,j} \inner{\bx_l,\bx_r} - \sum\limits_{l \in I_f} |u_j| \lambda_l \sigma'_{l,j} \Delta_{j} =\\
    &\sum\limits_{l \in I_f} \left( u_j \lambda_l y_l \sigma'_{l,j} \inner{\bx_l,\bx_r} -|u_j| \lambda_l \sigma'_{l,j} \Delta_{j}\right)~,
\end{align*}

one can see that for any $l \in I_f$:

\[u_j \lambda_l y_l \sigma'_{l,j} \inner{\bx_l,\bx_r} - |u_j| \lambda_l \sigma'_{l,j} \Delta_{j,r} =  \lambda_l \sigma'_{l,j} \left(u_j\inner{\bx_l,\bx_r} - |u_j|\Delta_{j,r}\right)~,\]

which is the exact same modification that in \lemref{lem:samesigma} is proven to not effect the sign. Thus, using \lemref{lem:samesigma} for any $l \in S_f$ will conclude in 
\[ \sign(\widetilde{\bw}_j^\top \bx_r) = \sign(\bw_j^\top \bx_r)~.\]

The next important issue we need to address to use the similar proof for forgetting multiple points is the norm of the fix. If we denote $\bu =  (\sum\limits_{l \in I_f}|u_1| \lambda_l \sigma'_{l,1} \Delta_1, ..,\sum\limits_{l \in I_f}|u_n| \lambda_l \sigma'_{l,n} \Delta_n)$ we get a factor $k$ in the upper bound for $\norm{\bu}$, using \lemref{lem:fixupperbounds}:

\begin{align*}
   \norm{\sum\limits_{l \in I_f}|u_j| \lambda_l \sigma'_{l,j} \Delta_{j}} = \sum\limits_{l \in I_f} |u_j| \lambda_l \sigma_{l,j}' \norm{\Delta_j} \leq& k\frac{1}{\sqrt{n}}20.4n \frac{\epsilon_d \sqrt{1.1+\frac{\epsilon_d}{n}}}{2\sqrt{m}n} \\
    \leq& k\frac{1}{\sqrt{n}}20.4 \frac{\epsilon_d \sqrt{1.1+\frac{\epsilon_d}{n}}}{2\sqrt{m}} \\
    \leq& \frac{k\epsilon_d \left(20.4 + \frac{1}{2}\sqrt{1.1+\frac{\epsilon_d}{n}} \right)}{\sqrt{nm}} \leq \frac{22 k\epsilon_d}{\sqrt{mn}} ~,
    % < \frac{1.6 \epsilon_1}{\sqrt{n}}
\end{align*}

Lastly, we add a factor $k$ for the margin difference, by straightforward accumulating the margin difference for each $l \in I_f$, getting 
    \begin{align*}
-\frac{9k\epsilon_d}{mn} \leq y_r \left[ N(\widetilde{\btheta}, \bx_r) - N(\btheta, \bx_r) \right] \leq \frac{9k\epsilon_d}{mn}~,
\end{align*}

We now ready to prove the multi-point version.

\subsubsection*{Proof of $\widetilde{\theta}$ has the direction of a $(\epsilon + \frac{9\epsilon_d\epsilon}{\frac{m}{k}-9\epsilon_d} + \frac{23k\epsilon_d}{\sqrt{m}}, \delta + \frac{9\epsilon_d\delta}{\frac{m}{k}-9\epsilon_d}+ \frac{22.6k\epsilon_d}{m})$-approximate KKT point of the margin maximization problem (\eqref{eq:margmax})  w.r.t. $S \setminus \{\bx_l,y_l\}$:}

\subsubsection*{(1) Dual Feasibility: For all $r \in [m]_{-l}$, $\lambda_r \geq 0$.} Same. Directly from dual feasibility for $\btheta$ (\defref{def:epsdeltakkt}).

\subsubsection*{(2) Stationarity: $\norm{\widetilde{\btheta} - \sum\limits_{i \in I_r} \lambda_i y_i \nabla_{\btheta} N(\widetilde{\btheta}, \bx_i)} \leq \epsilon + \frac{22k \epsilon_d}{\sqrt{m}}$.} 

We showed that for $j \in [n], i \in [m]$, $\onefunc_{\{\widetilde{\bw}_j^T \bx_j \geq 0\}} = \onefunc_{\{\bw_j^T \bx_j \geq 0\}}$, thus similarly having 
\begin{align*}
    &\norm{\widetilde{\btheta} - \sum\limits_{i \in I_r} \lambda_i y_i \nabla_{\btheta} N(\widetilde{\btheta}, \bx_i)} =\\
    =&\norm{\sum_{i \in I_r}\lambda_iy_i \nabla_{\btheta} N(\btheta, \bx_i) + \bv_{\epsilon} + \bu - \sum\limits_{i \in I_r} \lambda_i y_i \nabla_{\btheta} N(\widetilde{\btheta}, \bx_i)}\\
    =& \norm{ \bv_{\epsilon} + \bu} \leq \norm{\bv_\epsilon} + \norm{\bu}~.
\end{align*}

Using the upper bound from we showed, we have
\begin{align*}
\norm{\bu} = \norm{(\sum\limits_{l \in I_f}|u_1| \lambda_l \sigma'_{l,1} \Delta_1, ..,\sum\limits_{l \in I_f}|u_n| \lambda_l \sigma'_{l,n} \Delta_n)} =& \sqrt{\sum\limits_{j=1}^{n} \norm{ \sum\limits_{l \in I_f} u_j \lambda_l \sigma_{l,j}' \Delta_j} ^2} \leq \\
\leq& \sqrt{n} \max_{j \in [n]} \sum\limits_{l \in I_f} |u_j| \lambda_l \sigma'_{l,j} \norm{\Delta_j} \\
\leq&\sqrt{n}  \frac{22 k\epsilon_d}{\sqrt{mn}} \leq \frac{22 k\epsilon_d}{\sqrt{m}}~,
\end{align*}

\subsubsection*{(3) Complementarity Slackness: For all $r \in [m]_{-l}$, $\lambda_r \left( y_r N(\widetilde{\btheta}, \bx_r)-1 \right) \leq \delta + \frac{184k\epsilon_d}{m}$.} Same proof using the modified margin difference $\frac{9k\epsilon_d}{mn}$.

\subsubsection*{(4) Primal Feasibility: For all $r \in [m]_{-l}$, $y_i N(\bx_i, \widetilde{\btheta}) \geq 1-\frac{9k\epsilon_d}{mn}$.} Same.

To conclude, $\widetilde{\btheta}$ is an $(\epsilon + \frac{22 k\epsilon_d}{\sqrt{m}}, \delta + \frac{184k\epsilon_d}{m}, \frac{9k\epsilon_d}{mn})$-approximate KKT for the margin maximization problem (\eqref{eq:margmax}) w.r.t. $S_r$ (\defref{def:epsdeltagammakkt}). Using \lemref{lem:fixthemargin} we conclude that $\frac{1}{1-\frac{9k\epsilon_d}{mn}}\widetilde{\btheta}$ is an $(\epsilon + \frac{9\epsilon_d\epsilon}{\frac{m}{k}-9\epsilon_d} + \frac{23k\epsilon_d}{\sqrt{m}}, \delta + \frac{9\epsilon_d\delta}{\frac{m}{k}-9\epsilon_d}+ \frac{22.6k\epsilon_d}{m})$-approximate $KKT$ for the margin maximization problem (\eqref{eq:margmax}) w.r.t. $S_r$ according to \defref{def:epsdeltakkt}, which finish the proof.

\subsubsection*{Proof of $\cossim(\hat{\theta} , \widetilde{\theta}) \geq 1-  \frac{82k\epsilon_d}{m-k}$:}

For the cosine similarly, by noting that $ \widetilde{\btheta} = \hat{\btheta} + \bu$, we have that (same as \ref{app:nonlinunlearning})
\[ \cossim(\hat{\btheta} , \widetilde{\btheta}) \geq 1- \frac{2 \norm{\bu}}{\norm{\hat{\btheta}}}~. \]
we have $\norm{\bu} \leq \frac{22 k\epsilon_d}{\sqrt{m}}$ and for $\norm{\hat{\btheta}}$, we can follow that same proof with only replace $\sum\limits_{i \in [m]_{-l}}  \lambda_i$ with $\sum\limits_{i \in I_r}  \lambda_i$, which will slightly effect the norm, having
\[
\norm{\hat{\btheta}}^2 \geq 0.3(m-k)~.
\]
Thus, we get for the ratio:
\[
\frac{\norm{\bu}}{\norm{\hat{\btheta}}} \leq \frac{\frac{22 k\epsilon_d}{\sqrt{m}}}{\sqrt{0.3(m-k)}}
    \leq \frac{41k\epsilon_d}{m-k}
\]

Joining it all together we have
\[
\cossim(\hat{\btheta} , \widetilde{\btheta}) \geq 1- \frac{2 \norm{\bu}}{\norm{\hat{\btheta}}} \geq 1- \frac{82k\epsilon_d}{m-k}~,
\]
which conclude the proof.
\end{proof}

\subsection{The Identity is an Unsuccessful Unlearning Algorithm}\label{app:nonlinidentity}

Similarly to the linear case, 
%we show the applicability of the $(\epsilon, \delta, \tau)$-successful definition to capture unlearning, by proving that keeping the original network doesn't consist of a successful unlearning algorithm. Particularly, we show that for network as defined in \thmref{thm:linunlearning}, its cosine similarity to any $(\epsilon, \delta)$-approximate KKT point for $S \setminus {\bx_l, y_l}$ is relatively large.
we complement \thmref{thm:nonlinunlearning} by providing the following remark, that shows that keeping the original network is not a successful unlearning algorithm.
Particularly, we show that for the network in \thmref{thm:nonlinunlearning}, its cosine similarity to any $(\epsilon, \delta)$-approximate KKT point for $S \setminus \{(\bx_l, y_l)\}$ is relatively large (see proof in \appref{app:nonlinidentity}).

\begin{remark}\label{remark:nonlinidentity}
    In the same settings as \ref{thm:nonlinunlearning}, the algorithm $\mathcal{A}_I(\btheta, S, r) = \btheta$, is $(\epsilon, \delta, \rho)$-successful only for $\rho \geq \frac{C}{m}  + C(\epsilon_d + \epsilon + \widetilde{\epsilon})$ for some $C>0$.
\end{remark}

\begin{proof}
In this section we show that the original network $\btheta$ is not a good candidate for the unlearning tasks according to the $(\epsilon, \delta, \tau)$-successful definition (\defref{def:suc_unlearning}). Formally, we look at the simple unlearning algorithm $\mathcal{A}_I(\btheta, S, r) = \btheta$. 
% First, we show that $\mathcal{A}_I$ is only $(\epsilon', \delta' ,0)$ successful for $\epsilon'$ that is at least constant. Second, 
We show that $\btheta$ will have a small cosine-similarity with any KKT point w.r.t. the retain set $S \setminus (\bx_l, y_l)$. Namely, that $\mathcal{A}_I$ is $(\epsilon', \delta' ,\tau')$ successful for $\tau'$ that is at least $O(\frac{1}{mn}) - O(\frac{\epsilon_d}{n})$.

Next, we show for $\tau > 0$. Let $\widetilde{\btheta}$ be an $(\widetilde{\epsilon}, \widetilde{\delta})$-approximate KKT point w.r.t. $S \setminus (\bx_l, y_l)$. We show that $\tau \geq O(\frac{1}{mn}) - O(\frac{\epsilon_d}{n})$.

From stationarity for $\btheta$ w.r.t. $S$, and for $\widetilde{\btheta}$ w.r.t. $S \setminus (\bx_l,y_l)$ we get that
% As in previous part, we have 
\[
\btheta = \sum_{i \in [m]}\lambda_iy_i \nabla_{\btheta} N(\btheta, \bx_i) + \bv_{\epsilon}~,
\]
and 
\[
\widetilde{\btheta} = \sum_{i \in [m]_{-l}}\widetilde{\lambda}_iy_i \nabla_{\btheta} N(\widetilde{\btheta}, \bx_i) + \bv_{\widetilde{\epsilon}}~.
\]

We denote $\alpha_i = \sum\limits_{j \in [n]} u_j \lambda_i \sigma'_{i,j}$ and $\widetilde{\alpha}_i = \sum\limits_{j \in [n]} u_j \widetilde{\lambda_i} \widetilde{\sigma}'_{i,j}$, and $\underline{\btheta} = \btheta - \bv_\epsilon , \underline{\widetilde{\btheta}} = \widetilde{\btheta} - \bv\widetilde{\epsilon}$

By Cauchy–Schwarz inequality we have

\begin{align*}
    &\inner{\btheta, \widetilde{\btheta}} = \inner{\underline{\btheta} + \bv_{\epsilon}, \underline{\widetilde{\btheta}}+ \bv_{\widetilde{\epsilon}}} = \\
    &\leq \inner{\underline{\btheta}, \underline{\widetilde{\btheta}}} + |\inner{\bv_\epsilon, \underline{\widetilde{\btheta}}}| + |\inner{\bv_{\widetilde{\epsilon}},\underline{\btheta}}| \\
    &\leq \inner{\underline{\btheta}, \underline{\widetilde{\btheta}}} + \epsilon\norm{\underline{\widetilde{\btheta}}} + \widetilde{\epsilon}\norm{\underline{\btheta}} ~.
    % &\inner{\bv_\epsilon, \sum_{i \in [m]_{-l}}\widetilde{\lambda}_iy_i \nabla_{\btheta} N(\widetilde{\btheta}, \bx_i)} +\\
    % & \inner{\bu_l, \sum_{i \in [m]_{-l}}\widetilde{\lambda}_iy_i \nabla_{\btheta} N(\widetilde{\btheta}, \bx_i)} + \\
    % &\inner{\sum_{i \in [m]_{-l}}\lambda_iy_i \nabla_{\btheta} N(\btheta, \bx_i), \bv_{\widetilde{\epsilon}}}+ \\
    % & \inner{\bv_\epsilon, \bv_{\widetilde{\epsilon}}} + \inner{\bu_l, \bv_{\widetilde{\epsilon}}}
\end{align*}

% For readability, we denote $\alpha_i = \sum\limits_{j \in [n]} u_j \lambda_i \sigma'_{i,j}$ and $\widetilde{\alpha}_i = \sum\limits_{j \in [n]} u_j \widetilde{\lambda_i} \widetilde{\sigma}'_{i,j}$.

For the inner product between the sums, we have 
\begin{align*}
    \inner{\underline{\btheta},\underline{\widetilde{\btheta}}} &= \inner{\sum_{i \in [m]_{-l}}\lambda_iy_i \nabla_{\btheta} N(\btheta, \bx_i), \sum_{i \in [m]}\widetilde{\lambda}_iy_i \nabla_{\btheta} N(\widetilde{\btheta}, \bx_i)} = \\
    &= \sum\limits_{j \in [n]}\inner{\sum\limits_{i \in [m]} u_j \lambda_i y_i \sigma'_{i,j} \bx_i, \sum\limits_{i \in [m]_{-l}} u_j \widetilde{\lambda}_i y_i \widetilde{\sigma}'_{i,j} \bx_i}\\
    &= \inner{\sum\limits_{i \in [m]} \sum\limits_{j \in [n]}u_j \lambda_i y_i \sigma'_{i,j} \bx_i, \sum\limits_{i \in [m]_{-l}} \sum\limits_{j \in [n]} u_j \widetilde{\lambda}_i y_i \widetilde{\sigma}'_{i,j} \bx_i}\\
    &= \inner{\sum\limits_{i \in [m]} \alpha_i y_i \bx_i, \sum\limits_{i \in [m]_{-l}} \widetilde{\alpha}_i y_i \bx_i}\\
    &\leq |\inner{\sum\limits_{i \in [m]} \alpha_i y_i \bx_i, \sum\limits_{i \in [m]_{-l}} \widetilde{\alpha}_i y_i \bx_i}|\\
    &\leq \sum\limits_{i \in [m]_{-l}} \alpha_i \widetilde{\alpha}_i \norm{\bx_i}^2 + \sum\limits_{i \neq k \in [m]_{-l}} \alpha_i \widetilde{\alpha}_k \inner{\bx_i, \bx_k} +  \sum\limits_{i \in [m]_{-l}} \alpha_l \widetilde{\alpha}_i \inner{\bx_l, \bx_i}\\
    &\leq \sum\limits_{i \in [m]_{-l}} \alpha_i \widetilde{\alpha}_i \norm{\bx_i}^2 + \phi \sum\limits_{i \neq k \in [m]_{-l}} \alpha_i \widetilde{\alpha}_k +  \phi \sum\limits_{i \in [m]_{-l}} \alpha_l \widetilde{\alpha}_i\\
    % &\leq \sum\limits_{j \in [n]} \left( \sum\limits_{i \in [m]_{-l}} u_j^2 \lambda_i \widetilde{\lambda}_i \sigma'_{i,j} \widetilde{\sigma}'_{i,j} \norm{\bx_i}^2 + \sum\limits_{i \in [m]_{-l}} \sum\limits_{k \neq i \in [m]_{-l}} u_j^2 \lambda_i \widetilde{\lambda}_k \sigma'_{i,j} \widetilde{\sigma}'_{k,j} \inner{\bx_i, \bx_k} \right)\\
    % & =  \sum\limits_{j \in [n]}\frac{1}{n} \sum\limits_{i \in [m]_{-l}}      \lambda_i\sigma'_{i,j}  \widetilde{\lambda}_i \widetilde{\sigma}'_{i,j} \norm{\bx_i}^2 +  \frac{1}{n} \sum\limits_{i \in [m]_{-l}} \sum\limits_{k \neq i \in [m]_{-l}}  \sum\limits_{j \in [n]} \lambda_i \sigma'_{i,j} \widetilde{\lambda}_k  \widetilde{\sigma}'_{k,j} \inner{\bx_i, \bx_k}~. \\
\end{align*}

For lower bounds of the norms we perform similar calculations. We note that $\norm{\widetilde{\btheta}} \geq \norm{\underline{\widetilde{\btheta}}} - \epsilon$, and
\begin{align*}
     \norm{\underline{\widetilde{\btheta}}}^2 &= \sum\limits_{j \in [n]} \norm{\widetilde{\bw}_j}^2 \\
    &= \sum\limits_{j \in [n]}\norm{\sum\limits_{i \in [m]_{-l}} u_j \widetilde{\lambda}_i y_i \widetilde{\sigma}'_{i,j} \bx_i}^2 \\
    &= \sum\limits_{j \in [n]}\inner{\sum\limits_{i \in [m]_{-l}} u_j \widetilde{\lambda}_i y_i \widetilde{\sigma}'_{i,j} \bx_i, \sum\limits_{i \in [m]_{-l}} u_j \widetilde{\lambda}_i y_i \widetilde{\sigma}'_{i,j} \bx_i}\\
    &= \inner{\sum\limits_{i \in [m]_{-l}} \sum\limits_{j \in [n]} u_j \widetilde{\lambda}_i y_i \widetilde{\sigma}'_{i,j} \bx_i, \sum\limits_{i \in [m]_{-l}} \sum\limits_{j \in [n]} u_j \widetilde{\lambda}_i y_i \widetilde{\sigma}'_{i,j} \bx_i}\\
    &= \inner{\sum\limits_{i \in [m]_{-l}} \widetilde{\alpha}_i y_i \bx_i, \sum\limits_{i \in [m]_{-l}} \widetilde{\alpha}_i y_i \bx_i}\\
    &\leq |\inner{\sum\limits_{i \in [m]_{-l}} \widetilde{\alpha}_i y_i \bx_i, \sum\limits_{i \in [m]_{-l}} \widetilde{\alpha}_i y_i \bx_i}|\\
    &\geq \sum\limits_{i \in [m]_{-l}} \widetilde{\alpha}_i^2 \norm{\bx_i}^2 - \sum\limits_{i \neq k \in [m]_{-l}} |\widetilde{\alpha}_i \widetilde{\alpha}_k| \inner{\bx_i, \bx_k}\\
    &\geq \sum\limits_{i \in [m]_{-l}} \widetilde{\alpha}_i^2 \norm{\bx_i}^2 - \phi \sum\limits_{i \neq k \in [m]_{-l}} |\widetilde{\alpha}_i \widetilde{\alpha}_k|\\
\end{align*}

and similarly
\begin{align*}
     \norm{\underline{\btheta}}^2 &= \sum\limits_{j \in [n]} \norm{{\bw}_j}^2 \\
    &= \sum\limits_{j \in [n]}\norm{\sum\limits_{i \in [m]} u_j {\lambda}_i y_i {\sigma}'_{i,j} \bx_i}^2 \\
    &\geq \sum\limits_{i \in [m]} \alpha_i^2 \norm{\bx_i}^2 - \sum\limits_{i \neq k \in [m]} |\alpha_i \alpha_k| \inner{\bx_i, \bx_k}\\
    &\geq \alpha_l^2 \norm{\bx_l}^2 + \sum\limits_{i \in [m]_{-l}} \alpha_i^2 \norm{\bx_i}^2 - \phi \sum\limits_{i \neq k \in [m]_{-l}} |\alpha_i \alpha_k|\\
\end{align*}

Plug it all in the cosine similarity definition we get

\begin{align*}
    &\cossim(\btheta, \widetilde{\btheta}) = \frac{\inner{\btheta, \widetilde{\btheta}}}{\norm{\btheta} \norm{\widetilde{\btheta}}} \leq \frac{\inner{\underline{\btheta}, \underline{\widetilde{\btheta}}}}{\norm{\btheta} \norm{\widetilde{\btheta}}} + \frac{\epsilon \norm{\underline{\widetilde{\btheta}}} + \widetilde{\epsilon}\norm{\underline{\btheta}}}{\norm{\btheta} \norm{\widetilde{\btheta}}} + \frac{\epsilon \widetilde{\epsilon}}{\norm{\btheta} \norm{\widetilde{\btheta}}}
    % &\frac{\inner{\sum_{i \in [m]_{-l}}\lambda_iy_i \nabla_{\btheta} N(\btheta, \bx_i), \sum_{i \in [m]_{-l}}\widetilde{\lambda}_iy_i \nabla_{\btheta} N(\widetilde{\btheta}, \bx_i)} + \inner{\bu_l, \sum_{i \in [m]_{-l}}\widetilde{\lambda}_iy_i \nabla_{\btheta} N(\widetilde{\btheta}, \bx_i)}}{\norm{\btheta} \norm{\widetilde{\btheta}}} + \frac{\epsilon\norm{\widetilde{\btheta}} + \widetilde{\epsilon}\norm{\btheta}}{\norm{\btheta} \norm{\widetilde{\btheta}}}
\end{align*}

bounding the second fraction we have 
\begin{align*}
    &\frac{\epsilon \norm{\underline{\widetilde{\btheta}}} + \widetilde{\epsilon}\norm{\underline{\btheta}}}{\norm{\btheta} \norm{\widetilde{\btheta}}} \leq \frac{\epsilon}{\norm{\btheta}} + \frac{\widetilde{\epsilon}}{\norm{\widetilde{\btheta}}} \\
\end{align*}

and note that using \lemref{lem:lambda_lowerbound} and \lemref{lem:lambda_upperbound}, if we denote $l = \left(0.9 - 4.64\frac{\epsilon_d}{n} - 1.92\epsilon\right)$ for all $i \in [m]$ 
\[
-l\sqrt{n} \leq \alpha_i, \widetilde{\alpha}_i \leq 20.4\sqrt{n}
\]
\begin{align*}
    \norm{\btheta}^2 &\geq {\sum\limits_{i \in [m]} \alpha_i^2 \norm{\bx_i}^2 - \phi \sum\limits_{i \neq k \in [m]_{-l}} \alpha_i \alpha_k}\\
    &\geq \left( \sum\limits_{i \in [m]} |\alpha_i| \right) \left( 0.9 l \sqrt{n} - \phi 20.4m\sqrt{n}\right)\\
    &\geq ml\sqrt{n} (0.9 l \sqrt{n} - \frac{5.1\epsilon_d}{\sqrt{n}}) \\
    &\geq mln (0.9 l - \frac{5.1\epsilon_d}{n}) \geq \frac{C}{mn} \\
\end{align*}

and similarly $\norm{\btheta}^2 > \frac{C}{mn}$ then \[
\frac{\epsilon}{\norm{\btheta}} + \frac{\widetilde{\epsilon}}{\norm{\widetilde{\btheta}}} \leq \frac{C(\epsilon + \widetilde{\epsilon})}{\sqrt{mn}}
\]

bounding the first fraction we have
\begin{align*}
    &\frac{\inner{\underline{\btheta}, \underline{\widetilde{\btheta}}}}{\norm{\btheta} \norm{\widetilde{\btheta}}} \leq \\
    &\frac{\sum\limits_{i \in [m]_{-l}} \alpha_i \widetilde{\alpha}_i \norm{\bx_i}^2 + \phi \sum\limits_{i \neq k \in [m]_{-l}} \alpha_i \widetilde{\alpha}_k + \phi \sum\limits_{i \in [m]_{-l}} \alpha_l \widetilde{\alpha}_i}{\sqrt{\sum\limits_{i \in [m]_{-l}} \widetilde{\alpha}_i^2 \norm{\bx_i}^2 - \phi \sum\limits_{i \neq k \in [m]_{-l}} \widetilde{\alpha}_i \widetilde{\alpha}_k - \epsilon} \sqrt{\alpha_l^2 \norm{\bx_l}^2 + \sum\limits_{i \in [m]_{-l}} \alpha_i^2 \norm{\bx_i}^2 - \phi \sum\limits_{i \neq k \in [m]_{-l}} \alpha_i \alpha_k- \widetilde{\epsilon}}}
\end{align*}

% We note that by Cauchy–Schwarz inequality 
% \[
% \sum\limits_{i \in [m]_{-l}} \alpha_i \widetilde{\alpha}_i \norm{\bx_i}^2 \leq \sqrt{\sum\limits_{i \in [m]_{-l}} \widetilde{\alpha}_i^2 \norm{\bx_i}^2} \sqrt{\sum\limits_{i \in [m]_{-l}} \alpha_i^2 \norm{\bx_i}^2}
% \]

We lower bound the norm of the parameter

\begin{align*}
     \norm{\underline{\btheta}}^2 &= \sum\limits_{j \in [n]} \norm{{\bw}_j}^2 \\
    % &= \sum\limits_{j \in [n]}\norm{\sum\limits_{i \in [m]} u_j {\lambda}_i y_i {\sigma}'_{i,j} \bx_i}^2 \\
    % &\geq \sum\limits_{i \in [m]} \alpha_i^2 \norm{\bx_i}^2 - \sum\limits_{i \neq k \in [m]} |\alpha_i \alpha_k| \inner{\bx_i, \bx_k}\\
    &\geq \sum\limits_{i \in [m]} \alpha_i^2 \norm{\bx_i}^2 - \phi \sum\limits_{i \neq k \in [m]} |\alpha_i \alpha_k|\\
    &\geq (\sum\limits_{i \in [m]} \alpha_i) [a - \phi m b] \geq m0.9a [a - \frac{0.6 \epsilon_d}{n}]
\end{align*}

As $a - \frac{0.6 \epsilon_d}{n} > C$ for some $C>0$, we note we get a similar equation as in the linear case (\ref{app:linidentity}), and skip to the result, having

\begin{align*}
    \cossim(\btheta, \widetilde{\btheta}) \leq 1 - \frac{C}{m}  + C(\epsilon_d + \epsilon + \widetilde{\epsilon})~.
\end{align*}
\end{proof}

\section{Appendix for \secref{sec:generalization}}
\subsection{Proofs for settings properties}
\label{app:gen_data}

We first show this dataset $S=\{(\bx_i,y_i)\}_{i=1}^m \sim \mathcal{D}_{MG}^m$ satisfy the conditions we discuss in our paper:
\begin{enumerate}
    \item For all $\bx_i \in S$, $\norm{\bx_i}^2 \in [1-\psi, 1+\psi]$ for $\psi=0.1$.
    \item For all $(\bx_i,y_i), (\bx_j,y_j)\in S$ s.t. $i \neq j$, $|\inner{\bx_i,\bx_j}| \leq \phi$
\end{enumerate}
For a sample $(\bx_i,y_i) \sim \mathcal{D}$, we first show that $\bx_i$'s norm is a bounded constant. Denote $\bx_i = \bmu_i + \bzeta_i$ for $\norm{\bmu_i} = \cdot d^{-\frac{1}{4}+\alpha}$ for $\alpha \in (0,\frac{1}{4})$, and $\bzeta_i \sim \mathcal{N}(0, \frac{1}{d}I_d)$.

We show tighter bounds for $\norm{\bzeta_i}^2$.

\begin{lemma}
\label{lem:gen_normalnormbound}
    Let $i \in [m]$. Then, w.p. $\geq 1-(2e^{-\frac{d}{1700}}) $, $\norm{\bzeta_i}^2 \in [0.95,1.05]$.
\end{lemma}
\begin{proof}
    For the lower bound, similar to \lemref{lem:randomnormlarge}, we have for $w \sim \mathcal{N}(\zero,\sigma^2 I_n)$
    \[
	 	\Pr\left[ n - \norm{\frac{w}{\sigma}}^2 \geq 2\sqrt{nt} \right] \leq e^{-t}~.
    \]
    We let $t=\frac{1}{1600} \cdot n$, $\sigma^2 = \frac{1}{d}$ and $n=d$ and get 
    \[
	 \Pr\left[\norm{w}^2 \leq \frac{95}{100}  \right] 
		\leq e^{-\frac{d}{1600}}~.
    \]
    as desired.
    For the upper bound, similar to \lemref{lem:randomnormsmall}, we have for $w \sim \mathcal{N}(\zero,\sigma^2 I_n)$
    	\[
	 	\Pr\left[ \norm{\frac{w}{\sigma}}^2 - n \geq 2\sqrt{nt} + 2t \right] \leq e^{-t}~.
	\]
    We let $t=\frac{1}{1700} \cdot n$, $\sigma^2 = \frac{1}{d}$ and $n=d$ and get
    \[
    \Pr\left[ \norm{w}^2  \geq 1.05 \right] 
		\leq e^{-\frac{d}{1700}}~.
    \]
\end{proof}

\begin{lemma}
    \label{lem:gen_normbound}
    w.p. $1-(2e^{-\frac{d}{1700}})$, for sufficiently large $d$, $\norm{\bx_i}^2 \in [0.9 , 1.1 ] $.
\end{lemma}

\begin{proof}
    We denote $\bx_i = \bmu_i + \bzeta_i$, such that $\bzeta_i \sim \mathcal{N}(0, \frac{1}{d}I_d)$. From \lemref{lem:gen_normalnormbound} we get that w.p. $1-(2e^{-\frac{d}{1700}})$, $\norm{\bzeta_i}^2 \in [0.95, 1.05]$. 
    
    As for $\norm{\bmu_i}$, we note that $\norm{\bmu_i}^2 = d^{2(-\frac{1}{4}+\alpha)} = d^{(-\frac{1}{2}+2\alpha)}$, therefore if enough to take $d$ such that 
    \[d^{2\alpha - \frac{1}{2}} \leq 0.01 \iff d^{\frac{1}{2}- \alpha}\geq 100 \iff \log(d) \geq \frac{\log(100)}{\frac{1}{2}- \alpha} \]

    Then, for such $d$ we have,
    \[
    \norm{\bx_i}^2 = \norm{\bmu_i + \bzeta_i}^2 = \norm{\bmu_i}^2 + \norm{\bzeta_i}^2 + 2\inner{\bmu_i, \bzeta_i}
    \]
    \[\norm{\bmu_i}^2 + \norm{\bzeta_i}^2 - 2|\inner{\bmu_i, \bzeta_i}| \leq \norm{\bx}^2 \leq \norm{\bmu_i}^2 + \norm{\bzeta_i}^2 + 2|\inner{\bmu_i, \bzeta_i}|\]
    \[ 2|\inner{\bmu_i, \bzeta_i}| \leq 2\norm{\bmu_i} \norm{\bzeta_i} \leq 2 \cdot 0.01 \cdot 1.05 = 0.021\]
    and therefore,
    \[0.9 < 0.929 \leq \norm{\bx_i}^2 \leq 1.081 < 1.1\]
as desired.
\end{proof}

Next, we look at two samples $(\bx_i,y_i), (\bx_j,y_j) \sim \mathcal{D}_{MG}$, showing that if $i\neq j$, $\bx_i, \bx_j$ are almost orthogonal.

\begin{lemma}
\label{lem:gen_innercorbound}
    Let $i\neq j$, and let $(\bx_i,y_i), (\bx_j,y_j) \sim \mathcal{D}_{MG}$. Then, for sufficiently large $d$, w.p.  $\geq 1- e^{-d/500} + 6d^{-\frac{\log(d)}{2}}$:
    \[|\inner{\bx_i,\bx_j}|-\inner{\bmu_i,\bmu_j} \in [- 2\norm{\bmu_i} \frac{\log(d)}{\sqrt{d}} - 1.1 \frac{\log(d)}{\sqrt{d}}, 2\norm{\bmu_i} \frac{\log(d)}{\sqrt{d}} +   1.1 \frac{\log(d)}{\sqrt{d}}]\] 

\end{lemma}

\begin{proof}
    Let $\bx_i,\bx_j $ data points.
    % such that $i \neq j$ and $\bmu_i = \bmu_j$. 
    We denote $\bx_i = \bmu_i + \bzeta_i$ and $\bx_j = \bmu_j + \bzeta_j$
    We look at -
    \begin{align*}
        \inner{\bx_i,\bx_j} = \inner{ \bmu_i + \bzeta_i, \bmu_j + \bzeta_j} &= \inner{\bmu_i, \bmu_j} + \inner{ \bmu_i,  \bzeta_j} + \inner{\bzeta_i, \bmu_j} + \inner{ \bzeta_i,  \bzeta_j} \\
    \end{align*}
    Since $\bmu_i \in \reals^n$ and $\bzeta_j \sim \mathcal{N}(0, \frac{1}{d}I_d)$, we get from \lemref{lem:randominnermul_1} for $t=\frac{\log(d)}{\sqrt{d}}$ that w.p. $\geq 1-2d^{-\frac{\log(d)}{2}}$
    \[
    |\inner{\bmu_i, \bzeta_j}| \leq \norm{\bmu_i}\frac{\log(d)}{\sqrt{d}}
    \]
    From the same argument $|\inner{\bmu_j, \bzeta_i}| \leq \norm{\bmu_j}\frac{\log(d)}{\sqrt{d}}$.
    
    Finally, From \lemref{lem:randominnermul} we get that w.p. $\geq 1-(e^{-d/500} + 2d^{-\frac{\log(d)}{2}})$, $|\inner{ \bzeta_i,  \bzeta_j}| \leq 1.1 \frac{\log(d)}{\sqrt{d}}$. 
    Combining all together,
    \[
    \Pr  \left[|\inner{\bx_i,\bx_j}| - \inner{\bmu_i, \bmu_j}  \geq  2\norm{\bmu_i} \frac{\log(d)}{\sqrt{d}}  +   1.1 \frac{\log(d)}{\sqrt{d}} \right] \leq e^{-d/500} + 6d^{-\frac{\log(d)}{2}}
    \]
    and the claim follows.
\end{proof}

\begin{lemma}
\label{lem:gen_phiboundedfromzero}
    For $d$ large enough and $\norm{\bmu_+} = \frac{\log(d)}{d^{\alpha}}$, for $\alpha \in (0,\frac{1}{4})$,
    \[\norm{\bmu_+}^2 > 2\norm{\bmu_+} \frac{\log(d)}{\sqrt{d}} +   1.1 \frac{\log(d)}{\sqrt{d}} \]
\end{lemma}
\begin{proof}
    \begin{align*}
    &\norm{\bmu_+}^2 - 2\norm{\bmu_+} \frac{\log(d)}{\sqrt{d}} -   1.1 \frac{\log(d)}{\sqrt{d}}  \\
    =&\frac{1}{d^{\frac{1}{2}-2\alpha}} - 2\frac{\log(d)}{d^{\alpha+3/4}}  -   1.1 \frac{\log(d)}{\sqrt{d}} \\
    =& d^{-\frac{1}{2}} \left(d^{2\alpha} -  2\log(d)d^{-\frac{1}{4}} - 1.1 \log(d)\right)\\
\end{align*}

it's enough to find $d$ such that 
\[
d^{2\alpha} \geq 2\log(d)d^{-\frac{1}{4}} + 1.1 \log(d) \iff 2\alpha \geq \frac{ \log \left(2\log(d)d^{-\frac{1}{4}} + 1.1 \log(d)\right)}{\log d}
\]
which is possible since r.h.s goes to $0$ when $d$ goes to infinity. 

\end{proof}

\begin{lemma}
\label{lem:gen_dataass}
    Let a dataset $S=\{(\bx_i,y_i)\}_{i=1}^m$ be such that $\forall i$, $\bx_i \in \reals^d$ and $(\bx_i,y_i) \sim \mathcal{D}_{MG}$, for $m\leq d$ and for sufficiently large $d$. Then, w.p. $\geq 1- (2me^{-\frac{d}{1700}}+ m^2 e^{-d/500} + 2m^2d^{-\frac{\log(d)}{2}})$
    \begin{enumerate}
        \item For all $(\bx,y) \in S$, $\norm{\bx}^2 \in [0.9, 1.1]$
        \item For all $(\bx_i,y_i), (\bx_j,y_j)\in S$, $|\inner{\bx_i,\bx_j}| \leq \phi$ for $\phi \leq \frac{\epsilon_d}{4mn}$
    \end{enumerate}
\end{lemma}

\begin{proof}
\begin{enumerate}
    \item  First, 
    \[
    \Pr \left[\forall (\bx,y) \in S, \norm{\bx}^2 \in [0.9, 1.1]\right] = \Pr \left[\max_{(\bx,y) \in S} \norm{\bx}^2 \in [0.9, 1.1] \right]~,
    \]
    and the claim follows w.p. $\geq 1- 2me^{-\frac{d}{1700}} $, directly from using simple union, given \lemref{lem:gen_normbound}.
    \item First, 
    \[
    \Pr \left[\forall (\bx_i,y_i), (\bx_j,y_j)\in S, |\inner{\bx_i,\bx_j}| \leq \frac{\epsilon_d}{4mn} \right] = \Pr \left[\max_{(\bx_i,y_i), (\bx_j,y_j) \in S} |\inner{\bx_i,\bx_j}| \leq \frac{\epsilon_d}{4mn} \right]~.
    \]
    From \lemref{lem:gen_innercorbound} we get that w.p.  $\geq 1- e^{-d/500} + 6d^{-\frac{\log(d)}{2}}$:
    \[|\inner{\bx_i,\bx_j}|-\inner{\bmu_i,\bmu_j} \in [- 2\norm{\bmu_i} \frac{\log(d)}{\sqrt{d}} - 1.1 \frac{\log(d)}{\sqrt{d}}, 2\norm{\bmu_i} \frac{\log(d)}{\sqrt{d}} +   1.1 \frac{\log(d)}{\sqrt{d}}]~.\] 
    Therefore, we get that maximal value for $|\inner{\bx_i,\bx_j}|$ if we take $i \neq j$ such that $y_i = y_j$, resulting in
    \begin{align*}
        |\inner{\bx_i,\bx_j}| \leq \norm{\bmu_i}^2 + 2\norm{\bmu_i} \frac{\log(d)}{\sqrt{d}} +   1.1 \frac{\log(d)}{\sqrt{d}}
    \end{align*}

From \lemref{lem:gen_phiboundedfromzero} one can see its enough to choose $d$ such that 
\[
2 \norm{\bmu_+}^2  = 2 \frac{1}{d^{\frac{1}{2}-2\alpha}} \leq \frac{\epsilon_d}{4mn}~,
\]
which is possible since $\frac{\epsilon_d}{4mn}$ is given constant and $\lim_{d \rightarrow \infty} \frac{1}{d^{\frac{1}{2}-2\alpha}} = 0 $. Then, from using simple union, the claim follows.
\end{enumerate}
\end{proof}

For the next lemma, we add few notations for readability. 
\begin{enumerate}
    \item $\phi^+_{\max} = \max_{i,j}\{\inner{\bx_i,\bx_j} : y_i = y_j\}, \phi^+_{\min} = \min_{i,j}\{\inner{\bx_i,\bx_j} : y_i = y_j\}$ 
    \item $\phi^-_{\max} = \max_{i,j}\{\inner{\bx_i,\bx_j} : y_i \neq y_j\}, \phi^-_{\min} = \min_{i,j}\{\inner{\bx_i,\bx_j} : y_i \neq y_j\}$. 
\end{enumerate}
\begin{lemma}
\label{lem:gen_allphis}
    Let a dataset $S=\{(\bx_i,y_i)\}_{i=1}^m$ be such that $\forall i$, $\bx_i \in \reals^d$ and $(\bx_i,y_i) \sim \mathcal{D}_{MG}$. Then, for $m\leq d$ and for sufficiently large $d$, w.p. $\geq 1- (m e^{-d/500} + 6md^{-\frac{\log(d)}{2}})$, for all $(\bx_i,y_i), (\bx_j,y_j)\in S$:
    \begin{align*}
     0 < \phi^+_{\max} = -\phi^-_{\min} = \norm{\bmu_i} + 2\norm{\bmu_i} \frac{\log(d)}{\sqrt{d}} +   1.1 \frac{\log(d)}{\sqrt{d}} \leq \frac{\epsilon_d}{4mn}\\
     0 < \phi^+_{\min} = -\phi^-_{\max} = \norm{\bmu_i} - 2\norm{\bmu_i} \frac{\log(d)}{\sqrt{d}} -   1.1 \frac{\log(d)}{\sqrt{d}} 
        \end{align*}

\end{lemma}
\begin{proof}
    The proof is directly from \lemref{lem:gen_innercorbound}, using simple union bound same as \lemref{lem:gen_dataass}. Both larger than $0$ from \lemref{lem:gen_phiboundedfromzero}. 
\end{proof}

\begin{lemma}
\label{lem:gen_innersigns}
    Suppose a two-layer neural network $N(\btheta, \bx)  = \sum\limits_{j=1}^{n} u_j \sigma(\bw_j^\top \bx)$, trained on a dataset $S=\{(\bx_1,y_1),...,(\bx_m,y_m)\} \sim \mathcal{D}_{MG}^m$, described in \defref{def:exampledata}. Assume that $\btheta$ is a KKT point of the margin maximization problem (\eqref{eq:margmax}) w.r.t. $S$ as in \defref{def:epsdeltakkt}. Let  $(\bx_t, y_t) \sim \mathcal{D}$, Then for all $j \in [n]$
    \[
    \sign(\hat{\bw}_j^\top \bx_t) = \sign(\bw_j^\top \bx_t) = y_t \sign(u_j)
    \]
\end{lemma}
\begin{proof}
 Let  $(\bx_t, y_t) \sim \mathcal{D}$.
    Since $\btheta$ is a $KKT$ point, from \defref{def:epsdeltakkt} we get that
\[
\bw_j = u_j \sum\limits_{i=1}^{m}  \lambda_i y_i \sigma'_{i,j} \bx_i~,~ \bw_j^\top \bx_t = u_j \sum\limits_{i=1}^{m}  \lambda_i y_i \sigma'_{i,j} \inner{\bx_i, \bx_t} 
\]
\[
\hat{\bw}_j = u_j \sum\limits_{i\in [m]_{-l}}  \lambda_i y_i \sigma'_{i,j} \bx_i~,~ \hat{\bw}_j^\top \bx_t = u_j \sum\limits_{i\in [m]_{-l}}  \lambda_i y_i \sigma'_{i,j} \inner{\bx_i, \bx_t} 
\]
where $\sigma'_{i,j} = \onefunc_{\bw_j^T \bx_j \geq 0}$.\\

Case 1: $y_t =1$.\\
We note that for all $i \in [m]$, $y_i \inner{\bx_i, \bx_t} \geq \phi^+_{\min}>0$: If $y_i =1$, $y_i \inner{\bx_i, \bx_t} = \inner{\bx_i, \bx_t} \geq \phi^+_{\min}$, else $y_i = -1$ and $\inner{\bx_i, \bx_t} \leq \phi^-_{\max}$ so $-\inner{\bx_i, \bx_t} \geq -\phi^-_{\max} = \phi^+_{\min}$, from \lemref{lem:gen_allphis}. Therefore, for all $j \in [n]$, $\sign(\hat{\bw}_j^\top \bx_t) = \sign(\bw_j^\top \bx_t) = \sign(u_j) = y_t \sign(u_j)$. 

Case 2: $y_t =-1$.\\
We note that for all $i \in [m]$, $y_i \inner{\bx_i, \bx_t} \leq \phi^-_{\max}<0$: If $y_i =1$, $y_i \inner{\bx_i, \bx_t} = \inner{\bx_i, \bx_t} \leq \phi^-_{\max}$, else $y_i = -1$ and $\inner{\bx_i, \bx_t} \geq \phi^+_{\min}$ so $-\inner{\bx_i, \bx_t} \geq -\phi^+_{\min} = \phi^-_{\max}$, from \lemref{lem:gen_allphis}. Therefore, for all $j \in [n]$, $\sign(\hat{\bw}_j^\top \bx_t) = \sign(\bw_j^\top \bx_t) = -\sign(u_j) = y_t \sign(u_j)$. 
\end{proof}

\subsection{Proof for \thmref{thm:regulargenerlization}}
\label{app:generalization}
% \om{fix domination lemma for new mu with no log}
% \om{w.p. 1-mbla, for d large enough (don't need to say d large enough, only in the proof) we get that PR[]> 1 - bla}
% \om{to remove $\log(d)$ from the $\mu$, and have $a \rightarrow \frac{1}{4}-\alpha$, and then $d$ depend on $\alpha$}
First, we note that according to \lemref{lem:gen_dataass}, w.p. $\geq 1- (2me^{-\frac{d}{1700}}+ m^2 e^{-d/500} + 2m^2d^{-\frac{\log(d)}{2}})$ over the choice of $S$, $S$ satisfies \assref{ass:data}. For readability, the following proof we assume $S$ satisfies \assref{ass:data}.
Given a data point $(\bx_t, y_t) \sim \mathcal{D}_{MG}$, we show that 
\[y_t N(\btheta,\bx_t) = y_t \sum\limits_{j=1}^{n} u_j \sigma(\bw_j^\top \bx_t) > 0~.\]
We denote $\bx_t = \bmu_t + \bzeta_t$, for $\bzeta_t \sim \mathcal{N}(0, \frac{1}{d})$. We denote $I^+ = \{i \in [m]: y_i = 1\}$, $I^- = \{i \in [m]: y_i = -1\}$. We also denote $\phi^+_{\max} = \max_{i,j \in [m]}\{\inner{\bx_i,\bx_j} : y_i = y_j\}, \phi^+_{\min} = \min_{i,j \in [m]}\{\inner{\bx_i,\bx_j} : y_i = y_j\}$ and $\phi^-_{\max} = \max_{i,j \in [m]}\{\inner{\bx_i,\bx_j} : y_i \neq y_j\}, \phi^-_{\min} = \min_{i,j \in [m]}\{\inner{\bx_i,\bx_j} : y_i \neq y_j\}$.

Next, from 

From \lemref{lem:gen_allphis} we get that $\phi^-_{\max} = -\phi^+_{\min}$ and $\phi^-_{\min} = -\phi^+_{\max}$
Since $\btheta$ is a $KKT$ point, from \defref{def:epsdeltakkt} we get that
\[
\bw_j = u_j \sum\limits_{i=1}^{m}  \lambda_i y_i \sigma'_{i,j} \bx_i~,~ \bw_j^\top \bx_t = u_j \sum\limits_{i=1}^{m}  \lambda_i y_i \sigma'_{i,j} \inner{\bx_i, \bx_t} 
\]
where $\sigma'_{i,j} = \onefunc_{\bw_j^T \bx_j \geq 0}$.\\

Case 1: $y_t =1$.\\
We show that $N(\btheta,\bx_t) > 0 $. From \lemref{lem:gen_innersigns}, for all $j \in [n]$, $\sign(\bw_j^\top \bx_t) = \sign(u_j)$. Hence,

\begin{align*}
    N(\btheta,\bx_t) &= \sum\limits_{j=1, u_j<0}^{n} u_j \sigma(\bw_j^\top \bx_t) + \sum\limits_{j=1, u_j\geq0}^{n} u_j \sigma(\bw_j^\top \bx_t) \\
    &= \sum\limits_{j=1, u_j\geq0}^{n} u_j \bw_j^\top \bx_t \\
    &= \sum\limits_{j=1, u_j\geq0}^{n} u_j^2 \sum\limits_{i=1}^{m}  \lambda_i y_i \sigma'_{i,j} \inner{\bx_i, \bx_t} \\
    &=\sum\limits_{i=1}^{m} y_i \inner{\bx_i, \bx_t} \sum\limits_{j=1, u_j\geq0}^{n} u_j^2 \lambda_i  \sigma'_{i,j}\\ 
\end{align*}

First, we note that for all $i \in [m]$, $y_i \inner{\bx_i, \bx_t} \geq \phi^+_{\min}>0$: If $y_i =1$, $y_i \inner{\bx_i, \bx_t} = \inner{\bx_i, \bx_t} \geq \phi^+_{\min}$, else $y_i = -1$ and $\inner{\bx_i, \bx_t} \leq \phi^-_{\max}$ so $-\inner{\bx_i, \bx_t} \geq -\phi^-_{\max} = \phi^+_{\min}$, from \lemref{lem:gen_allphis}. Next, since $S$ satisfies \assref{ass:data}, and $\btheta$ satisfies \ref{def:epsdeltakkt} for $\epsilon=\delta=0$ we get from \lemref{lem:lambda_lowerbound} that for all $i \in [m]$, $\sum\limits_{j=1, u_j\geq0}^{n} u_j^2 \lambda_i  \sigma'_{i,j}> 0$.

% Since this is a sum non-negative expressions we get that $N(\btheta,\bx_t) \geq 0$
% To get the strict inequality $N(\btheta,\bx_t) > 0$, we use the fact that $\btheta$ is a KKT point of the margin maximization problem (\eqref{eq:nonlin_margmax}) w.r.t. $S$ as in \ref{def:kkt}. Let $k \in [m]$ such that $y_k = 1$, from (\eqref{eq:nonlin_margmax}) we get that $1 \leq N(\bx_k, \btheta) = \sum\limits_{i=1}^{m} y_i \inner{\bx_i, \bx_k} \sum\limits_{j=1, u_j\geq0}^{n} u_j^2 \lambda_i  \sigma'_{i,j}$, therefore necessarily $\sum\limits_{i=1}^{m} \sum\limits_{j=1, u_j\geq0}^{n} u_j^2 \lambda_i  \sigma'_{i,j} > 0 $ and the claim follows.

Case 2: $y_t =-1$.\\
Similarly, we show that $N(\btheta,\bx_t) < 0 $. From \lemref{lem:gen_innersigns}, for all $j \in [n]$, $\sign(\bw_j^\top \bx_t) = -\sign(u_j)$. Hence,

\begin{align*}
    N(\btheta,\bx_t) &= \sum\limits_{j=1, u_j<0}^{n} u_j \sigma(\bw_j^\top \bx_t) + \sum\limits_{j=1, u_j\geq0}^{n} u_j \sigma(\bw_j^\top \bx_t) \\
    &= \sum\limits_{j=1, u_j<0}^{n} u_j \bw_j^\top \bx_t \\
    &= \sum\limits_{j=1, u_j<0}^{n} u_j^2 \sum\limits_{i=1}^{m}  \lambda_i y_i \sigma'_{i,j} \inner{\bx_i, \bx_t} \\
    &=\sum\limits_{i=1}^{m} y_i \inner{\bx_i, \bx_t} \sum\limits_{j=1, u_j<0}^{n} u_j^2 \lambda_i  \sigma'_{i,j}\\ 
\end{align*}

We similarly note that that for all $i \in [m]$, $y_i \inner{\bx_i, \bx_t} \leq \phi^-_{\max}<0$: If $y_i =1$, $y_i \inner{\bx_i, \bx_t} = \inner{\bx_i, \bx_t} \leq \phi^-_{\max}$, else $y_i = -1$ and $\inner{\bx_i, \bx_t} \geq \phi^+_{\min}$ so $-\inner{\bx_i, \bx_t} \geq -\phi^+_{\min} = \phi^-_{\max}$, from \lemref{lem:gen_allphis}. And from \lemref{lem:lambda_lowerbound} we get that $\sum\limits_{j=1, u_j<0}^{n} u_j^2 \lambda_i  \sigma'_{i,j}>0$ and the claim follows.\\

% Since this is a sum non-positive expressions we get that $N(\btheta,\bx_t) \leq 0$
% To get the strict inequality $N(\btheta,\bx_t) > 0$, we use the same argument as previous case. Let $k \in [m]$ such that $y_k = -1$, from (\eqref{eq:nonlin_margmax}) we get that $- 1 \geq N(\bx_k, \btheta) = \sum\limits_{i=1}^{m} y_i \inner{\bx_i, \bx_k}  \sum\limits_{j=1, u_j<0}^{n} u_j^2 \lambda_i  \sigma'_{i,j}$, therefore necessarily $\sum\limits_{i=1}^{m}  \sum\limits_{j=1, u_j<0}^{n} u_j^2 \lambda_i  \sigma'_{i,j} > 0 $ and the claim follows.

For showing that 
\[y_t N(\hat{\btheta}, \bx_t) = y_t \sum\limits_{j=1}^{n} u_j \sigma(\hat{\bw}_j^\top \bx_t) > 0~,\]
the proof is almost identical. In the end of each case we look at 
\[
\sum\limits_{i\in [m]_{-l}} y_i \inner{\bx_i, \bx_t} \sum\limits_{j=1, u_j\geq0}^{n} u_j^2 \lambda_i  \sigma'_{i,j}~,
\]

and all the same arguments holds, concluding generalization for $\hat{\btheta}$ as well, which finishes the proof.

We note that the same arguments can be used to show generalization for the case of unlearning a forget set $S_{\text{forget}} \subseteq S$ of any size $k<m$ using the extended algorithm $\mathcal{A}_{\text{k-GA}}$, discussed in \secref{sec:kset}. In this case, we instead look at
\[
\sum\limits_{i\in S \setminus S_{\text{forget}}} y_i \inner{\bx_i, \bx_t} \sum\limits_{j=1, u_j\geq0}^{n} u_j^2 \lambda_i  \sigma'_{i,j}~,
\]
yet the same arguments hold, concluding generalization.

\section{Experiment details}\label{app:exp}

% \subsection{Unlearning with a two-layer ReLU network}
We take a high dimensional data set, where $m=10$, $d=1000$, the data distribution is $\mathcal{N}(0, \frac{1}{d}I_d)$. As mentioned in Example.~\ref{exp:gausiandata}, the data satisfies \assref{ass:data} for small value of $\phi$ and $\psi$. We experiment with fully-connected ReLU networks, trained using SGD optimizer with binary cross entropy loss that is normalized to have a margin of size $1$. In this experiment, for each data point $\bx_i \in S$, we calculate $\lambda_i$, and unlearn it using the gradient ascent algorithm $\mathcal{A}_{\text{GA}}$ with step size $\alpha \lambda_i$ for $\alpha \in [0,1.5]$, resulting in $\widetilde{\btheta}_i(\alpha)$. For each $\widetilde{\btheta}_i(\alpha)$ we calculate the corresponding $\epsilon, \delta$ for its KKT conditions with respect to $S \setminus (\bx_i, y_i)$.
In Figure~\ref{fig:lambda}, we sample one point from $S$, preform the unlearning algorithm for all $10$ networks, and average the results.

% \subsection{Experiments for two layer network}
We test for a two-layer fully-connected ReLU network $\btheta$ as in Eq.~\eqref{eq:NN}, with $n=400$. We initialize the network with small initialization for the first layer by dividing its standard deviation by a factor of $10^5$. We train with full batch size for $10^5$ epochs, using SGD optimizer with a $10^{-5}$ wight decay factor. 
%We trained the network on $m=10$ gaussian vectors (i.e. sampled from $\sim \mathcal{N}(0, \frac{1}{\sqrt{d}})$) of dimension $d=1000$.
% The training learning rate is $2$, weight init variance, weight decay factor,
% $m=10, d=1000, n=400, lr=2, init factor=100k, full batch, 100k epochs, wd=1-5$

% \subsection{Unlearning with a four-layer ReLU network}
% We train all the layers. Learning rate is $1$, the weight initialization variance is $\frac{1}{1000d}$. We init the output layer as in Eq.~\ref{eq:NN}. The model width is $40$ (for all hidden layers), we train with batch size $1$ for $?$ epochs, and with weight decay factor of $1^-7$. 

% \begin{itemize}
%     \item 
% \end{itemize}

\end{document}